\documentclass[nohyperref]{article}

\usepackage{microtype}
\usepackage{graphicx}
\usepackage{subfigure}
\usepackage{booktabs} 

\usepackage{hyperref}

\usepackage[accepted]{icml2022}

\usepackage{amsmath}
\usepackage{amssymb}
\usepackage{mathtools}
\usepackage{amsthm}

\usepackage[capitalize,noabbrev]{cleveref}

\usepackage[textsize=tiny]{todonotes}

\usepackage{epsfig}
\usepackage{graphicx}
\usepackage{amsmath}
\usepackage{amssymb}
\usepackage{amsmath,mathtools,amssymb,amsthm,dsfont,cleveref,enumitem,diagbox,multirow,multicol,siunitx,colortbl}
\usepackage[most]{tcolorbox}
\usepackage[linesnumbered,ruled,vlined,algo2e]{algorithm2e}

\usepackage{aisecure-math}

\usepackage{xspace}
\usepackage{nicefrac}
\usepackage{wrapfig}
\definecolor{tabgray}{gray}{0.95}

\usepackage{silence}
\usepackage{balance}
\usepackage[title, titletoc]{appendix}
\usepackage{titletoc}

\newcommand\DoToC{%
  \startcontents
  \printcontents{}{1}{\textbf{Contents}\vskip3pt\hrule\vskip5pt}
  \vskip3pt\hrule\vskip5pt
}

\newcommand{\approachbold}{\textbf{D}ouble \textbf{S}ampling \textbf{R}andomized \textbf{S}moothing\xspace}
\newcommand{\shortApproach}{DSRS\xspace}

\newcommand{\ext}{\textup{ext}}
\newcommand{\sgn}{\mathrm{sgn}}

\newcommand{\dif}{\mathrm{d}}
\newcommand{\wF}{\widetilde{F}}
\def\rNP{r_{\mathsf{N-P}}}
\def\rDSRS{r_{\mathsf{DSRS}}}
\def\rDSRSmulti{r_{\mathsf{DSRS}}^{\mathsf{multi}}}
\def\multi{{\mathsf{multi}}}
\def\gNg{\gN^{\mathsf{g}}}
\def\gNtrunc{\gN_{\mathsf{trunc}}}
\def\gNgtrunc{\gNg_{\mathsf{trunc}}}
\def\rtight{r_{\mathsf{tight}}}
\def\Pcon{P_{\mathsf{con}}}
\def\GammaCDF{\Gamma\mathrm{CDF}}
\def\BetaCDF{\mathrm{BetaCDF}}
\def\Beta{\mathrm{Beta}}

\crefname{algorithm}{Alg.}{Algs.}
\crefname{algocf}{Alg.}{Algs.}
\crefname{equation}{Eqn.}{Eqns.}

\newcount\Comments  
\newcount\arxiv  
\usepackage{color}
\definecolor{darkgreen}{rgb}{0,0.5,0}
\definecolor{darkblue}{rgb}{0,0,0.5}
\definecolor{purple}{rgb}{1,0,1}
\definecolor{gray}{rgb}{0.5,0.5,0.5}
\definecolor{darkred}{rgb}{0.756,0.055,0.0}
\newcommand{\kibitz}[2]{\ifnum\Comments=0\textcolor{#1}{#2}\fi}

\ifnum\Comments=0
\fi

\ifnum\arxiv=1
\fi

\begin{document}

\twocolumn[
\icmltitle{Double Sampling Randomized Smoothing}



\icmlsetsymbol{equal}{*}

\begin{icmlauthorlist}
\icmlauthor{Linyi Li}{uiuc}
\icmlauthor{Jiawei Zhang}{uiuc}
\icmlauthor{Tao Xie}{pku}
\icmlauthor{Bo Li}{uiuc}
\end{icmlauthorlist}

\icmlaffiliation{uiuc}{University of Illinois Urbana-Champaign, Illinois, USA}
\icmlaffiliation{pku}{Peking University, Beijing, China}

\icmlcorrespondingauthor{Linyi Li}{linyi2@illinois.edu}
\icmlcorrespondingauthor{Tao Xie}{taoxie@pku.edu.cn}
\icmlcorrespondingauthor{Bo Li}{lbo@illinois.edu}

\icmlkeywords{Certified Deep Learning, Robustness, Verification}

\vskip 0.3in
]



\printAffiliationsAndNotice{}  

\begin{abstract}
    Neural networks~(NNs) are known to be vulnerable against adversarial perturbations, and thus there is a line of work aiming to provide robustness certification for NNs, such as randomized smoothing, which  samples smoothing noises from a certain distribution to certify the robustness for a smoothed classifier. 
    However, as shown by previous work, the certified robust radius in randomized smoothing suffers from scaling to large datasets~(``curse of dimensionality'').
    To overcome this hurdle, we propose a \approachbold~(\textbf{\shortApproach}) framework, which exploits the sampled probability from an \textit{additional smoothing distribution} to tighten the robustness certification of the previous smoothed classifier.
    Theoretically, 
    under mild assumptions, we prove that \shortApproach can certify $\Theta(\sqrt{d})$ robust radius under $\ell_2$ norm where $d$ is the input dimension, implying  that
    \emph{\shortApproach may be able to break the
    curse of dimensionality of randomized smoothing}.
    We instantiate \shortApproach for a generalized family of Gaussian smoothing and propose an efficient and sound computing method based on customized dual optimization considering sampling error.
    Extensive experiments on MNIST, CIFAR-10, and ImageNet verify our theory and show that \shortApproach certifies larger robust radii than existing baselines consistently under different settings. 
    Code is available at \url{https://github.com/llylly/DSRS}.
\end{abstract}

\part*{}
\vspace{-3em}

\section{Introduction}

    \label{sec:intro}
    
    Neural networks~(NNs) have achieved great advances on a wide range of tasks, but 
    have been shown  
    vulnerable against adversarial examples~\cite{szegedy2013intriguing,goodfellow2014explaining,eykholt2018robust,wang2021adversarial,qiu2020semanticadv,li2020qeba,zhang2021progressive}. 
    A plethora of empirical defenses are proposed to improve the robustness; however, most of these are broken by  strong adversaries again~\citep{carlini2017towards,athalye2018obfuscated,tramer2020adaptive}.
    Recently, there are great efforts in developing certified defenses for NNs under certain adversarial constraints~\citep{wong2018provable,raghunathan2018certified,li2020sok,xu2020automatic,li2019robustra}.
    
    \begin{figure}[!t]
        \centering
        \includegraphics[width=\linewidth]{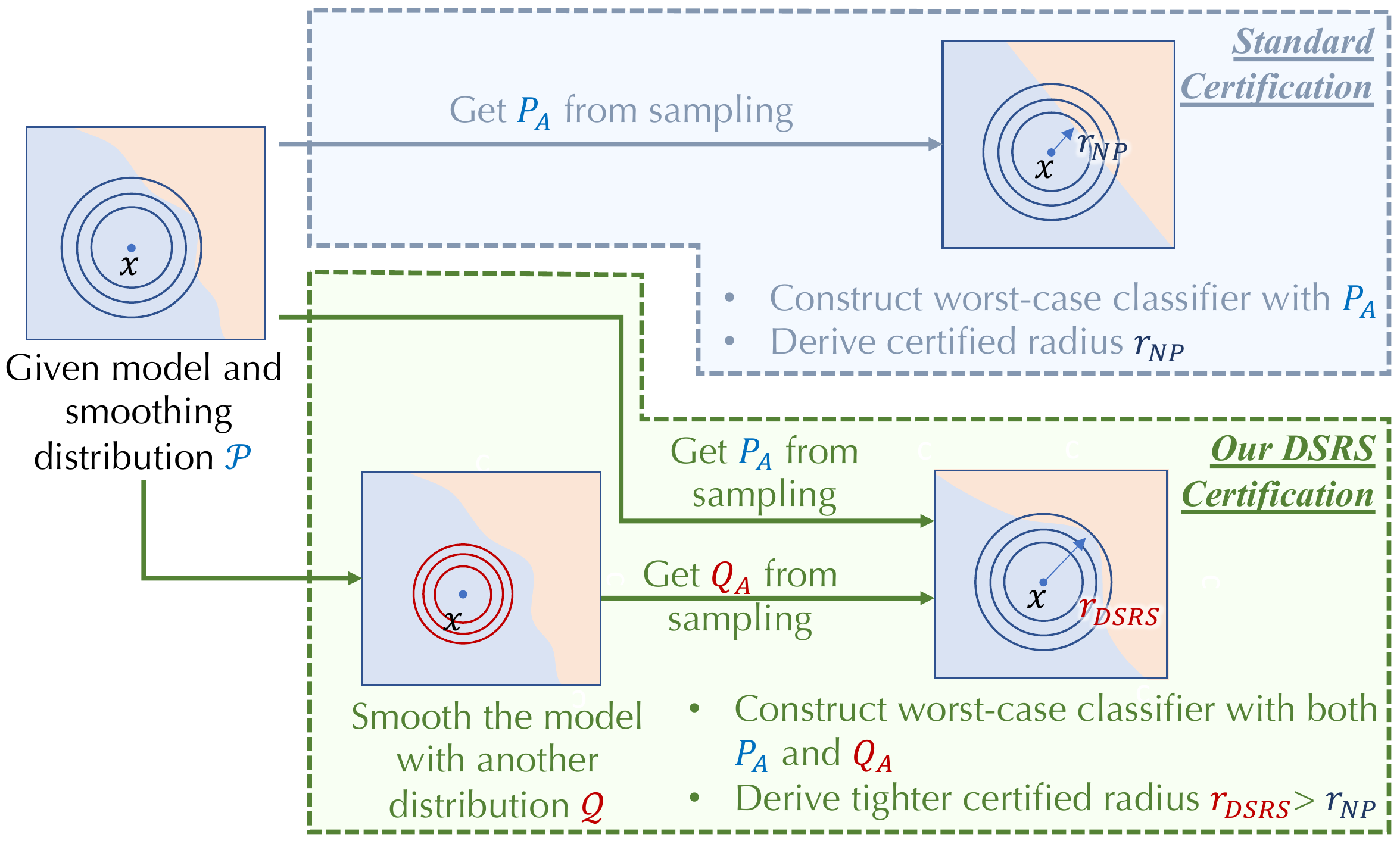}
        \vspace{-2.5em}
        \label{fig:1}
        \caption{\textbf{Upper}: Standard certification for randomized smoothing leverages information from only one distribution~(\textcolor{darkblue}{$P_A$}) to compute robustness certification.
        \textbf{Lower}:
        \shortApproach leverages information from two distributions~(\textcolor{darkblue}{$P_A$} and \textcolor{darkred}{$Q_A$}) to compute certification for the smoothed classifier, yielding significantly larger certified radius.}
        \vspace{-1em}
    \end{figure}
    
    Randomized smoothing~\citep{cohen2019certified,li2019certified} has emerged as a popular technique to provide certified robustness for large-scale datasets.
    Concretely, it samples noise from a certain smoothing distribution to construct a smoothed classifier, and thus certifies the robust radius for the smoothed classifier.
    Compared to other techniques~\citep{wong2018provable,mirman2018differentiable,gowal2019scalable,zhang2019towards},
    randomized smoothing is efficient and agnostic to the model, and is applicable to a wide range of ML models, including large ResNet~\citep{he2016deep} on the ImageNet dataset.
    
    To improve the certified robust radius, existing studies~\citep{cohen2019certified,lee2019tight,li2019certified,yang2020randomized} have explored different smoothing distributions.
    However, the improvement is limited.
    For example, $\ell_2$ certified robust radius does not increase on large datasets despite that the input dimension $d$ increases~\citep{cohen2019certified}, resulting in a low $\ell_\infty$ certified radius on large datasets, theoretically shown as an intrinsic barrier of randomized smoothing~(``curse of dimensionality'' or ``$\ell_\infty$ barrier'')~\citep{yang2020randomized,blum2020random,kumar2020curse,hayes2020extensions,wu2021completing}.
    
    Given these challenges toward tight robustness certification, a natural question arises:
    \emph{\textbf{Q1}: Is it possible to circumvent the barrier of randomized smoothing by certifying with additional ``information"?}
    \emph{\textbf{Q2}: What type of information is needed to provide tight robustness certification?}
    To answer these questions, we propose a  \approachbold~(\shortApproach) framework to leverage the sampled noises from an \textit{additional smoothing distribution} as additional information to tighten the robust certification.
    In theory, we show that
    (1)~ideally, if the decision region of the base classifier is known, \shortApproach can provide tight robustness certification;
    (2)~more practically, if the inputs, which can be correctly classified by the base classifier, satisfy the concentration property within an input-centered ball with constant mass under standard Gaussian measure, the standard Neyman-Pearson-based certification~\citep{li2019certified,cohen2019certified,salman2019provably,yang2020randomized} can  certify only a dimension-independent $\ell_2$ radius, whereas \shortApproach with generalized Gaussian smoothing can certify radius $\Omega(\sqrt d)$ (under $\ell_2$ norm),  which would increase with the dimension $d$, leading to tighter certification.
    Under more general conditions, we provide numerical simulations to verify our theory.
    Our results provide a positive answer to Q1 and sufficient conditions for Q2, i.e., \shortApproach may be able to circumvent the
    barrier of randomized smoothing.

    Motivated by the theory, we leverage a type of generalized Gaussian~\citep{zhang2020black} as the smoothing distribution and truncated generalized Gaussian as an additional distribution.
    For this type of concretization, we propose an efficient and sound computation method to compute the certifiably robust radius for practical classifiers considering sampling error.
    Our method formulates the certification problem given additional information as a constrained optimization problem and  leverages specific properties of the dual problem to decompose the effects of different dual variables to solve it.
    \shortApproach is fully scalable since the computational time is nearly independent of the size of the dataset, model, or  sampling.
    Our extensive experimental evaluation on MNIST, CIFAR-10, and ImageNet shows that
    (1)~under large sampling size~($2\times 10^5 - 8\times 10^5$),
    the certified radius of \shortApproach consistently increases as suggested by our theory;
    (2)~under practical sampling size~($10^5$), DSRS can certify consistently higher robust radii than existing baselines, including standard Neyman-Pearson-based certification.
    
    As further discussed in \Cref{newadx:general}, we believe that \shortApproach as a framework opens a wide range of future directions for selecting or optimizing different forms of \textit{additional information} to tighten the certification of randomized smoothing. 

    We summarize the main technical \emph{contributions} as follows:
    \begin{itemize}[leftmargin=5mm,itemsep=-0.5mm]
        \vspace{-0.75em}
        
        \item We propose a general robustness certification framework \shortApproach,  which leverages additional information by sampling from another smoothing distribution.
        
        \item We prove that under practical concentration assumptions, \shortApproach certifies $\Omega(\sqrt d)$ radius under $\ell_2$ norm with $d$ the input dimension, suggesting a possible way to circumvent the intrinsic barrier of randomized smoothing.
        
        \item We concretize \shortApproach by generalized Gaussian smoothing mechanisms and propose a method to efficiently compute the certified radius for given classifiers.
        
        \item We conduct extensive experiments, showing that \shortApproach provides consistently tighter robustness certification than existing baselines, including standard Neyman-Pearson-based certification across different models on MNIST, CIFAR-10, and ImageNet.
        
    \end{itemize}

\noindent\textbf{Related Work.}
    For the certification method of randomized smoothing, most existing methods leverage only the true-class prediction probability to certify.
    In this case, the tightest possible robustness certification is based on the Neyman-Pearson lemma~\cite{neyman1933ix} as first proposed by \citet{cohen2019certified} for certifying $\ell_2$ radius under Gaussian smoothing.
    Several methods extend this certification to accommodate different smoothing distributions and different $\ell_p$ norms~\citep{dvijotham2020framework,yang2020randomized,zhang2020black,levine2021improved}.
    In randomized smoothing,
    the $\ell_2$ certified robust radius $r$ is similar across datasets of different scales, resulting in the vanishing $\ell_\infty$ certified radius $r/{\sqrt d}$ when input dimension increases.
    This limitation of existing certification methods of randomized smoothing is formally proved~\citep{yang2020randomized,blum2020random,kumar2020curse,hayes2020extensions,wu2021completing} and named ``$\ell_\infty$ barrier'' or ``curse of dimensionality''.
    
    Recent work tries to incorporate additional information besides true-class prediction probability to tighten the certification and bypass the barrier.
    For $\ell_2$ and $\ell_\infty$ certification, to the best of our knowledge, gradient magnitude is the only exploited additional information~\citep{mohapatra2020higherorder,levine2020tight}.
    However, in practice, the improvement is relatively marginal and requires a large number of samples~(see \Cref{newadx:sub:compare-higher-order}).
    Some other methods provide tighter certification given specific model structures~\citep{kumar2020certifying,chiang2020detection,lee2019tight,awasthi2020adversarial}.
    \shortApproach
    instead focuses on leveraging model-structure-agnostic additional information.
    More discussion on related work is in \Cref{newadx:related}. 

\section{Preliminaries and Background}
    \label{sec:prelim}

    Define $[C] := \{1,\,\dots\,C\}$.
    Let $\vDelta^C$ be the $C$-dimensional probability simplex.
    We consider a multiclass classification model $F: \sR^d \to [C]$ as the \emph{base classifier}, where $d$ is the input dimension, and the model outputs hard-label class prediction within $[C]$.
    The \emph{original smoothing distribution} $\gP$ and \emph{additional smoothing distribution} $\gQ$ are both supported on $\sR^d$.
    We let $p(\cdot)$ and $q(\cdot)$ be their density functions respectively.
    We assume that both $p$ and $q$ are positive and differentiable almost everywhere, i.e., the set of singular points has zero measure under either $\gP$ or $\gQ$.
    These assumptions hold for common smoothing distributions used in the literature such as Gaussian distribution~\citep{lecuyer2019certified,li2019certified,cohen2019certified,yang2020randomized}.
    
    \textbf{Randomized smoothing} constructs a smoothed classifier from a given base classifier by adding input noise following \emph{original} smoothing distribution $\gP$.
    For input $\vx \in \sR^d$,
    we define \emph{prediction probability under $\gP$} by function $f^\gP: \sR^d \to \vDelta^C$:
    \vspace{-0.5em}
    \begin{equation}
        f^\gP(\vx)_c := \Pr_{\vepsilon\sim\gP} [F(\vx + \vepsilon) = c] \quad \text{where} \quad c\in [C].
        \label{eq:f}
        \vspace{-0.5em}
    \end{equation}
    The \emph{smoothed classifier} $\wF^\gP: \sR \to [C]$~(or $\wF$ when $\gP$ is clear from the context) predicts the class with the highest confidence after smoothing with $\gP$: 
    \vspace{-0.5em}
    \begin{equation}
        \wF^\gP(\vx) := \argmax_{c\in [C]} f^\gP(\vx)_c.
        \label{eq:F}
    \end{equation}
    \vspace{-1.5em}
    
    We focus on robustness certification against $\ell_p$-bounded perturbations for smoothed classifier $\wF$, where the standard certification method is called Neyman-Pearson-based certification~\citep{cohen2019certified}~(details in \Cref{newadx:detail-n-p-cert}, certified radius from it denoted by $\rNP$).
    Concretely, certification methods compute robust radius $r$ defined as below.
    \begin{definition}[Certified Robust Radius]
        \label{def:certified-robust-radius}
        Under $\ell_p$ norm~($p \in \sR_+ \cup \{+\infty\}$), for given smoothed classifier $\wF^\gP$ and input $\vx_0 \in \sR^d$ with true label $y_0\in [C]$, 
        a radius $r \ge 0$ is called \emph{certified (robust) radius} for $\wF^\gP$ if $\wF^\gP$ always predicts $y_0$ for any input within the $r$-radius ball centered at $\vx_0$:
        \vspace{-1mm}
        \begin{equation}
            \forall \vdelta \in \sR^d, \|\vdelta\|_p < r,
            \wF^\gP(\vx_0 + \vdelta) = y_0.
        \end{equation}
        \vspace{-8mm}
    \end{definition}

\section{\shortApproach Overview}

    We propose \approachbold (\shortApproach), which leverages the prediction probability from an \textit{additional} smoothing distribution $\gQ$~(formally $Q_A := f^\gQ(\vx_0)_{y_0} = \Pr_{\vepsilon\sim\gQ}[F(\vx+\vepsilon)=y_0]$), along with the prediction probability from the original smoothing distribution $\gP$~(formally $P_A := f^\gP(\vx_0)_{y_0}$ as in \Cref{eq:f}, also used in Neyman-Pearson-based certification), to provide robustness certification for $\gP$-smoothed classifier $\wF^\gP$.
    Note that both $P_A$ and $Q_A$ can be obtained from Monte-Carlo sampling~(see \Cref{subsec:cons-opt-formulation,subsec:preprocessing}).
    Formally, we let $\rDSRS$ denote the tightest possible certified radius with prediction probability from $\gQ$, then $\rDSRS$ can be defined as below.
    
    \begin{definition}[$\rDSRS$]
        \label{def:r-dsrs}
        Given $P_A$ and $Q_A$, 
        \vspace{-1mm}
        \begin{equation}
            \label{eq:r-dsrs}
            \begin{aligned}
                & \rDSRS := \max r \quad \mathrm{s.t.} \\
                & \forall F: \sR \to [C], f^\gP(\vx_0)_{y_0} = P_A, f^\gQ(\vx_0)_{y_0} = Q_A \\
                \Rightarrow & \forall \vx, \|\vx - \vx_0\|_p < r, \wF^\gP(\vx) = y_0.
            \end{aligned}
        \end{equation}
    \end{definition}
    
    Intuitively, $\rDSRS$ is the maximum possible radius, such that any smoothed classifier constructed from base classifier satisfying $P_A$ and $Q_A$ constraints cannot predict other labels when the perturbation magnitude is within the radius.
    
    In \Cref{sec:theory}, we will analyze the theoretical properties of \shortApproach, including comparing $\rDSRS$ and $\rNP$ under the concentration assumption.
    Computing $\rDSRS$ is nontrivial, so in \Cref{sec:dsrs}, we will introduce a practical computational method that exactly solves $\rDSRS$ when $\gP$ and $\gQ$ are standard and generalized (truncated) Gaussian.
    In \Cref{newadx:det-computional-method}, we will show method variants to deal with other forms of $\gP$ and $\gQ$ distributions.
    In \Cref{newadx:general}, we will further generalize the \shortApproach framework.
    
    \noindent\textbf{Smoothing Distributions.}
    Now we formally define the smoothing distributions used in \shortApproach.
    We mainly consider standard Gaussian $\gN$~\citep{cohen2019certified,yang2020randomized} and generalized Gaussian $\gNg$~\citep{zhang2020black}.
    Let $\gN(\sigma)$ to represent standard Gaussian distribution with covariance matrix $\sigma^2\mI_d$ that has density function $\propto \exp(-\|\vepsilon\|_2^2/(2\sigma^2))$.\footnote{In this paper, $\gN(\sigma)$ is a shorthand of $\gN(0,\sigma^2\mI_d)$.}
    For $k\in\sN$, we let $\gNg(k,\sigma)$ to represent generalized Gaussian whose density function $\propto \|\vepsilon\|_2^{-2k} \exp(-\|\vepsilon\|_2^2/(2\sigma'^2))$ where $\sigma' = \sqrt{d/(d-2k)}\sigma$.
    Here we use $\sigma'$ instead of $\sigma$ to ensure that the expected noise $\sqrt{\E \|\vepsilon\|_2^2}$ of $\gNg(k,\sigma)$ is the same as $\gN(\sigma)$.
    The generalized Gaussian as the smoothing distribution overcomes the ``thin shell'' problem of standard Gaussian and improves certified robustness~\citep{zhang2020black}; and we will reveal more of its theoretical advantages in \Cref{sec:theory}.
    
    As the additional smoothing distribution $\gQ$, we will mainly consider truncated distributions within a small $\ell_2$ radius ball.
    Specially, truncated standard Gaussian is denoted by $\gNtrunc(T,\sigma)$ with density function $\propto \exp(-\|\vepsilon\|_2^2/(2\sigma^2)) \cdot \1[\|\vepsilon\|_2 \le T]$;
    and truncated generalized Gaussian is denoted by $\gNgtrunc(k,T,\sigma)$ with density function $\propto \|\vepsilon\|_2^{-2k}\exp(-\|\vepsilon\|_2^2/(2\sigma'^2)) \cdot \1[\|\vepsilon\|_2 \le T]$.
    
    In \Cref{tab:smooth-dist-def}, we summarize these distribution definitions.
    
    \begin{table}[!t]
        \caption{Definitions of smoothing distributions in this paper. In the table, $k\in \sN$, $\sigma' = \sqrt{d/(d-2k)}\sigma$.}
        \label{tab:smooth-dist-def}
        \centering
        \resizebox{\linewidth}{!}{
        \begin{tabular}{ccc}
            \toprule
            Name & Notation & Density Function \\
            \midrule
            Standard Gaussian & $\gN(\sigma)$ & $\propto \exp\left(-\frac{\|\vepsilon\|_2^2}{2\sigma^2}\right)$ \\
            Generalized Gaussian & $\gNg(k,\sigma)$ & $\propto \|\vepsilon\|_2^{-2k} \exp\left(-\frac{\|\vepsilon\|_2^2}{2\sigma'^2}\right)$ \\
            Truncated Standard Gaussian & $\gNtrunc(T,\sigma)$ & $\propto \exp\left(-\frac{\|\vepsilon\|_2^2}{2\sigma^2}\right)\cdot \1[\|\vepsilon\|_2\le T]$ \\
            Truncated Generalized Gaussian & $\gNgtrunc(k,T,\sigma)$ & $\propto \|\vepsilon\|_2^{-2k} \exp\left(-\frac{\|\vepsilon\|_2^2}{2\sigma'^2}\right)\cdot \1[\|\vepsilon\|_2\le T]$ \\
            \bottomrule
        \end{tabular}
        }
        \vspace{-1em}
    \end{table}

\section{Theoretical Analysis of \shortApproach}
    \label{sec:theory}
    \label{sec:tightness-barrier}
    In this section, we theoretically analyze \shortApproach to answer the following core question:
    \emph{Does $Q_A$, the prediction probability under additional smoothing distribution, provide sufficient information for tightening the robustness certification?}
    We first show that if the support of $\gQ$ is the decision region of true class, \shortApproach can certify the smoothed classifier's maximum possible robust radius.
    Then, under concentration assumption, we show the $\ell_2$ certified radius of \shortApproach can be $\Omega(\sqrt d)$ that is asymptotically optimal for bounded inputs.
    Finally, under more general conditions, we conduct both numerical simulations and real-data experiments to verify that the certified radius of \shortApproach increases with data dimension $d$.
    These analyses provide a positive answer to the above core question.
    
    \vspace{-2mm}
    \subsection*{\shortApproach can certify the tightest possible robust radius.}
    \vspace{-2mm}
        
        Given an original smoothing distribution $\gP$ and a base classifier $F_0$.
        At input point $\vx_0 \in \sR^d$ with true label $y_0$, we define the tightest possible certified robust radius $\rtight$ to be the largest $\ell_p$ ball that contains no adversarial example for \emph{smoothed} classifier $\wF_0^\gP$:
        \vspace{-2mm}
        \begin{equation*}
            \rtight := \max r \, \mathrm{s.t.} \, \forall \vdelta\in\sR^d, \|\vdelta\|_p < r, \wF_0^\gP(\vx_0+\vdelta) = y_0.
            \vspace{-2mm}
        \end{equation*}
        Then, for binary classification, if we choose an additional smoothing distribution $\gQ$ whose support is the decision region or its complement, then \shortApproach can certify robust radius $\rtight$.
        
        \begin{theorem}
            \label{thm:suffices-binary-classification}
            Suppose the original smoothing distribution $\gP$ has non-zero density everywhere, i.e., $p(\cdot) > 0$.
            For binary classification with base classifier $F_0$, 
            at point $\vx_0\in\sR^d$,
            let $\gQ$ be an additional distribution that satisfies: (1)~its support is the decision region of an arbitrary class $c\in [C]$ shifted by $\vx_0$: $\supp(\gQ) = \{\vx - \vx_0: F_0(\vx) = c\}$;
            (2)~for any $\vx \in \supp(\gQ)$, $0 < q(\vx)/p(\vx) < +\infty$.
            Then, plugging $P_A = f_0^\gP(\vx_0)_c$ and $Q_A = f_0^\gQ(\vx_0)_c$~(see \Cref{eq:f}) into \Cref{def:r-dsrs}, we have $\rDSRS = \rtight$ under any $\ell_p$~($p \ge 1$).
        \end{theorem}
        
        \begin{proof}[Proof sketch.]
            We defer the proof to \Cref{newadx:sub:proof-suffices-binary-classification}.
            At a high level, with this type of $\gQ$, we have $Q_A = 1$ or $Q_A = 0$.
            Then, from the mass of the $\gQ$'s support on $\gP$ and $P_A$, we can conclude that the $\gQ$'s support is exactly the decision region of label $c$ or its complement.
            Thus, the DSRS constraints~(in \Cref{eq:r-dsrs}) are satisfied iff $F$ differs from $F_0$ in a zero-measure set, and thus we exactly compute the smoothed classifier $\wF_0^\gP$'s maximum certified robust radius in \shortApproach.
            An extension to multiclass setting is in \Cref{newadx:sub:cor-suffices-multiclass-classification}.
        \end{proof}
        
        \begin{remarkbox}
            For any base classifier $F_0$, $\gQ$ that satisfies conditions in \Cref{thm:suffices-binary-classification} exists, implying that with \shortApproach, certifying a strictly tight robust radius is possible.
            In contrast, Neyman-Pearson-based is proved to certify tight robust radius for linear base classifiers~\citep[Section 3.1]{cohen2019certified}, but for arbitrary base classifiers, its tightness is not guaranteed.
            This result suggests that, to certify a tight radius, just one additional smoothing distribution $\gQ$ is sufficient rather than multiple ones.
            
            On the other hand, it is challenging to find $\gQ$ whose support~(or its complement) exactly matches the decision region of an NN classifier.
            In the following, we analyze the tightness of \shortApproach under weaker assumptions.
        \end{remarkbox}
    
    \vspace{-2mm}
    \subsection*{\shortApproach can certify $\Omega(\sqrt d)$ $\ell_2$ radius under concentration assumption.}
    \vspace{-2mm}
    
        We begin by defining the concentration property.
    
        \begin{definition}[($\sigma,\Pcon$)-Concentration]
            \label{def:concentration}
            Given a base classifier $F_0$, at input $\vx_0 \in \sR$ with true label $y_0$, we call $F_0$ satisfies ($\sigma,\Pcon$)-concentration property, if for within $\Pcon$-percentile of small $\ell_2$ magnitude Gaussian $\gN(\sigma)$ noise, the adversarial example occupies zero measure.
            Formally, ($\sigma,\Pcon$)-concentration means
            \vspace{-0.5em}
            \begin{subequations}
                \begin{align}
                    & \Pr_{\vepsilon \sim \gN(\sigma)} [F_0(\vx_0 + \vepsilon) = y_0 \,|\, \|\vepsilon_0\|_2 \le T] = 1 
                    \label{eq:concentration}  \\
                    \mathrm{where} & \quad T \, \mathrm{satisfies} \, \Pr_{\vepsilon \sim \gN(\sigma)} [\|\vepsilon\|_2 \le T] = \Pcon.
                \end{align}
            \vspace{-1.5em}
            \end{subequations}
        \end{definition}
        
        Intuitively, ($\sigma,\Pcon$)-concentration implies that the base classifier has few adversarial examples for small magnitude noises during standard Gaussian smoothing.
        In \Cref{fig:smooth-model-landscape-and-magnitude-density}~(in \Cref{newadx:illustration-concentration-fig}), we empirically verified that a well-trained base classifier on ImageNet may satisfy this property for a significant portion of inputs.
        Furthermore, \citet{salman2019provably} show that promoting this concentration property by adversarially training the smoothed classifier improves the certified robustness.
        With this concentration property, \shortApproach certifies the radius $\Omega(\sqrt{d})$ under $\ell_2$ norm, as the following theorem shows.
        
        \begin{theorem}
            \label{thm:concentration-sqrt-d}
            Let $d$ be the input dimension and $F_0$ be the base classifier.
            For an input point $\vx_0 \in \sR^d$ with true class $y_0$, suppose $F_0$ satisfies $(\sigma,\Pcon)$-Concentration property.
            Then, for any sufficiently large $d$,
            for the classifier $\wF_0^{\gP'}$ smoothed by generalized Gaussian $\gP'=\gNg(k,\sigma)$ with $d/2 - 15 \le k < d/2$, \shortApproach with additional smoothing distribution $\gQ = \gNgtrunc(k,T,\sigma)$ can certified $\ell_2$ radius 
            \vspace{-0.5em}
            \begin{equation}
                \rDSRS \ge 0.02\sigma\sqrt{d}
                \vspace{-0.5em}
            \end{equation}
            where $T = \sigma\sqrt{2\GammaCDF_{d/2}^{-1}(\Pcon)}$ and $\GammaCDF_{d/2}$ is the CDF of gamma distribution $\Gamma(d/2,1)$.
        \end{theorem}
        \vspace{-0.5em}
        \begin{proof}[Proof sketch.]
            We defer the proof to \Cref{newadx:sub:concentration-sqrt-d}.
            At high level, 
            based on the standard Gaussian distribution's property~(\Cref{prop:thm-2-1}), we find $Q_A = 1$ under concentration property~(\Cref{lem:thm-2-2}).
            With $Q_A = 1$, we derive a lower bound of $\rDSRS$ in \Cref{lem:thm-2-3}.
            We then use: (1)~the concentration of beta distribution $\Beta(\frac{d-1}{2},\frac{d-1}{2})$~(see \Cref{lem:thm-2-4}) for large $d$;
            (2)~the relative concentration of gamma $\Gamma(d/2,1)$ distribution around mean for large $d$~(see \Cref{prop:thm-2-5} and resulting \Cref{fact:thm-2-7});
            and (3)~the misalignment of gamma distribution $\Gamma(d/2-k,1)$'s mean and median for small $(d/2-k)$~(see \Cref{prop:thm-2-6}) to lower bound the quantity in \Cref{lem:thm-2-3} and show it is large or equal to $0.5$.
            Then, using the conclusion in \Cref{sec:dsrs} we conclude that $\rDSRS \ge 0.02\sigma\sqrt{d}$.
        \end{proof}
        
        \begin{remarkbox}
            (1)~For standard Neyman-Pearson based certificaton, $\rNP = \sigma\Phi^{-1}(f^\gP(\vx_0)_{y_0})$.
            Along with the increase of input dimension $d$, to achieve growing $\ell_2$ certified radius, one needs the prediction probability of true class under $\gP$, namely $f^\gP(\vx_0)_{y_0}$, to grow simultaneously, which is challenging.
            Indeed, across different datasets, $f^\gP(\vx_0)_{y_0}$ is almost a constant, which leads to a constant $\ell_2$ certified radius and shrinking $\ell_\infty$ radius for large $d$.
            We further empirically illustrate this property in \Cref{newadx:unchanged-pa}.
            
            (2)~In contrast, as long as the model satisfies concentration property, which may be almost true on large datasets as reflected by \Cref{fig:smooth-model-landscape-and-magnitude-density},
            with our specific choices of $\gP$ and $\gQ$,
            \shortApproach can achieve $\Omega(\sigma\sqrt d)$ $\ell_2$ radius on large datasets.
            This rate translates to a constant $\Omega(\sigma)$ $\ell_\infty$ radius on large datasets and thus breaks the curse of dimensionality of randomized smoothing.
            We remark that this $\sqrt d$ rate is optimal when dataset input is bounded such as images~(otherwise, the $\omega(1)$ $\ell_\infty$ radius leads the radius to exceed the constant $\ell_\infty$ diameter for large $d$).
            Therefore, \emph{under the assumption of concentration property, \shortApproach provides asymptotically optimal certification for randomized smoothing}. 
            
            (3)~Smoothing with generalized Gaussian distribution and choosing a parameter $k$ that is close to $d/2$ play an essential role in proving the $\Omega(\sigma\sqrt d)$ certified radius. Otherwise, in \Cref{adxsec:concentration-constant-k-forbid-sqrt-d} we have \Cref{thm:concentration-constant-k-forbid-sqrt-d} that shows any certification methods cannot certify an $\ell_2$ radius $c\sqrt d$ for any $c > 0$.
            This adds another theoretical evidence for the superiority of generalized Gaussian that is cross-validated by \citet{zhang2020black}.
        \end{remarkbox}

    \begin{figure}[t]
        \centering
        \subfigure[When holding probability in \Cref{eq:concentration} is obtained from sampling $N$ times and confidence level $1-\alpha=99.9\%$, relation between certified radius~($y$-axis) and input data dimension $d$~($x$-axis).
            Different curves correspond to different $N$s.]
            {
            \label{fig:simulation-sampling}
            \includegraphics[width=\linewidth]{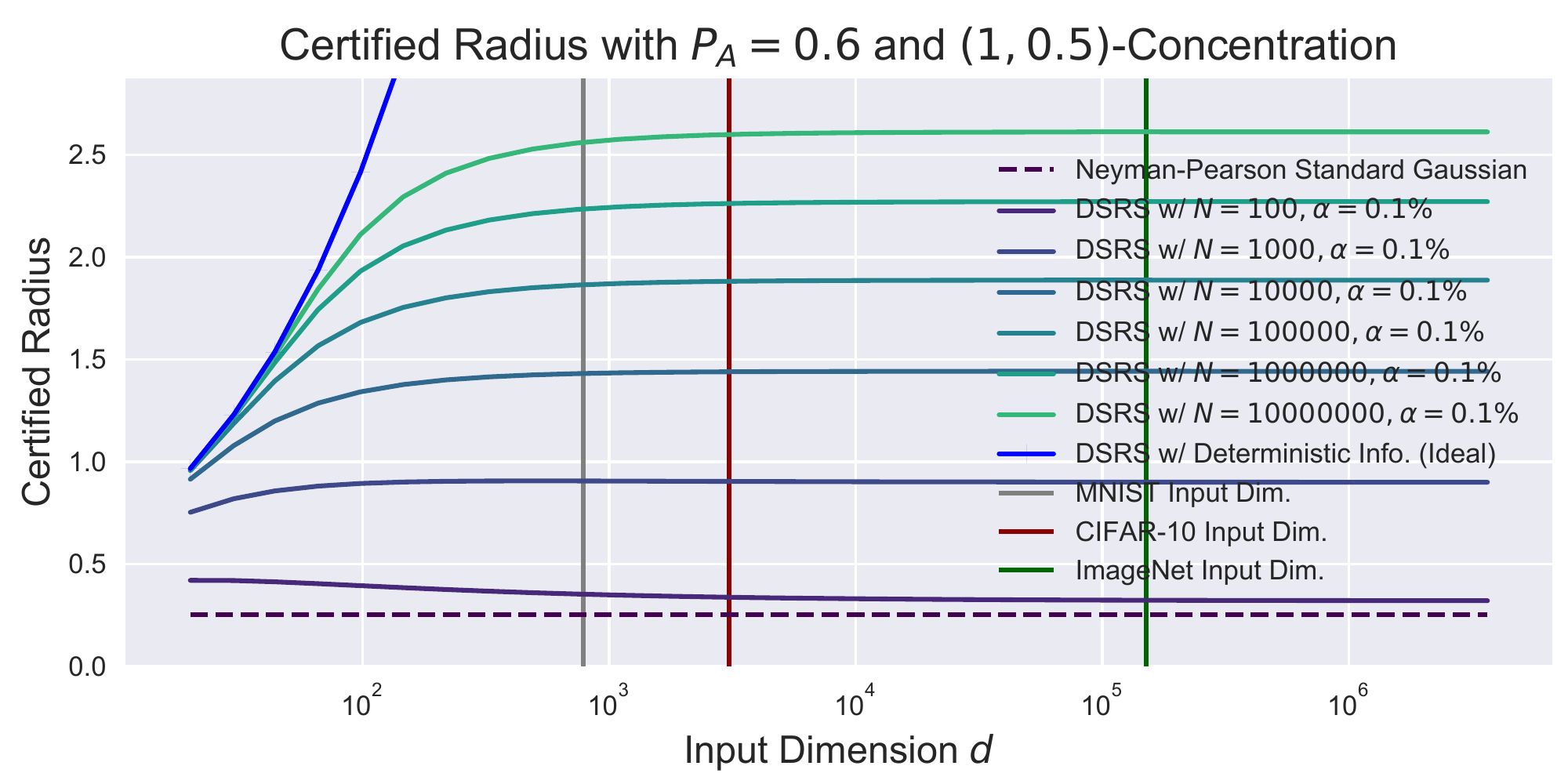}
        }
        \vspace{-1em} \\
        \subfigure[Relation between certified radius~($x$-axis) and certified accuracy~($y$-axis) on ImageNet models.
            Different curves correspond to Neyman-Pearson and \shortApproach certification with different $N$s. Sampling error considered, confidence level $=99.9\%$.]{
            \label{fig:real-data-sampling}
            \includegraphics[width=\linewidth]{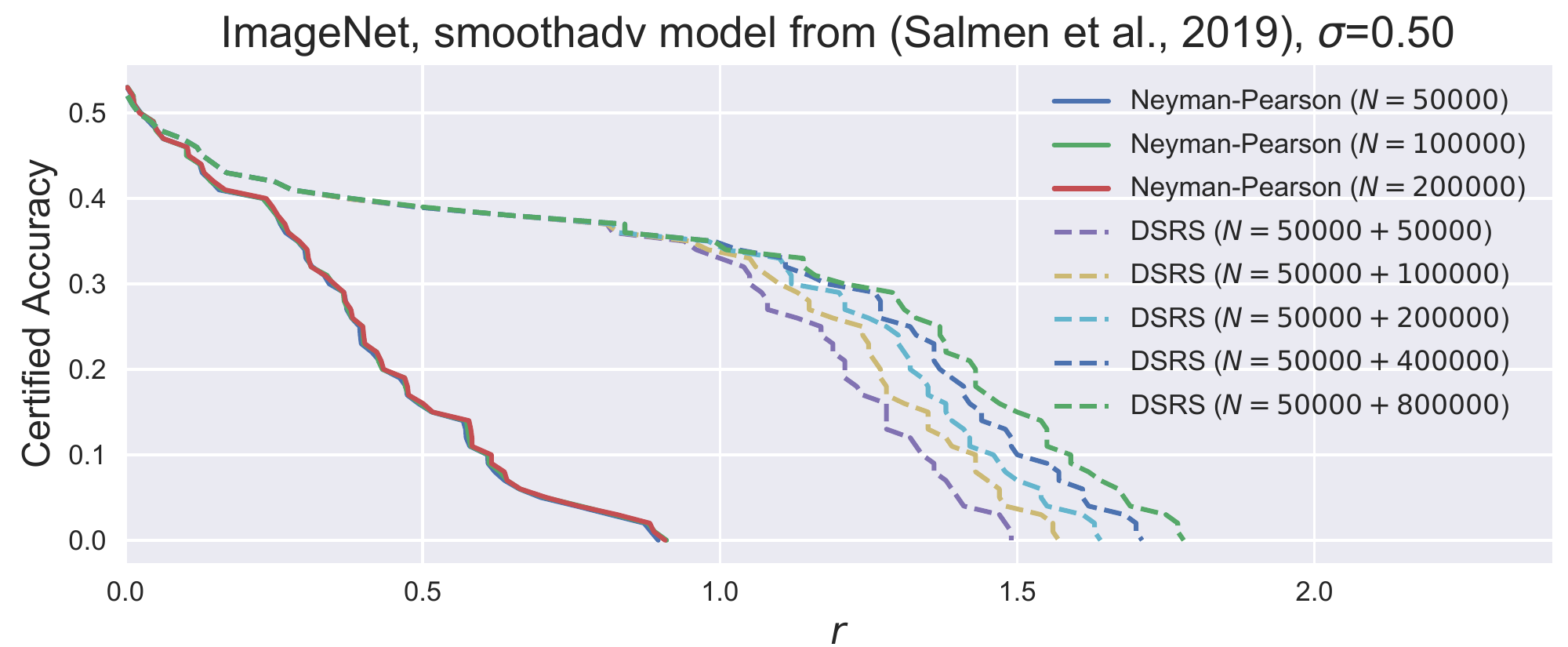}
        }
        \vspace{-1.2em}
        \caption{\small Tendency of \shortApproach certified robust radius considering sampling error.
        In both \textbf{(a)} and \textbf{(b)}, \shortApproach certified radius grows along with the increase of sampling number $N$ but Neyman-Pearson radius is almost fixed.
        }
        \label{fig:general-tendency-sampling}
        \vspace{-4mm}
    \end{figure}
        
    \vspace{-2mm}
    \subsection*{\shortApproach certifies tighter radius under general scenarios. }
    \vspace{-2mm}
    
        When the concentration property does not absolutely hold, a rigorous theoretical analysis becomes challenging, since the impact of a noninfinite dual variable needs to be taken into account. 
        This dual variable is inside a Lambert $W$ function where typical approximation bounds are too loose to provide non-trivial convergence rates.
        Thus, we leverage the numerical computational method introduced in \Cref{sec:dsrs} to provide numerical simulations and real-data experiments.
        We generalize the concentration assumption by changing the holding probability in \Cref{eq:concentration} from $1$ to $\alpha^{1/N}$, which corresponds to $(1-\alpha)$-confident lower bound of $Q_A$ given $N$ times of Monte-Carlo sampling, where we set $\alpha=0.1\%$ following the convention~\citep{cohen2019certified}.
        In this scenario, we compare \shortApproach certification with Neyman-Pearson certification numerically in \Cref{fig:general-tendency-sampling}~(numerical simulations in \Cref{fig:simulation-sampling} and ImageNet experiments in \Cref{fig:real-data-sampling}).
        
        In \Cref{fig:simulation-sampling}, we assume $(\sigma,\Pcon)$-concentration with $\sigma=1$, $\Pcon=0.5$ and different sampling number $N$s.
        We further assume $P_A = f^\gP(\vx_0)_{y_0} = 0.6$ as the true-class prediction probability under $\gP$.
        In \Cref{fig:real-data-sampling}, we take the model weights trained by \citet{salman2019provably} on ImageNet and apply generalized Gaussian smoothing with $d/2 - k = 4$ and $\sigma = 0.50$.
        We uniformly pick $100$ samples from the test set and compute $(1-\alpha)$-confident certified radius for each sample.
        We report certified accuracy~(under different $\ell_2$ radius $r$) that is the fraction of certifiably correctly classified samples by the smoothed classifier.
        
        \begin{remarkbox}
            When the sampling error and confidence interval come into play, they quickly suppress the $\Omega(\sqrt d)$ growth rate of \shortApproach certified radius~(\textcolor{blue}{blue} curve)  as shown in \Cref{fig:simulation-sampling}.
            Nonetheless, \shortApproach still certifies a larger radius than the standard Neyman-Pearson method and increasing the sampling number further enlarges the gap.
        \end{remarkbox}
        
        We consider another relaxed version of concentration property in \Cref{newadx:dsrs-relaxed-concentration}, where \shortApproach still provides significantly tighter robustness certification than Neyman-Pearson.

\section{\shortApproach Computational Method}

    \label{sec:dsrs}
    
    \begin{figure}[!t]
        \includegraphics[width=\linewidth]{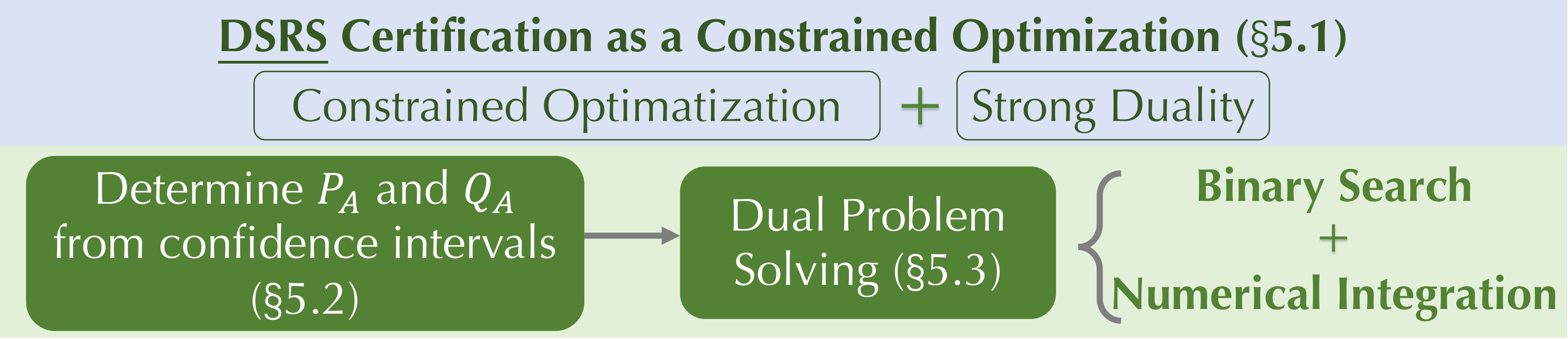}
        \vspace{-1.5em}
        \caption{\small Overview of \shortApproach computational method.}
        \label{fig:dsrs-structure}
        \vspace{-1.5em}
    \end{figure}

    The theoretical analysis in \Cref{sec:tightness-barrier} implies that \textit{additional smoothing distribution} $\gQ$ helps to tighten the robustness certification over standard Neyman-Pearson-based certification significantly.
    In this section, we propose an efficient computational method to compute this tight certified robust radius $\rDSRS$~(see \Cref{def:r-dsrs}) when $\gP$ is generalized Gaussian and $\gQ$ is truncated $\gP$ as suggested by \Cref{thm:concentration-sqrt-d,thm:concentration-constant-k-forbid-sqrt-d}.
    
    Compared with the classical certification for randomized smoothing or its variants~(cf. \citep{kumar2020certifying}), incorporating additional information raises a big challenge:
    the Neyman-Pearson lemma~(\citeyear{neyman1933ix}) can no longer be served as the foundation of the certification algorithm due to its incapability to handle the additional information.

        Thus, we propose a novel \shortApproach computational method by formalizing robustness certification as a constrained optimization problem and proving its strong duality~($\S$\ref{subsec:cons-opt-formulation}).
        Then, we propose an efficient algorithm to solve this specific dual optimization problem considering sampling error.
        The detailed algorithm can be found in \Cref{alg:DSRS-pipeline} in \Cref{newadx:sub:dsrs-core-algorithm}:
        1) we first perform a binary search on the certified radius $r$ to determine the maximum radius that we can certify; 
        2) 
        for current $r$, we determine the smoothed prediction confidence $P_A$ and $Q_A$ from the confidence intervals of predicting the true class~($\S$\ref{subsec:preprocessing}); 
        3) then, for current $r$ we solve the dual problem by quick binary search for dual variables $\lambda_1$ and $\lambda_2$~(see \Cref{eq:dual}) along with numerical integration~($\S$\ref{subsec:joint-binary-search}).
        To guarantee the soundness of numerical-integration-based certification, we take the maximum possible error into account during the binary search.
        We will discuss further extensions in $\S$\ref{subsec:extensions}.

    \subsection{\shortApproach as Constrained Optimization}
        \label{subsec:cons-opt-formulation}
        \vspace{-2mm}
        We first formulate the robustness certification as a constrained optimization problem and then show several foundational properties of the problem.
        
        Following the notation of \Cref{def:r-dsrs}, from the given base classifier $F_0$, we can use Monte-Carlo sampling to obtain 
        \begin{equation}
            \label{eq:first-appear-p-a-q-a}
            P_A = f_0^\gP(\vx_0)_{y_0},\quad Q_A = f_0^\gQ(\vx_0)_{y_0}. 
        \end{equation}
        In $\S$\ref{subsec:preprocessing} we will discuss how to handle confidence intervals of $P_A$ and $Q_A$.
        For now, we assume $P_A$ and $Q_A$ are fixed.
        
        Given perturbation vector $\vdelta\in \sR^d$, to test whether smoothed classifier $\wF_0^\gP$ still predicts true label $y_0$, we only need to check whether the prediction probability $f_0^\gP(\vx_0 + \vdelta)_{y_0} > 0.5$. This can be formulated as a constrained optimization problem $(\tC)$:
        \begin{subequations}
            \allowdisplaybreaks
            \label{eq:primal}
            \begin{align}
                \hspace{-0.9em} \underset{f}{\mathrm{minimize}} \quad & \E_{\vepsilon\sim\gP} [f(\vepsilon + \vdelta)] \label{eq:primal-opt} \\
                \mathrm{s.t.}\quad & \E_{\vepsilon\sim\gP} [f(\vepsilon)] = P_A,  \quad 
                \E_{\vepsilon\sim\gQ} [f(\vepsilon)] = Q_A, \label{eq:primal-pa-qa} \\
                & 0 \le f(\vepsilon) \le 1 \quad \forall \vepsilon\in\sR^d. \label{eq:primal-f-cons}
            \end{align}
        \end{subequations}
        \begin{remark}
            $(\tC)$ seeks for the minimum possible $f^\gP(\vx_0 + \vdelta)_{y_0}$ given \Cref{eq:first-appear-p-a-q-a}'s constraint.
            Concretely, we let $f$ represent whether the base classifier predicts label $y_0$: $f(\cdot) = \1[F(\cdot\, + \vx_0) = y_0]$, and accordingly impose $f \in [0,\,1]$ in \Cref{eq:primal-f-cons}.
            Then, \Cref{eq:primal-opt,eq:primal-pa-qa} unfold $f^\gP(\vx_0 + \vdelta)_{y_0}, f^\gP(\vx_0)_{y_0},$ and $f^\gQ(\vx_0)_{y_0}$ respectively and impose \Cref{eq:first-appear-p-a-q-a}'s constraint.
        \end{remark}
        
        We let $\tC_\vdelta(P_A,\,Q_A)$ denote the optimal value of \Cref{eq:primal} when feasible.
        Thus, under norm $p$,  to certify the robustness within radius $r$, we only need to check whether
        \begin{equation}
            \forall \vdelta, \|\vdelta\|_p < r \Rightarrow \tC_\vdelta(P_A,\,Q_A) > 0.5.
            \label{eq:check-cond}
        \end{equation}
        This formulation yields the tightest robustness certification given information from $\gP$ and $\gQ$ under the binary setting.
        Under the multiclass setting, there are efforts towards tighter certification by using ``$>$ maximum over other classes'' instead of ``$> 0.5$'' in \Cref{eq:check-cond}~\cite{dvijotham2020framework}.
        For saving the sampling cost and also to follow the convention~\cite{cohen2019certified,yang2020randomized,jeong2020consistency,zhai2019macer}, we mainly consider ``$> 0.5$'' for multiclass setting, and extension to the other form is straightforward.
        
        Since our choices of $\gP$ and $\gQ$~(standard/generalized (truncated) Gaussian) are isotropic and centered around origin, when certifying radius $r$, for $\ell_2$ certification we only need to test $\tC_\vdelta(P_A,\,Q_A) > 0.5$ with $\vdelta = (r,\,0,\,\dots,\,0)^\T$; 
        and 
        for $\ell_\infty$ we only need to divide $\ell_2$ radius by $\sqrt{d}$.
        This trick can also be extended for $\ell_1$ case~\cite{zhang2020black}.
        
        Directly solving $(\tC)$ is challenging. Thus, we construct the Lagrangian dual problem $(\tD)$:
        \vspace{-0.45em}
        \begin{subequations}
            \label{eq:dual}
            \begin{equation}
                \hspace{-0.2em} \underset{\lambda_1,\,\lambda_2 \in \sR}{\mathrm{maximize}} \, \Pr_{\vepsilon\sim\gP} [p(\vepsilon) < \lambda_1 p(\vepsilon+\vdelta) + \lambda_2 q(\vepsilon+\vdelta)]
                \label{eq:dual-opt}
            \end{equation}
            \vspace{-1.8em}
            \begin{equation}
                \begin{aligned}
                \mathrm{s.t.} & \Pr_{\vepsilon\sim\gP} [p(\vepsilon-\vdelta) < \lambda_1p(\vepsilon) + \lambda_2q(\vepsilon)] = P_A, \\
                & \Pr_{\vepsilon\sim\gQ} [p(\vepsilon-\vdelta) < \lambda_1p(\vepsilon) + \lambda_2q(\vepsilon)] = Q_A.
                \end{aligned}
            \end{equation}
        \end{subequations}
        In \Cref{eq:dual},  $p(\cdot)$ and $q(\cdot)$ are the density functions of distributions $\gP$ and $\gQ$ respectively.
        We let $\tD_\vdelta(P_A,\, Q_A)$ denote the optimal objective value to \Cref{eq:dual-opt} when it is feasible.
        \begin{theorem}
            For given $\vdelta\in \sR^d$, $P_A$, and $Q_A$,
            if $\tC$ and $\tD$ are both feasible, then $\tC_\vdelta(P_A,\, Q_A) = \tD_\vdelta(P_A,\, Q_A)$.
            \label{thm:strong-duality}
        \end{theorem}
        The theorem states the strong duality between $(\tC)$ and $(\tD)$.
        We defer the proof to \Cref{newadx:sub:strong-duality-proof}. The proof is based on min-max inequality and feasibility condition of $(\tD)$.
        Intuitively, we can view $(\tC)$, a functional optimization over $f$, as a linear programming~(LP) problem over infinite number of variables $\{f(\vx): \vx\in\R^d\}$ 
        so that the strong duality holds, which
        guarantees the tightness of \shortApproach in the primal space.
    
    \subsection{Dealing with Confidence Intervals}
        \label{subsec:preprocessing}
    
        It is practically intractable to know the exact $P_A$ and $Q_A$ in \Cref{eq:first-appear-p-a-q-a} by only querying the model's prediction for finite times.
        The common practice is using Monte-Carlo sampling, which gives confidence intervals of $P_A$ and $Q_A$ with a predefined confidence level $1 - \alpha$.
        
        Suppose we have confidence intervals $[\underline{P_A},\,\overline{P_A}]$ 
        and $[\underline{Q_A},\,\overline{Q_A}]$. 
        To derive a sound certification, we need to certify that for \emph{any} $P_A \in [\underline{P_A},\,\overline{P_A}]$ and \emph{any} $Q_A \in [\underline{Q_A},\,\overline{Q_A}]$, $\tC_\vdelta(P_A,\,Q_A) > 0.5$.
        Given the infinite number of possible $P_A$ and $Q_A$, the brute-force method is intractable. Here, \emph{without} computing $\tC_\vdelta$, we show how to solve
        \begin{equation}
            \hspace{-0.215em}
            (P_A,\,Q_A) = \argmin_{P_A' \in [\underline{P_A},\,\overline{P_A}],\\ Q_A' \in [\underline{Q_A},\,\overline{Q_A}]} \tC_\vdelta(P_A',\,Q_A').
            \label{eq:preprocess-q-a-q-a}
        \end{equation}
        If solved $P_A$ and $Q_A$ satify $\tC_\vdelta(P_A,\,Q_A) > 0.5$,  
        then for any $P_A$ and $Q_A$ within the confidence intervals, we can certify the robustness against perturbation $\vdelta$.
        We observe the following two properties of $\tC_\vdelta$.
        \begin{proposition}
            $\tC_\vdelta(\cdot,\,\cdot)$ is convex in the feasible region.
            \label{prop:convex}
        \end{proposition}
        \begin{proposition}
            With respect to $x \in [0,\,1]$,
            functions $x \mapsto \min_{y} \tC_\vdelta(x,\,y)$ and $x \mapsto \argmin_{y} \tC_\vdelta(x,\,y)$ are monotonically non-decreasing.
            Similarly, with respect to $y \in [0,\,1]$,
            functions $y \mapsto \min_{x} \tC_\vdelta(x,\,y)$ and $y \mapsto \argmin_{x} \tC_\vdelta(x,\,y)$ are monotonically non-decreasing.
            \label{prop:mono}
        \end{proposition}
        
        These two propositions characterize the landscape of $\tC_\vdelta(\cdot,\cdot)$---convex and monotonically non-decreasing along both $x$ and $y$ axes.
        Thus, desired ($P_A$, $Q_A$)~(location of minima within the bounded box) lies on the box boundary, and we only need to compute the location of boundary-line-sliced minima and compare it with box constraints to solve \Cref{eq:preprocess-q-a-q-a}.
        Formally, we propose an efficient algorithm~(\Cref{alg:determine-p-a-q-a}, omitted to \Cref{newadx:sub:dsrs-core-algorithm}) to solve $(P_A,Q_A)$.
        \begin{theorem}
            If \Cref{eq:preprocess-q-a-q-a} is feasible,
            the $P_A$ and $Q_A$ returned by \Cref{alg:determine-p-a-q-a} solve \Cref{eq:preprocess-q-a-q-a}.
            \label{thm:alg-2-good}
        \end{theorem}
        The above results are proved in \Cref{newadx:sub:preprocessing}.
        On a high level, we prove \Cref{prop:convex} by definition; we prove \Cref{prop:mono} via a reduction to classical Neyman-Pearson-based certification and analysis of this reduced problem; and we prove \Cref{thm:alg-2-good}
        based on \Cref{prop:convex,prop:mono} along with exhaustive and nontrivial analyses of all possible cases.
    
    \subsection{Solving the Dual Problem}
        \label{subsec:joint-binary-search}
    
        After the smoothed prediction confidences $P_A$ and $Q_A$ are determined from the confidence intervals, now we solve the dual problem $\tD_\vdelta(P_A,Q_A)$ as defined in \Cref{eq:dual}.
        We solve the problem based on the following theorem:
        \begin{theorem}[Numerical Integration for \shortApproach with Generalized Gaussian Smoothing]
            \label{thm:concrete-equations}
            In $\tD_\vdelta(P_A,Q_A)$, let $r = \|\vdelta\|_2$, 
            when $\gP = \gNg(k,\sigma)$ and $\gQ = \gNgtrunc(k,T,\sigma)$, let $\sigma' := \sqrt{d/(d-2k)}$ and let $\nu := \GammaCDF_{d/2-k}(T^2/(2\sigma'^2))$,
                \vspace{-1em}
            {
            \allowdisplaybreaks
            \scriptsize
            \begin{align*}
                R(\lambda_1,\lambda_2) & := \Pr_{\vepsilon\sim\gP} [p(\vepsilon) < \lambda_1 p(\vepsilon + \vdelta) + \lambda_2 q(\vepsilon + \vdelta)] \\
                = & \left\{ \begin{array}{lr}
                \E_{t\sim\Gamma(d/2-k,1)} u_1(t), & \lambda_1 \le 0 \\
                \E_{t\sim\Gamma(d/2-k,1)} u_1(t) + u_2(t), & \lambda_1 > 0
                \end{array} \right. \mathrm{where} \\
                u_1(t) = & \BetaCDF_{\frac{d-1}{2}} \left(
                    \dfrac{\min\{T^2, 2\sigma'^2 k W(\frac{t}{k}e^{\frac t k} (\lambda_1 + \nu \lambda_2)^{\frac 1 k}) \}}{4r\sigma'\sqrt{2t}} \right. \\
                    & \hspace{3em} \left. - 
                    \dfrac{(\sigma'\sqrt{2t} - r)^2}
                    {4r\sigma'\sqrt{2t}}
                \right), \\
                u_2(t) = \hspace{1em} & \hspace{-1em} \max\left\{
                    \BetaCDF_{\frac{d-1}{2}} \left(
                    \dfrac{2\sigma'^2k W(\frac{t}{k} e^{\frac t k} \lambda_1^{\frac 1 k}) - (\sigma'\sqrt{2t} - r)^2}{4r\sigma'\sqrt{2t}} 
                    \right) 
                \right. \\
                & \hspace{3em} \left. -\BetaCDF_{\frac{d-1}{2}} \left(
                    \dfrac{T^2 - (\sigma'\sqrt{2t} - r)^2}{4r\sigma'\sqrt{2t}}
                \right), 0 \right\}, \\
                P(\lambda_1,\lambda_2) & := \Pr_{\vepsilon\sim\gP} [p(\vepsilon-\vdelta) < \lambda_1 p(\vepsilon) + \lambda_2 q(\vepsilon)] \\
                & \hspace{-2em} = \E_{t\sim\Gamma(d/2-k,1)} \left\{ \begin{array}{lr}
                u_3(t,\lambda_1), & t \ge T^2 / (2\sigma'^2) \\
                u_3(t,\lambda_1 + \nu\lambda_2), & t < T^2 / (2\sigma'^2).
                \end{array} \right. \mathrm{where} \\
                u_3(t,\lambda) & = \BetaCDF_{\frac{d-1}{2}} \left(
                \dfrac{(r + \sigma'\sqrt{2t})^2}{4r\sigma'\sqrt{2t}}
                - \dfrac{2k\sigma'^2 W(\frac{t}{k}e^{\frac t k}\lambda^{-\frac 1 k})}{4r\sigma'\sqrt{2t}} \right), \\
                Q(\lambda_1,\lambda_2) & := \Pr_{\vepsilon\sim\gQ} [p(\vepsilon-\vdelta) < \lambda_1 p(\vepsilon) + \lambda_2 q(\vepsilon)] \\
                & = \nu \E_{t\sim\Gamma(d/2-k,1)} u_3(t,\lambda_1 + \nu\lambda_2) \cdot \1[t \le T^2/(2\sigma'^2)].
            \end{align*}
            }
            %
            In above equations, $\Gamma(d/2-k,1)$ is gamma distribution and $\GammaCDF_{d/2-k}$ is its CDF, 
            $\BetaCDF_{\frac{d-1}{2}}$ is the CDF of distribution $\Beta(\frac{d-1}{2}, \frac{d-1}{2})$, 
            and $W$ is the principal branch of Lambert $W$ function.
        \end{theorem}
        When $\gP$ is standard Gaussian and $\gQ$ is truncated standard Gaussian, we derive similar expressions as detailed in \Cref{newadx:sub:cert-std-truncated-gaussian}.
        We prove \Cref{thm:concrete-equations} in \Cref{newadx:sub:proof-concrete-equations}. The proof extends the level-set sliced integration and results from \citep{yang2020randomized}.
        With the theorem, we can rewrite the dual problem $\tD_\vdelta(P_A,Q_A)$ as 
        \begin{equation}
            \label{eq:dual-2} 
                \underset{\lambda_1,\,\lambda_2 \in \sR}{\mathrm{max}}\,  R(\lambda_1,\,\lambda_2) 
                \,
                \mathrm{s.t.} \,  P(\lambda_1,\,\lambda_2) = P_A,\, Q(\lambda_1,\,\lambda_2) = Q_A,
        \end{equation}
        Given concrete $\lambda_1$ and $\lambda_2$, from the theorem, these function values $P(\lambda_1,\lambda_2)$, $Q(\lambda_1,\lambda_2)$, and $R(\lambda_1,\lambda_2)$ can be easily computed with one-dimensional numerical integration using \texttt{SciPy} package.
        
        Now, solving $\tD_\vdelta(P_A,Q_A)$ reduces to finding dual variables $\lambda_1$ and $\lambda_2$ such that $P(\lambda_1,\lambda_2) = P_A$ and $Q(\lambda_1,\lambda_2) = Q_A$.
        Generally, we find that there is only one unique feasible pair $(\lambda_1,\lambda_2)$ for \Cref{eq:dual-2}, so finding out such a pair is sufficient.
        We prove the uniqueness and discuss how we deal with edge cases where multiple feasible pairs exist in \Cref{adxsubsec:main-approach-uniqueness}.
        
        Normally, such solving process is expensive.
        However, we find a particularly efficient method to solve $\lambda_1$ and $\lambda_2$ and the algorithm description is in \Cref{alg:dual-binary-search}~(in \Cref{newadx:sub:dsrs-core-algorithm}).
        At a high level, from \Cref{thm:concrete-equations}, we observe that $Q(\lambda_1,\lambda_2)$ is determined only by the sum $(\lambda_1 + \nu\lambda_2)$ and non-decreasing w.r.t. this sum.
        Therefore, we apply binary search to find out $(\lambda_1 + \nu\lambda_2)$ that satisfies $Q(\lambda_1,\lambda_2) = Q_A$.
        Then, we observe that 
        \vspace{-4mm}
        \begin{equation*}
            \scriptsize
            P(\lambda_1,\lambda_2) - \frac{Q(\lambda_1,\lambda_2)}{\nu} = \overbrace{\underset{\rt\sim\Gamma(d/2-k,1)}{\E} u_3(\rt,\lambda_1) \cdot \1\left[\rt \ge \frac{T^2}{2\sigma'^2}\right]}^{:=h(\lambda_1)}.
            \vspace{-2mm}
        \end{equation*}
        Thus, we need to find $\lambda_1$ such that $h(\lambda_1) = P_A - Q_A / \nu$.
        We observe that $h(\lambda_1)$ is non-decreasing w.r.t. $\lambda_1$, and we use binary search to solve $\lambda_1$.
        Combining with the value of $(\lambda_1 + \nu\lambda_2)$, we also obtain $\lambda_2$.
        We lastly leverage numerical integration to compute $R(\lambda_1,\lambda_2)$ following \Cref{thm:concrete-equations} to solve the dual problem $\tD_\vdelta(P_A,Q_A)$.
        
        \vspace{-0.5em}
        \paragraph{Practical Certification Soundness.}
        As a practical certification method, we need to guarantee the certification soundness in the presence of numerical error.
        In \shortApproach, there are two sources of numerical error: numerical integration error when computing $P(\lambda_1,\lambda_2)$, $Q(\lambda_1,\lambda_2)$, and $R(\lambda_1,\lambda_2)$, and the finite precision of binary search on $\lambda_1$ and $\lambda_2$.
        For numerical integration, we notice that typical numerical integration packages such as \texttt{scipy} support setting an absolute error threshold $\Delta$ and raising warnings when such threshold cannot be reached.
        We set the absolute threshold $\Delta=1.5\times 10^{-8}$, and abstain when the threshold cannot be reached~(which never happens in our experimental evaluation).
        Then, when computing $P$, $Q$, and $R$, suppose the numerical value is $v$, we use the lower bound $(v-\Delta)$ and upper bound $(v+\Delta)$ in the corresponding context to guarantee the soundness.
        For the finite precision in binary search, we use the left endpoint or the right endpoint of the final binary search interval to guarantee soundness.
        For example, we use the left endpoint of $\lambda_1$ in $R$ computation, and use the left endpoint of $(\lambda_1 + \nu\lambda_2)$ minus right endpoint of $\lambda_1$ to get the lower bound of $\lambda_2$ to use in $R$ computation.
        As a result, we always get an under-estimation of $R$ so the certification is sound.
        Further discussion is in \Cref{newadxsec:guaranteeing-soundness}.
        
        To this point, we have introduced the \shortApproach computational method.
        Complexity and efficiency analysis is omitted to \Cref{newadx:complexity-efficiency}.
        Implementation details are in \Cref{newadxsec:impl-details}.

    \subsection{Extensions}
        \label{subsec:extensions}
        
        We mainly discussed \shortApproach computational method for generalized Gaussian $\gP$ and truncated generalized Gaussian $\gQ$ under $\ell_2$ norm.
        Can we extend it to other settings?
        Indeed, \shortApproach is a general framework.
        In appendices, we show following extensions:
        (1)~\shortApproach for generalized Gaussian with different variances as $\gP$ and $\gQ$~(in \Cref{newadx:sub:cert-gaussian-diff-var});
        (2)~\shortApproach for other $\ell_p$ norms~(in \Cref{newadx:sub:cert-other-lp});
        and (3)~\shortApproach that leverages other forms of additional information covering gradient magnitude information~\citep{mohapatra2020higherorder,levine2020tight}~(in \Cref{newadx:general}).

\begin{table*}[!t]
    \centering
    \caption{\small Certified robust accuracy under different radii $r$ with different certification approaches.}
    
    
    \resizebox{0.98\linewidth}{!}{
    \begin{tabular}{c|c|c|cccccccccccc}
        \toprule
        \multirow{2}{*}{Dataset} & Training & Certification & \multicolumn{12}{c}{Certified Accuracy under Radius $r$} \\
        & Method & Approach &  0.25  &  0.50  &  0.75  &  1.00  &  1.25  &  1.50  &  1.75  &  2.00  &  2.25  &  2.50  &  2.75  &  3.00  \\
        \hline\hline
        \multirow{6}{*}{MNIST} & Gaussian Aug. & Neyman-Pearson & 
        \bf 97.9\% &	96.9\% &	94.6\% &	88.4\% &	78.7\% &	57.6\% &	41.0\% &	25.5\% &	13.6\% &	6.2\% &	2.1\% &	0.9\% \\
        & \citep{cohen2019certified} & \bf DSRS & 
        \bf 97.9\% &	\bf 97.0\% &	\bf 95.0\% &	\bf 89.8\% &	\bf 83.4\% &	\bf 61.6\% &	\bf 48.4\% &	\bf 34.1\% &	\bf 21.0\% &	\bf 10.6\% &	\bf 4.4\% &	\bf 1.2\% \\
        \cline{2-15}
        & Consistency & Neyman-Pearson &
        \bf 98.4\% &	\bf 97.5\% &	\bf 96.0\% &	92.3\% &	83.8\% &	67.5\% &	49.1\% &	35.6\% &	21.7\% &	10.4\% &	4.1\% &	1.9\% \\
        & \citep{jeong2020consistency} & \bf DSRS & 
        \bf 98.4\% &	\bf 97.5\% &	\bf 96.0\% &	\bf 93.5\% &	\bf 87.1\% &	\bf 71.8\% &	\bf 55.8\% &	\bf 41.9\% &	\bf 31.4\% &	\bf 17.8\% &	\bf 8.6\% &	\bf 2.8\% \\
        \cline{2-15}
        & SmoothMix & Neyman-Pearson & 
        \bf 98.6\% &	97.6\% &	96.5\% &	91.9\% &	85.1\% &	73.0\% &	51.4\% &	40.2\% &	31.5\% &	22.2\% &	12.2\% &	4.9\% \\
        & \citep{jeong2021smoothmix} & \bf DSRS & 
        \bf 98.6\% &	\bf 97.7\% &	\bf 96.8\% &	\bf 93.4\% &	\bf 87.5\% &	\bf 76.6\% &	\bf 54.4\% &	\bf 46.2\% &	\bf 37.6\% &	\bf 29.2\% &	\bf 18.5\% &	\bf 7.2\% \\
        \hline\hline
        
        \multirow{6}{*}{CIFAR-10} & 
        Gaussian Aug. & Neyman-Pearson & 
        56.1\% &	41.3\% &	27.7\% &	18.9\% &	14.9\% &	10.2\% &	7.5\% &	4.1\% &	2.0\% &	0.7\% &	0.1\% &	\bf 0.1\% \\
        & \citep{cohen2019certified} & \bf DSRS &
        \bf 57.4\% &	\bf 42.7\% &	\bf 30.6\% &	\bf 20.6\% &	\bf 16.1\% &	\bf 12.5\% &	\bf 8.4\% &	\bf 6.4\% &	\bf 3.5\% &	\bf 1.8\% &	\bf 0.7\% &	\bf 0.1\% \\
        \cline{2-15}
        & Consistency & Neyman-Pearson &
        61.8\% &	50.9\% &	38.0\% &	32.3\% &	23.8\% &	19.0\% &	16.4\% &	13.8\% &	11.2\% &	9.0\% &	7.1\% &	5.1\% \\
        & \citep{jeong2020consistency} & \bf DSRS &
        \bf 62.5\% &	\bf 52.5\% &	\bf 38.7\% &	\bf 35.2\% &	\bf 28.1\% &	\bf 20.9\% &	\bf 17.6\% &	\bf 15.3\% &	\bf 13.1\% &	\bf 10.9\% &	\bf 8.9\% &	\bf 6.5\% \\
        \cline{2-15}
        & SmoothMix & Neyman-Pearson &
        63.9\% &	53.3\% &	40.2\% &	34.2\% &	26.7\% &	20.4\% &	17.0\% &	13.9\% &	10.3\% &	7.8\% &	4.9\% &	2.3\% \\
        & \citep{jeong2021smoothmix} & \bf DSRS &
        \bf 64.7\% &	\bf 55.5\% &	\bf 42.1\% &	\bf 35.9\% &	\bf 29.4\% &	\bf 22.1\% &	\bf 18.7\% &	\bf 16.1\% &	\bf 13.2\% &	\bf 10.2\% &	\bf 7.1\% &	\bf 3.9\% \\
        \hline\hline
        
        \multirow{6}{*}{ImageNet} &
        Guassian Aug. & Neyman-Pearson &
        57.1\% &	47.0\% &	39.3\% &	33.2\% &	24.8\% &	21.4\% &	17.6\% &	13.7\% &	10.2\% &	7.8\% &	5.7\% &	3.6\% \\
        & \citep{cohen2019certified} & \bf DSRS &
        \bf 58.4\% &	\bf 48.4\% &	\bf 41.4\% &	\bf 35.3\% &	\bf 28.8\% &	\bf 23.3\% &	\bf 21.3\% &	\bf 18.7\% &	\bf 14.2\% &	\bf 11.0\% &	\bf 9.0\% &	\bf 5.7\% \\
        \cline{2-15}
        & Consistency & Neyman-Pearson &
        59.8\% &	49.8\% &	43.3\% &	36.8\% &	31.4\% &	25.6\% &	22.1\% &	19.1\% &	16.1\% &	14.0\% &	10.6\% &	8.5\% \\
        & \citep{jeong2020consistency} & \bf DSRS &
        \bf 60.4\% &	\bf 52.4\% &	\bf 44.7\% &	\bf 39.3\% &	\bf 34.8\% &	\bf 28.1\% &	\bf 25.4\% &	\bf 22.6\% &	\bf 19.6\% &	\bf 17.4\% &	\bf 14.1\% &	\bf 10.4\% \\
        \cline{2-15}
        & SmoothMix & Neyman-Pearson & 
        46.7\% &	38.2\% &	28.8\% &	24.6\% &	18.1\% &	14.2\% &	11.8\% &	10.1\% &	8.9\% &	7.2\% &	6.0\% &	4.6\% \\
        & \citep{jeong2021smoothmix} & \bf DSRS &
        \bf 47.4\% &	\bf 40.0\% &	\bf 30.3\% &	\bf 26.8\% &	\bf 21.6\% &	\bf 15.7\% &	\bf 14.0\% &	\bf 12.1\% &	\bf 9.9\% &	\bf 8.4\% &	\bf 7.2\% &	\bf 5.3\% \\
        \bottomrule
    \end{tabular}
    }
    \label{tab:ell-2}
    \vspace{-1em}
\end{table*}

\vspace{-2mm}

\section{Experimental Evaluation}
    \label{sec:exp}
    
    In this section, we systematically evaluate \shortApproach and demonstrate that it achieves tighter certification than the classical Neyman-Pearson-based certification against $\ell_2$ perturbations on MNIST, CIFAR-10, and ImageNet.
    We focus on $\ell_2$ certification because additive randomized smoothing is not optimal for other norms~(e.g., $\ell_1$~\citep{levine2021improved}) or the certification can be directly translated from $\ell_2$ certification~(e.g., $\ell_\infty$~\citep{yang2020randomized} and semantic transformations~\citep{li2020transformation}).
    
    \subsection{Experimental Setup}
        \label{subsec:exp-setup}
        
        \vspace{-2mm}
        \noindent\textbf{Smoothing Distributions.}
        Following \Cref{thm:concentration-sqrt-d}, we use generalized Gaussian $\gNg(k,\sigma)$ as smoothing distribution $\gP$ where $d/2-15 \le k < d/2$.
        Specifically,
        we set $k$ to be $d/2-12$ on MNIST, $d/2-6$ on CIFAR-10, and $d/2-4$ on ImageNet.
        We use three different $\sigma$'s: $0.25$, $0.50$, and $1.00$. 
        
        In terms of the additional smoothing distribution $\gQ$,
        on MNIST and CIFAR-10, we empirically find that using generalized Gaussian with the same $k$ but different variance yields tighter robustness certification, and therefore we choose $\sigma_g$ to be $0.2$, $0.4$, and $0.8$ corresponding to $\gP$'s $\sigma$ being $0.25$, $0.50$, and $1.0$, respectively.
        On ImageNet, the concentration property~(see \Cref{def:concentration}) is more pronounced~(detail study in \Cref{newadx:sub:study-concentration}) and thus we use truncated generalized Gaussian $\gNgtrunc(k,T,\sigma)$ as $\gQ$.
        We apply a simple but effective algorithm as explained in \Cref{newadx:impl} to determine hyperparameter $T$ in $\gNgtrunc(k,T,\sigma)$.
        
        \vspace{-2mm}
        \noindent\textbf{Models and Training.}
        We consider three commonly-used or state-of-the-art training methods:
        Gaussian augmentation~\citep{cohen2019certified}, 
        Consistency~\citep{jeong2020consistency},
        and SmoothMix~\citep{jeong2021smoothmix}.
        We follow the default model architecture on each dataset respectively.
        We train the models with augmentation noise sampled from the corresponding generalized Gaussian smoothing distribution $\gP$.
        More training details can be found in \Cref{newadx:impl}.

        \vspace{-2mm}
        \noindent\textbf{Baselines.}
        We consider the Neyman-Pearson-based certification method as the baseline.
        This certification is widely used and is the tightest given only prediction probability under $\gP$~\cite{cohen2019certified,yang2020randomized,jeong2020consistency,li2020transformation}. 
        We remark that although there are certification methods that leverage more information, to the best of our knowledge, they are not visibly better than the Neyman-Pearson-based method on $\ell_2$ certification under practical sampling number~($10^5$).
        More comparisons in \Cref{newadx:sub:compare-higher-order} show \shortApproach is also better than these baselines.
        
        For both baseline and \shortApproach,
        following the convention, the certification confidence is $1 - \alpha = 99.9\%$, and we use $10^5$ samples for estimating $P_A$ and $Q_A$.
        Neyman-Pearson certification does not use the information from additional distribution, and all $10^5$ samples are used to estimate the interval of $P_A$.
        In \shortApproach, we use $5\times 10^4$ samples to estimate the interval of $P_A$ with confidence $1-\frac{\alpha}{2} = 99.95\%$ and the rest $5\times 10^4$ samples for $Q_A$ with the same confidence.
        By union bound, the whole certification confidence is $99.9\%$.
        
        \vspace{-2mm}
        \noindent\textbf{Metrics.}
        We uniformly draw $1000$ samples from the test set and report \emph{certified robust accuracy}~(under each $\ell_2$ radius $r$) that is the fraction of samples that are both correctly classified and have certified robust radii larger than or equal to $r$.
        Under each radius $r$, we report the highest certified robust accuracy among the three variances $\sigma \in \{0.25,0.50,1.00\}$ following \citep{cohen2019certified,salman2019provably}.
        We also report evaluation results with ACR metric~\citep{zhai2019macer} in \Cref{adxsubsec:acr}.
        
        
    \subsection{Evaluation Results}
        \vspace{-2mm}
        We show results in \Cref{tab:ell-2}.
        The corresponding curves and separated tables for each variance $\sigma$ are in \Cref{newadx:sub:curve-sepa-tables}.
        
        \textbf{For \emph{almost all} models and radii $r$, \shortApproach yields significantly higher certified accuracy.}
        For example, for Gaussian augmented models, when $r=2.0$, on MNIST the robust accuracy increases from $25.5\%$ to $34.1\%$~($+8.6\%$), 
        on CIFAR-10 from $4.1\%$ to $6.4\%$~($+2.3\%$),
        and on ImageNet from $13.7\%$ to $18.7\%$~($+5.0\%$).
        On average, on MNIST the improvements are around $6\%$ - $9\%$;
        on CIFAR-10, the improvements are around $1.5\%$ - $3\%$;
        and on ImageNet the improvements are around $2\%$ - $5\%$.
        Thus, \shortApproach can be used in conjunction with different training approaches and provides consistently tighter robustness certification.
        
        The improvements in the robust radius are not as substantial as those in \Cref{fig:real-data-sampling}~(which is around $2\times$).
        We investigate the reason in \Cref{newadx:sub:study-concentration}.
        In summary, the model in \Cref{fig:real-data-sampling} is trained with standard Gaussian smoothing augmentation and smoothed with generalized Gaussian.
        The models in this section are trained with generalized Gaussian augmentation. 
        Such training gives higher certified robustness, but in the meantime, gives more advantage to Neyman-Pearson-based certification.
        This finding implies that there may be a large space for exploring training approaches that favor \shortApproach certification since all existing training methods are designed for Neyman-Pearson-based certification.
        Nevertheless, even with these ``unsuitable'' training methods, \shortApproach still achieves significantly tighter robustness certification than the baseline.
        
        On the other hand, all the above results are restricted to generalized Gaussian smoothing.
        We still observe that standard Gaussian smoothing combined with strong training methods~\cite{salman2019provably,jeong2020consistency} and Neyman-Pearson certification~(the SOTA setting) yields similar or slightly higher certified robust accuracy than generalized Gaussian smoothing even with \shortApproach certification.
        Though \shortApproach with its suitable generalized Gaussian smoothing does not achieve SOTA certified robustness yet, given the theoretical advantages, we believe that with future tailored training approaches, \shortApproach with generalized Gaussian smoothing can bring strong certified robustness.
        More discussion is in \Cref{adxsubsec:limitation-future-directions}.
  
        \textbf{Ablation Studies.}
        We present several ablation studies in the appendix. and verify:
        (1)~Effectiveness of our simple heuristic for selecting hyperparameter for $\gQ$:
        We propose a simple heuristic to select the hyperparameter $T$ in smoothing distribution $\gQ = \gNgtrunc(k,T,\sigma)$.
        In \Cref{newadx:sub:t-heuristic-attempt-better-opt}, we propose a gradient-based optimization method to select such $\gQ$.
        We find that our simple heuristic has similar performance compared to the more complex optimization method but is more efficient.
        (2)~Comparison of different types of $\gQ$:
        by choosing different types of $\gQ$ distributions~(truncated Gaussian or Gaussian with different variance), \shortApproach has different performance as mentioned in \Cref{subsec:exp-setup}.
        In \Cref{newadx:sub:diff-var-dsrs}, we investigate the reason.
        In summary,
        when concentration property~(see \Cref{def:concentration}) is better satisfied, using truncated Gaussian as $\gQ$ is better; otherwise, using Gaussian with different variance is better.

\section{Conclusion}

    We propose a general \shortApproach framework that exploits information based on an additional smoothing distribution to tighten the robustness certification.
    We theoretically analyze and compare classical Neyman-Pearson and \shortApproach certification, showing that \shortApproach has the potential to break the curse of dimensionality of randomized smoothing.

\vspace{-0.5em}
\section*{Acknowledgements}
\vspace{-0.5em}
We thank the anonymous reviewers for their constructive feedback. 
This work is partially supported by 
 NSF grant No.1910100, NSF CNS No.2046726, C3 AI, and the Alfred P. Sloan Foundation.


\bibliography{bib}
\bibliographystyle{icml2022}

\newpage
\appendix
\onecolumn
\twocolumn

\allowdisplaybreaks

\ifnum\arxiv=1
\part*{Appendices}

\DoToC
\newpage
\fi

\section{Neyman-Pearson Certification}
    
    \label{newadx:detail-n-p-cert}
    
    The Neyman-Pearson-based robustness certification is the tightest certification given only prediction probability under $\gP$~\citep{cohen2019certified}.
    This certification and its equivalent variants are widely used for randomized smoothing.
    We use $\rNP$ to represent the certified radius from the Neyman-Pearson-based method.
    
    If the smoothing distribution $\gP$ is standard Gaussian, the following proposition gives the closed-form certified robust radius derived from the Neyman-Pearson lemma~\citep{neyman1933ix}.
    
    \begin{proposition}[\cite{cohen2019certified}]
        \label{prop:neyman-pearson}
        Under $\ell_2$ norm, given input $\vx_0 \in \sR^d$ with true label $y_0$.
        Let $\gP = \gN(\sigma)$ be the smoothing distribution, then Neyman-Pearson-based certification yields certified radius $\rNP = \sigma\Phi^{-1}(f^\gP(\vx_0)_{y_0})$, where $\Phi^{-1}$ is the inverse CDF of unit-variance Gaussian.
    \end{proposition}
    
    For other smoothing distributions, the concretization of the Neyman-Pearson certification method can be found in \citep{yang2020randomized}.
    
    \begin{remark}
        In practice, the routine is to use Monte-Carlo sampling to obtain a high-confidence interval of $f^\gP(\vx_0)_{y_0}$, which implies a high-confidence certification~($\rNP$) of robust radius.
        A tighter radius can be obtained when the runner-up prediction probability is known: $\rNP' = \frac{\sigma}{2}\left(\Phi^{-1}(f^\gP(\vx_0)_{y_0}) - \max_{y\in [C]: y\neq y_0} \Phi^{-1}(f^\gP(\vx_0)_y)\right)$. However, due to efficiency concern~(for $C$-way classification the sampling number needs to be more than $C$ times if using $\rNP'$ for certification instead of $\rNP$), the standard routine is to only use top-class probability and $\rNP$~\citep[Section 3.2.2]{cohen2019certified}.
        \shortApproach follows this routine.
    \end{remark}
    
    \vspace{-1em}
\section{Illustration of Concentration Assumption on ImageNet}

    \label{newadx:illustration-concentration-fig}
        \vspace{-1em}
        \begin{figure}[H]
            \centering
            \includegraphics[width=0.9\linewidth]{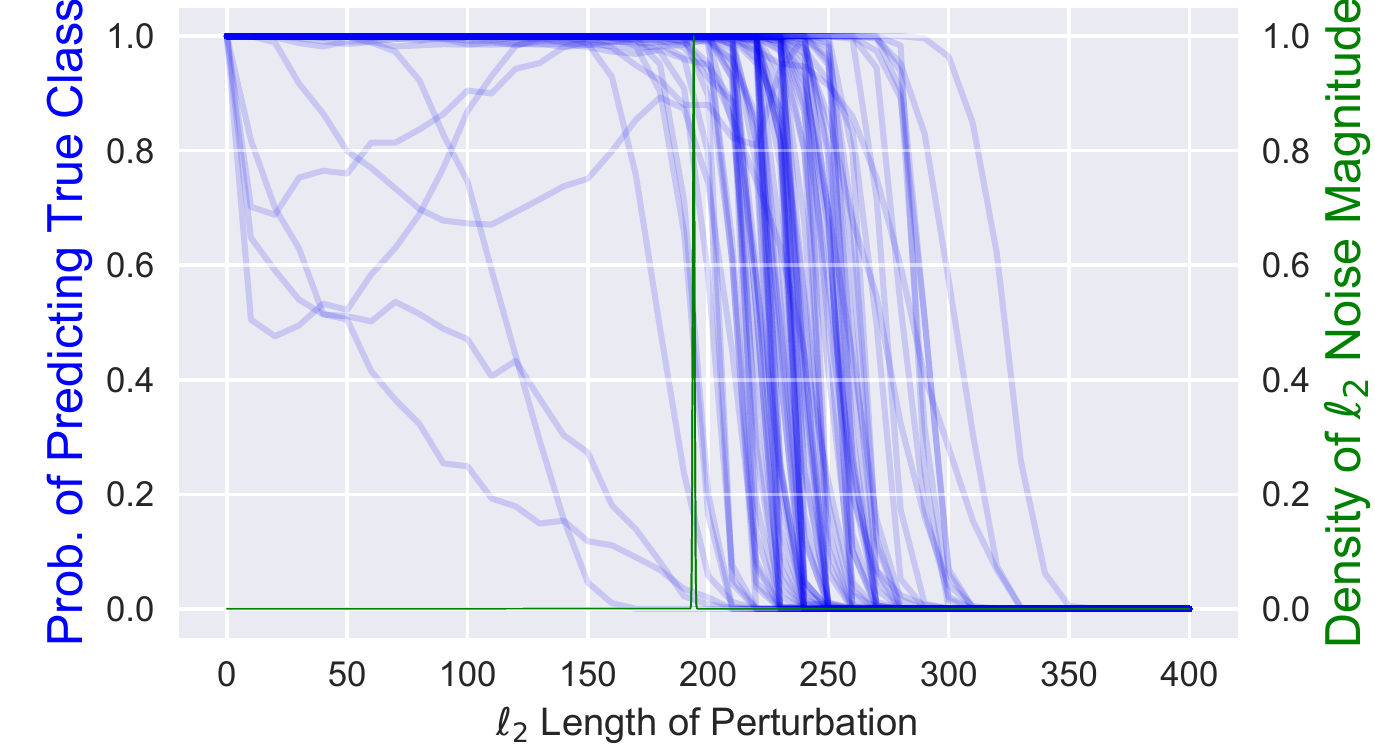}
            \vspace{-1em}
            \caption{\small 
            \textcolor{blue}{Blue curves}: Probability of true-prediction w.r.t. $\ell_2$ length of perturbations for a base classifier from \cite{salman2019provably} on ImageNet.
            Each line corresponds to one of $100$ uniformly drawn samples from test set~(detailed setup in \Cref{newadx:sub:study-concentration}).
            \textcolor{darkgreen}{Green curve}: Normalized density of $\ell_2$ noise magnitude for ImageNet standard Gaussian $\gN(\sigma)$ with $\sigma = 0.5$, which highly concentrates on $\sigma\sqrt{d}$. 
            \textit{Thus, for constant $\Pcon$, ($\sigma, \Pcon$)-concentration can be satisfied for a significant portion of input samples}.}
            \label{fig:smooth-model-landscape-and-magnitude-density}
            \vspace{1.0em}
        \end{figure}

\section{Illustration of Unchanged \texorpdfstring{$P_A$}{Pa} with Increasing \texorpdfstring{$d$}{d}}

    \label{newadx:unchanged-pa}
    
    In the first remark of \Cref{thm:concentration-sqrt-d}, we mention that $P_A = f^\gP(\vx_0)_{y_0}$ does not grow simultaneously along with the increase of input dimension $d$.
    In \Cref{fig:average-p-a}, to illustrate this phenomenon, we plot $P_A$ histograms for $1,000$ test samples from TinyImageNet and ImageNet.
    Note that TinyImageNet images are downscaled ImageNet images, so the data distribution only differs in the input dimension $d$.
    As we can observe, though $d$ varies, the $P_A$ distribution is highly similar, so $P_A$ is roughly unchanged along with the increase of input dimension $d$.
    
    \begin{figure}[!h]
        \centering
        \includegraphics[width=0.9\linewidth]{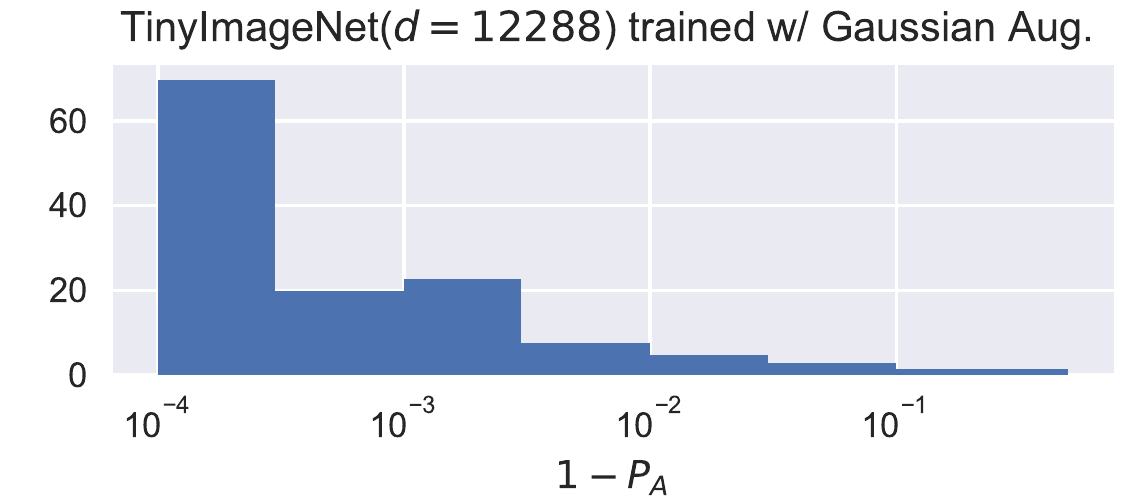}\\
        \includegraphics[width=0.9\linewidth]{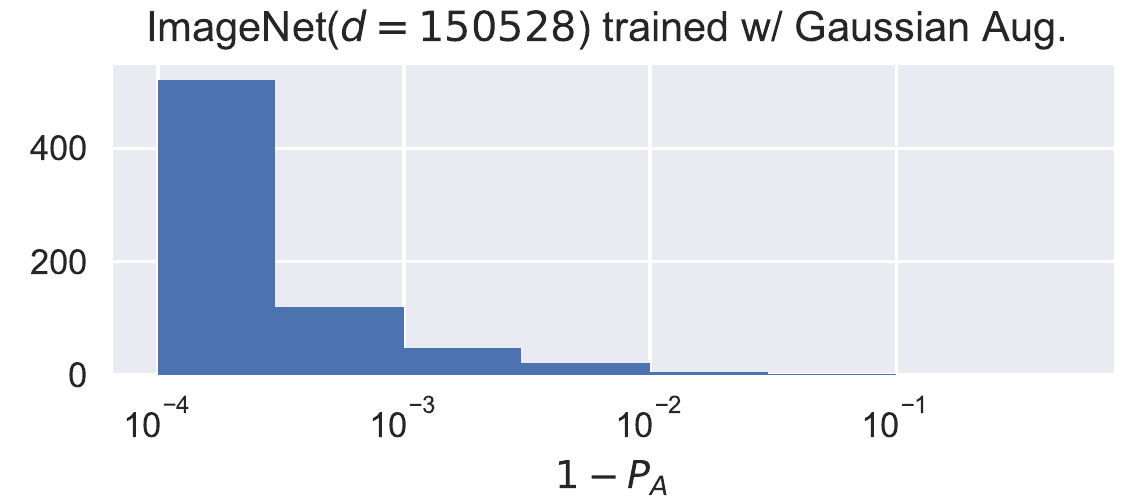} \\
        \caption{$P_A$ histograms for models trained on TinyImageNet~(left) and ImageNet~(right) with same $\sigma=0.50$.}
        \label{fig:average-p-a}
    \end{figure}
    
\section{\shortApproach under Relaxed Concentration Assumption}

    \label{newadx:dsrs-relaxed-concentration}
    
    In main text~(\Cref{sec:theory}), we generalize the concentration assumption by replacing the holding probability $1$ in \Cref{eq:concentration} by probability considering sampling confidence.
    In this appendix, we replace the holding probability in \Cref{eq:concentration} by $\exp(-d^\alpha)$ for $\alpha\in \{0.1, 0.2, 0.3, 0.4, 0,5\}$.
    
    With this relaxation,  we conduct numerical simulations using the same settings as in the main text, and the corresponding results are shown in \Cref{fig:simulation-a}.
    Note that some solid curves terminate when $d$ is large, which is due to the limitation of floating-point precision in numerical simulations, and we use dashed lines of the same color to plot the projected radius when $d$ is large.
    
    \begin{remarkbox}
        When the concentration property holds with probability $\exp(-d^\alpha)$~($0 < \alpha \le 0.5$) other than $1$, from \Cref{fig:simulation-a}, we observe that $\rNP d^{\alpha / 1.18}$ predicts the certified radius of \shortApproach well where $\rNP$ is Neyman-Pearson certified radius.
        Therefore, although the $\sqrt d$ growth rate of $\ell_2$ certified radius does not hold, the radius still increases along with the dimension $d$.
        Interestingly, along with the increase of dimension $d$, the vanishing probability $\exp(-d^\alpha)$ still implies the increasing volume of adversarial examples, and smoothed classifier is still certifiably robust with increasing radius reflected by \shortApproach despite the increasing adversarial volume.
    \end{remarkbox}

    \begin{figure}[t]
        \centering
         \includegraphics[width=\linewidth]{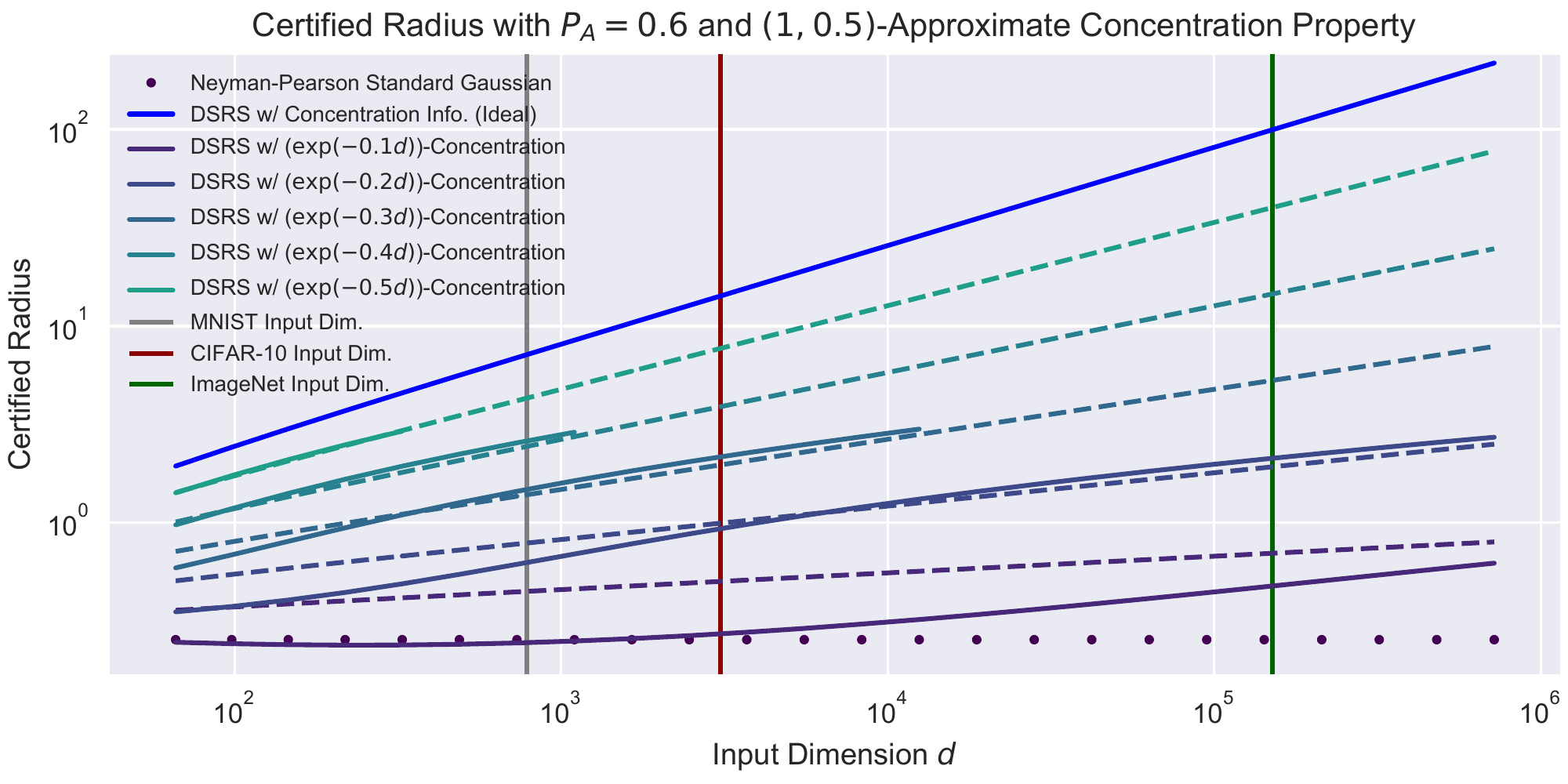}
         \vspace{-6mm}
         \caption{\small Tendency of \shortApproach certified robust radius with different input dimensions $d$ under relaxed concentration assumption:
         when holding probability in \Cref{eq:concentration} is $\exp(-d^\alpha)$ with $\alpha$ from $0.1$ to $0.5$;.
        \textcolor{blue}{Blue line}: \shortApproach when holding probability in \Cref{eq:concentration} is $1$. 
         Dotted line: Neyman-Pearson certification.
         Other solid lines: \shortApproach when holding probability in \Cref{eq:concentration} is $\exp(-d^\alpha)$.
         Other dashed lines: \shortApproach projected radius by $\rDSRS^{\text{proj}} = \rNP d^{\alpha / 1.18}$.
         $\alpha \in \{0.1,0.2,0.3,0.4,0.5\}$.
         Both $x$- and $y$-axes are logarithmic.}
         \label{fig:simulation-a}
    \end{figure}

\section{Omitted Details of \shortApproach Computational Method}

    \subsection{Algorithm Description}
        \label{newadx:sub:dsrs-core-algorithm}
        
        \begin{center}
        \begin{minipage}{0.9\linewidth}
        \begin{algorithm}[H]
            \setcounter{AlgoLine}{0}
            \SetCommentSty{footnotesize}
            \footnotesize\SetAlgoLined
                \KwData{Distributions $\gP$ and $\gQ$; $\vdelta$; $[\underline{P_A},\,\overline{P_A}]$ and $[\underline{Q_A},\,\overline{Q_A}]$}
                \KwResult{$P_A$ and $Q_A$ satisfying \Cref{eq:preprocess-q-a-q-a}}
                Compute $\underline{q} \gets \argmin_y \tC_\vdelta(\underline{P_A},\,y)$\;
                
                \uIf{$\underline{q} > \underline{Q_A}$}
                {\Return{$(\underline{P_A},\,\min\{\underline{q},\,\overline{Q_A}\})$}}
                \uElse{
                    Compute $\underline{p} \gets \argmin_x \tC_\vdelta(x,\,\underline{Q_A})$\;
                    
                    \Return{$(\max\{\min\{\underline{p},\,\overline{P_A}\}, \underline{P_A}\},\,\underline{Q_A})$}\;
                }
            \caption{Determining $P_A$ and $Q_A$ from confidence intervals~(see \S\ref{subsec:preprocessing}).}
            \label{alg:determine-p-a-q-a}
        \end{algorithm}
        \end{minipage}
        \end{center}
        
        \Cref{alg:determine-p-a-q-a} is a subroutine~(Line $7$) of \Cref{alg:DSRS-pipeline}.
        Note that Lines $1$ and $5$ of \Cref{alg:determine-p-a-q-a} solve the constrained optimization with only one constraint~(either one of \Cref{eq:primal-pa-qa}), reducing to the well-studied and solvable Neyman-Pearson-based certification.
        
        Note that we do not need to evaluate any value of $\tC_\vdelta$ in \Cref{alg:determine-p-a-q-a}.
        Although $\underline{q}$ and $\underline{p}$ in the algorithm are ``$\argmin_{x \text{ or } y}$'' over $\tC_\vdelta$, the free choices of $x$ or $y$ leave $\tC_\vdelta$'s constrained optimization with only one constraint and then $\underline{q}$ and $\underline{p}$ can be solved by Neyman-Pearson instead of evaluating $\tC_\vdelta$ directly.

        \begin{algorithm}[H]
            \footnotesize\setcounter{AlgoLine}{0}
            \SetCommentSty{footnotesize}
            \footnotesize
                \KwData{clean input $\vx_0$, base classifier $F_0$; distributions $\gP$ and $\gQ$; norm type $p$; confidence level $\alpha$; numerical integration error bound $\Delta$}
                \KwResult{Certified radius $r$}
                Query prediction $y_0 \gets \wF_0^\gP(\vx_0)$\;
                
                Sample and estimate the intervals of smoothed confidence $[\underline{P_A}, \overline{P_A}]$ under $\gP$ and $[\underline{Q_A}, \overline{Q_A}]$ under $\gQ$ with confidence $(1-\alpha)$ following \citep{cohen2019certified}\;
                
                Initialize: $r_l \gets 0, r_u \gets r_{\max}$\;
                
                \While(\tcc*[f]{Binary search on radius $r$}){$r_u - r_l > \mathsf{eps}$}{
                    $r_m \gets (r_l + r_u) / 2$\;
                    
                    $\vdelta \gets (r_m,\,0,\,\dots,\,0)^\T$
                    \tcc*[r]{for $\ell_2$ certification with $\ell_2$ symmetric $\gP$ and $\gQ$; for $\ell_\infty$ or $\ell_1$, can be adjusted following \citep{zhang2020black}}
                    Determine $P_A \in [\underline{P_A}, \overline{P_A}]$ and $Q_A \in [\underline{Q_A}, \overline{Q_A}]$
                    \tcc*[r]{See \Cref{subsec:preprocessing} and \Cref{alg:determine-p-a-q-a}}
                    $(\lambda_1,\,\lambda_2) \gets$ \textsc{DualBinarySearch}($P_A,\,Q_A$)
                    \tcc*[r]{See \Cref{subsec:joint-binary-search} and \Cref{alg:dual-binary-search}}
                    $v \gets R(\lambda_1,\,\lambda_2) - \Delta$
                    \tcc*[r]{Using \Cref{thm:concrete-equations}}
                    \eIf{$v> 0.5$}{
                        $r_l\gets r_m$
                    }{
                        $r_u\gets r_m$\;
                    }
                }
                \Return{$r_l$}\;
            \caption{\shortApproach computational method.}
            \label{alg:DSRS-pipeline}
        \end{algorithm}

        \Cref{alg:DSRS-pipeline} is the pseudocode of the whole \shortApproach computational method as introduced in \Cref{sec:dsrs}.
        
        \begin{center}
        \begin{minipage}{.9\linewidth}
        \begin{algorithm}[H]
            \SetCommentSty{footnotesize}
            \footnotesize\setcounter{AlgoLine}{0}
            \SetAlgoLined
                \KwData{Query access to $P(\cdot,\,\cdot)$ and $Q(\cdot,\,\cdot)$; $P_A$; $Q_A$; $\nu$; precision parameter $\epsilon$; numerical integration error bound $\Delta$}
                \KwResult{$\lambda_1$ and $\lambda_2$ satisfying constraints $P(\lambda_1,\lambda_2) = P_A, Q(\lambda_1,\lambda_2) = Q_A$~(see \Cref{eq:dual-2})}
                
                $a^L \gets 0$, $a^U \gets M$
                \tcc*[r]{search for $a = \lambda_1 + \nu\lambda_2$, $M$ is a large positive number}
                
                \While{$a^U - a^L > \epsilon$}{
                    $a^m \gets (a^L + a^U) / 2$\;
                    
                    \uIf{$Q(a^m,0) < Q_A$}
                    {
                        $a^L \gets a^m$
                    }
                    \uElse
                    {
                        $a^U \gets a^m$
                    }
                }
                
                \tcc{Following while-loop enlarges $a^L$ and $a^U$ until $[a^L, a^U]$ covers $a^*$ such that $Q(a^*,0) = Q_A$ under numerical integration error}
                \While{$(Q(a^L,0) + \Delta > Q_A)$ \text{ or } $(Q(a^U,0) - \Delta < Q_A)$}{
                    $t \gets a^U - a^L$\;
                    
                    $a^L \gets a^L - t/2$\;
                    
                    $a^U \gets a^U + t/2$\;
                } 
                
                $\lambda_1^L \gets 0$, $\lambda_1^U \gets M$
                \tcc*[r]{search for $\lambda_1$, $M$ is a large positive number}
                
                \While{$\lambda_1^U - \lambda_1^L > \epsilon$}{
                    $\lambda_1^m \gets (\lambda_1^L + \lambda_1^U) / 2$\;
                    
                    \uIf{$h(\lambda_1^m) - \Delta < P_A - Q_A / \nu$}
                    {
                        $\lambda_1^L \gets \lambda_1^m$
                    }
                    \uElse
                    {
                        $\lambda_1^U \gets \lambda_1^m$
                    }
                }

                \tcc{Following while-loop enlarges $\lambda_1^L$ and $\lambda_1^U$ until $[\lambda_1^L, \lambda_1^U]$ covers $\lambda_1^*$ such that $h(\lambda_1^*) = P_A-Q_A/\nu$ under numerical integration error}
                \While{$(h(\lambda_1^L) + \Delta > P_A - Q_A/\nu)$ \text{ or } $(h(\lambda_1^U) - \Delta < P_A - Q_A/\nu)$}{
                    $t \gets \lambda_1^U - \lambda_1^L$\;
                    
                    $\lambda_1^L \gets \lambda_1^L - t/2$\;
                    
                    $\lambda_1^U \gets \lambda_1^U + t/2$\;
                } 
                
                \Return{$(\lambda_1^L,\,(a^L - \lambda_1^U)/\nu)$} \tcc*[r]{for soundness, choose the left endpoint of $\lambda_1$ and $\lambda_2$ range}
            \caption{\textsc{DualBinarySearch} for $\lambda_1$ and $\lambda_2$.}
            \label{alg:dual-binary-search}
        \end{algorithm}
        \end{minipage}
        \end{center}
        
        \Cref{alg:dual-binary-search} is the dual variable search algorithm described in \Cref{subsec:joint-binary-search}.
        From Line $1$ to $8$, we conduct binary search for quantity $\lambda_1 + \nu\lambda_2$;
        from Line $14$ to $21$, we conduct binary search for quantity $\lambda_1$.
        Notice that our binary search interval is initialized to be the non-negative interval.
        This is because $Q(a^m,0) =0$ and $h(\lambda_1^m) = 0$ if $a^m$ and $\lambda_1^m$ are non-positive observed from \Cref{thm:concrete-equations}.
        
    \subsection{Guaranteeing Numerical Soundness}
        \label{newadxsec:guaranteeing-soundness}
        
        In \shortApproach computational method, to compute a practically sound robustness guarantee, we take the binary search error and numerical integration error into consideration.
        
        Specifically, in \Cref{alg:dual-binary-search}, we return a pair $(\lambda_1',\lambda_2')$ whose $R(\lambda_1',\lambda_2')$ lower bounds $R(\lambda_1, \lambda_2)$ where $(\lambda_1,\lambda_2)$ is the precise feasible pair.
        We achieve so by returning $(\lambda_1^L,\,(a^L - \lambda_1^U)/\nu)$.
        Specifically,
        as we can see $a^L$ is an underestimation of actual $(\lambda_1 + \nu\lambda_2)$ in the presence of binary search error and numerical integration error.
        We bound the numerical integration error by setting absolute error bound $\Delta = 1.5\times 10^{-8}$ in \texttt{scipy.integrad.quad} function.
        Then, $\lambda_1^L$ and $\lambda_1^U$ are underestimation and overestimation of the actual $\lambda_1$ in the presence of errors respectively.
        As a result, $(a^L - \lambda_1^U)/\nu$ is an underestimation of $\lambda_2$.
        Therefore, both $\lambda_1$ and $\lambda_2$ are underestimated and by the monotonicity of $R(\cdot,\cdot)$, the actual $R(\lambda_1,\lambda_2)$ would be underestimated to guarantee the soundness.
        
        Then, since the evaluation of $R$ involves numerical integration, we compare the lower bound of $R$: $R(\lambda_1,\lambda_2) - \Delta$ in \Cref{alg:DSRS-pipeline}~(line 9) with $0.5$ to determine whether current robustness radius can be certified or not.
    
    \subsection{Complexity and Efficiency Analysis}
    
        \label{newadx:complexity-efficiency}

            In this appendix, we briefly analyze the computational complexity of \shortApproach computational method introduced in \Cref{sec:dsrs}.
            Suppose the binary search precision is $\epsilon$, and each numerical integration costs $C$ time.
            First, the search of certified robust radius costs $O(\log (\sqrt{d}/\epsilon))$.
            For each searched radius, we first determine $P_A$ and $Q_A$ by running Neyman-Pearson-based certification, which has cost $O(\log(1/\epsilon)C)$.
            Then, solving dual variables takes two binary search rounds, which has cost $O(\log(1/\epsilon)C)$.
            The final one-time integration of $R(\lambda_1,\lambda_2)$ has cost $O(C)$.
            Thus, overall time complexity is $O(\log (\sqrt{d}/\epsilon) \log(1/\epsilon) C)$, which is the same as classical Neyman-Pearson certification and grows slowly~(in logarithmic factor) w.r.t. input dimension $d$.
            
            In practice, the certification time is on average $\SI{5}{s}$ to $\SI{10}{s}$ per sample across different datasets.
            For example, with $\sigma=0.50$ as the smoothing variance parameter, the certification time, as an overhead over Neyman-Pearson-based certification, is \SI{10.53}{s}, \SI{4.53}{s}, and \SI{3.21}{s} on MNIST, CIFAR-10, and ImageNet respectively.
            This overhead is almost negligible compared with the sampling time for estimating $P_A$ and $Q_A$ which is around $\SI{200}{s}$ on ImageNet and is the shared cost of all randomized smoothing certification methods.
            In summary, compared with standard Neyman-Pearson-based certification, the running time of \shortApproach is roughly the same.

\section{Extensions and Proofs in \texorpdfstring{\Cref{sec:theory}}{Tightness Analysis}}

    \label{newadx:pro-tightness-analysis}
    
    In this appendix, we provide formal proofs and theoretical extensions for the results in \Cref{sec:theory}.
    
    \subsection{Proof of \texorpdfstring{\lowercase{\Cref{thm:suffices-binary-classification}}}{Theorem 1}}
    
        \label{newadx:sub:proof-suffices-binary-classification}
        
        \begin{proof}[Proof of \Cref{thm:suffices-binary-classification}]
            We let $D_c$ denote the decision region of $F_0$ for class $c$, i.e., $D_c := \{\vx:\,F_0(\vx) = c\}$.
            Since $\gQ$ is supported on the decision region shifted by $\vx_0$, $f_0^\gQ(\vx_0)_c = 1$.
            Thus, from $f_0^\gQ(\vx_0)_c$, we know $(\supp(\gQ) + \vx_0) \setminus S \subseteq D_c$, where $S$ is some set with zero measure under $\gQ + \vx_0$.
            Since $0<q(\vx)/p(\vx)<+\infty$, $S$ also has zero measure under $\gP + \vx_0$.
            On the other hand, by $0 < q(\vx)/p(\vx) < +\infty$, we can determine the probability mass of $\supp(\gQ)$ on $\gP$, i.e., $\Pr_{\vepsilon\sim\gP} [\vepsilon \in \supp(\gQ)]$.
            Then, we observe that 
            \begin{align*}
                &  
                f_0^\gP(\vx_0)_c =
                \Pr_{\vepsilon\sim\gP} [F_0(\vx_0 + \vepsilon) = c] \\
                = & \Pr_{\vepsilon\sim\gP} [\vepsilon \in \sR^d \setminus \supp(\gQ)] 
                \\
                & \hspace{2em} 
                \cdot \Pr_{\vepsilon\sim\gP} [F_0(\vx_0+\vepsilon)=c \,|\, \vepsilon\in \sR^d\setminus \supp(\gQ)] \\
                & \vspace{2em} + \Pr_{\vepsilon\sim\gP} [\vepsilon \in \supp(\gQ)] \cdot \Pr_{\vepsilon\sim\gP} [F_0(\vx_0 + \vepsilon)=c \,|\, \vepsilon \in \supp(\gQ)] \\
                = & \Pr_{\vepsilon\sim\gP} [\vepsilon \in \sR^d \setminus \supp(\gQ)] 
                \\
                & \hspace{2em} 
                \cdot \Pr_{\vepsilon\sim\gP} [F_0(\vx_0+\vepsilon)=c \,|\, \vepsilon\in \sR^d\setminus \supp(\gQ)] \\
                &  + \Pr_{\vepsilon\sim\gP} [\vepsilon \in \supp(\gQ)].
            \end{align*}
            By the definition of $\supp(\gQ)$, we observe that $\Pr_{\vepsilon\sim\gP} [F_0(\vx_0+\vepsilon)=c \,|\, \vepsilon\in \sR^d\setminus \supp(\gQ)] = 0$.
            As a result, we will find that $f_0^\gP(\vx_0)_c = \Pr_{\vepsilon\sim\gP} [\vepsilon \in \supp(\gQ)]$.
            Then the \shortApproach certification method can know $\left(\left(\sR^d \setminus \supp(\gQ)\right) + \vx_0\right) \cap D_c$ has zero measure under $\gP + \vx_0$, i.e.,
            In summary, the certification method can determine that $\supp(\gQ) + \vx_0$ differs from $D_c$ on some set $\Delta$ with zero measure under $\gP+\vx_0$.
            Because $\gP$ has positive density everywhere, $\Delta$ also has zero measure under $\gP + \vx_0 + \vdelta$ for arbitrary $\vdelta\in\sR^d$.
            Thus, for arbitrary $\vdelta\in\sR^d$, the certification method can compute out
            \begin{align}
                f_0^\gP(\vx_0+\vdelta)_c & = \Pr_{\vepsilon\sim\gP} [F_0(\vx_0+\vdelta+\vepsilon)=c] \\ & \hspace{-2em} = \Pr_{\vepsilon\sim\gP+\vdelta} [\vx_0 + \vepsilon\in D_c] = \Pr_{\vepsilon\sim\gP+\vdelta}[\supp(\gQ)]. \nonumber
            \end{align}
            Under the binary classification setting, it implies that for any $\vdelta\in\sR^d$, the $f_0^\gP(\vx_0+\vdelta)$ can be uniquely determined by \shortApproach certification method.
            Since the smoothed classifier's decision at any $\vx_0+\vdelta$ is uniquely determiend by $f_0^\gP(\vx_0+\vdelta)$~(see \Cref{eq:F}), the certification method can exactly know $\wF_0^\gP(\vx+\vdelta)$ for any $\vdelta$ and thus determine tightest possible certified robust radius $\rtight$.
        \end{proof}
    
    \subsection{Extending \texorpdfstring{\lowercase{\Cref{thm:suffices-binary-classification}}}{Theorem 1} to Multiclass Setting}
    
        \label{newadx:sub:cor-suffices-multiclass-classification}
        
        For the multiclass setting, we define a variant of \shortApproach as follows.
        
        \begin{definition}[$\rDSRSmulti$]
            \label{def:r-dsrs-multi}
            Given $P_A \in [0,\,1]$ and $Q_A^\multi \in \sR^{C-1}$, 
            \vspace{-2mm}
            \begin{equation}
                \small \label{eq:r-dsrs-multi}
                \begin{aligned}
                    & \rDSRSmulti := \max r \quad \mathrm{s.t.} \\
                    & \forall F: \sR \to [C], f^\gP(\vx_0)_{y_0} = P_A, \\
                    & f^{\gQ_c}(\vx_0)_{c} = (Q_A^\multi)_c, c\in [C-1] \\
                    \Rightarrow & \forall \vx, \|\vx - \vx_0\|_p < r, \wF^\gP(\vx) = y_0.
                \end{aligned}
            \end{equation}
            \vspace{-6mm}
        \end{definition}
        
        In the above definition, $\rDSRSmulti$ is the tightest possible certified radius with prediction probability $Q_A^\multi$, where each component of $Q_A^\multi$, namely $(Q_A^\multi)_c$, corresponds to the prediction probability for label $c$ under additional smoothing distribution $\gQ_c$.
        Note that there are $(C-1)$ additional smoothing distributions $\{Q_c\}_{c=1}^{C-1}$ in this generalization for multiclass setting.
        
        With this \shortApproach generation, the following corollary extends the tightness analysis in \Cref{thm:suffices-binary-classification} from binary to the multiclass setting.
        
        \begin{corollary}
            \label{cor:suffices-multiclass-classification}
            Suppose the original smoothing distribution $\gP$ has positive density everywhere, i.e., $p(\cdot) > 0$.
            For multiclass classification with base classifier $F_0$, at point $\vx_0\in\sR^d$, for each class $c \in [C-1]$, let $\gQ_c$ be a distribution that satisfies: (1)~its support is the decision region of $c$ shifted by $\vx_0$: $\supp(\gQ_c) = \{\vx - \vx_0: F_0(\vx) = c\}$; (2)~for any $\vx\in\supp(\gQ)$, $0 < q_c(\vx)/p(\vx) < +\infty$.
            Then, plugging $P_A = f_0^\gP(\vx_0)_c$ and $Q_A^\multi$ where $(Q_A^\multi)_c = f_0^{\gQ_c}(\vx_0)_c$ for $c\in [C-1]$ into \Cref{def:r-dsrs-multi}, we have $\rDSRSmulti = \rtight$ under any $\ell_p$~($p \ge 1$).
            
        \end{corollary}
        
        \begin{proof}[Proof of \Cref{cor:suffices-multiclass-classification}]
            Similar as the proof of \Cref{thm:suffices-binary-classification}, the certification method can observe that for any $c \in [C-1]$, $f_0^{\gQ_c}(\vx_0)_c = 1$.
            Thus, the method knows $\supp(\gQ_c) + \vx_0 \approx D_c$ for arbitrary $c \in [C-1]$.
            Here, the ``$\approx$'' means that the difference between the two sets has zero measure under $\gP+\vx_0$.
            Thus, for arbitrary $\vdelta\in\sR^d$ the certification method can precisely compute out $f^\gP(\vx_0)_c$ for any $c \in [C-1]$.
            Since $f^\gP(\vx_0) \in \Delta^C$, we also know $f^\gP(\vx_0)_C$ and the smoothed classifier's prediction on $\vx_0+\vdelta$ can be uniquely determined.
            Then, following the same argument in \Cref{thm:suffices-binary-classification}'s proof, we can determine the tightest possible certified robust radius $\rDSRSmulti$.
        \end{proof}
        
        \begin{remark}
            To achieve the tightest possible certified radius $\rtight$, for binary classification, we only need one extra scalar as the additional information~($Q_A$), while for multiclass classification, we need $(C-1)$ extra scalars as the additional information~($Q_A^\multi \in \sR^{C-1}$).
            Following the convention as discussed in \Cref{newadx:detail-n-p-cert}, we are interested in tight certification under the binary classification for sampling efficiency concerns.
            Therefore, we focus on using only one extra scalar~($Q_A$) to additional information in \shortApproach.
            In both \Cref{thm:suffices-binary-classification} and \Cref{cor:suffices-multiclass-classification}, we only need finite quantities to achieve tight certification for \emph{any} smoothed classifier.
            In contrast, other existing work requires infinite quantities to achieve such optimal tightness~\citep[Section 3.1]{mohapatra2020higherorder}.
        \end{remark}
    
    \subsection{Proof of \texorpdfstring{\lowercase{\Cref{thm:concentration-sqrt-d}}}{Theorem 2}}
    
        \label{newadx:sub:concentration-sqrt-d}
        
        The proof of \Cref{thm:concentration-sqrt-d} is a bit complicated, which relies on several propositions and lemmas along with theoretical results in \Cref{sec:dsrs}.
        At high level, 
        based on the standard Gaussian distribution's property~(\Cref{prop:thm-2-1}), we find $Q_A = 1$ under concentration property~(\Cref{lem:thm-2-2}).
        With $Q_A = 1$, we derive a lower bound of $\rDSRS$ in \Cref{lem:thm-2-3}.
        We then use: (1)~the concentration of beta distribution $\Beta(\frac{d-1}{2},\frac{d-1}{2})$~(see \Cref{lem:thm-2-4}) for large $d$;
        (2)~the relative concentration of gamma $\Gamma(d/2,1)$ distribution around mean for large $d$~(see \Cref{prop:thm-2-5} and resulting \Cref{fact:thm-2-7});
        and (3)~the misalignment of gamma distribution $\Gamma(d/2-k,1)$'s mean and median for small $(d/2-k)$~(see \Cref{prop:thm-2-6}) to lower bound the quantity in \Cref{lem:thm-2-3} and show it is large or equal to $0.5$.
        Then, using the conclusion in \Cref{sec:dsrs} we conclude that $\rDSRS \ge 0.02\sigma\sqrt{d}$.
        
        \begin{adxproposition}
            \label{prop:thm-2-1}
            If random vector $\vepsilon \in \sR^d$ follows standard Gaussian distribution $\gN(\sigma)$, 
            then 
            \begin{equation}
                \Pr [\|\vepsilon\|_2 \le T] =  \GammaCDF_{d/2}\left(\frac{T^2}{2\sigma^2}\right),
            \end{equation}
            where $\GammaCDF_{d/2}$ is the CDF of gamma distribution $\Gamma(d/2,1)$.
        \end{adxproposition}
        
        \begin{proof}[Proof of \Cref{prop:thm-2-1}]
            According to \citep[Eqn. 5.19.4]{lozier2003nist}, he volume of a $d$-dimensional ball, i.e., $d$-ball, with radius $r$ is 
            $
                V_d(r) = \frac{\pi^{d/2}}{\Gamma(\frac{d}{2}+1)} r^d
            $.
            Thus,
            \begin{equation}
                \Vol(\{\vepsilon: \|\vepsilon\|_2=r\}) = V_d(r)' = \dfrac{d\pi^{d/2}}{\Gamma(\frac{d}{2}+1)} r^{d-1}.
            \end{equation}
        
            \begin{align*}
                & \Pr[\|\vepsilon\|_2 \le T] \\
                = & \int_0^T
                \dfrac{1}{(2\pi\sigma^2)^{d/2}} \cdot \exp\left(-\frac{r^2}{2\sigma^2}\right) \cdot \mathrm{Vol}(\{\vepsilon: \|\vepsilon\|_2=r\}) \dif r \\
                = & \int_0^T 
                \dfrac{1}{(2\pi\sigma^2)^{d/2}} \cdot \exp\left(-\frac{r^2}{2\sigma^2}\right) \cdot
                \dfrac{d\pi^{d/2}}{\Gamma(\frac d 2 + 1)} r^{d-1} \dif r \\
                = & \int_0^{T^2} \dfrac{1}{(2\sigma^2)^{d/2} \Gamma(\frac d 2)} \exp\left(-\frac{r}{2\sigma^2}\right) r^{d/2-1} \dif r \\
                = & \dfrac{1}{\Gamma(\frac d 2)} \int_0^{\frac{T^2}{2\sigma^2}} \exp(-r) r^{d/2-1} \dif r = \GammaCDF_{d/2} \left( \frac{T^2}{2\sigma^2} \right).
            \end{align*}
        \end{proof}
        
        With \Cref{prop:thm-2-1}, now we can show that $Q_A = 1$ under the condition of \Cref{thm:concentration-sqrt-d} as stated in the following lemma.
        
        \begin{lemma}
            \label{lem:thm-2-2}
            Suppose $F_0$ satisfies $(\sigma,\Pcon)$-concentration property at input point $\vx_0 \in \sR^d$,
            with additional smoothing distribution $\gQ = \gNgtrunc(k,T,\sigma)$ where $T^2 = 2\sigma^2\GammaCDF_{d/2}^{-1}(\Pcon)$ and $d/2 - 15 \le k < d/2$, we have
            \begin{equation}
                Q_A = \Pr_{\vepsilon\sim\gQ} [F_0(\vx_0+\vepsilon)=y_0] = 1.
            \end{equation}
        \end{lemma}
        
        \begin{proof}[Proof of \Cref{lem:thm-2-2}]
            According to \Cref{def:concentration}, for $T'$ that satisfies
            \begin{equation}
                \Pr_{\vepsilon\sim\gN(\sigma)} [\|\vepsilon\|_2 \le T'] = \Pcon
                \label{eq:pf-lem-thm-2-2-1}
            \end{equation}
            we have 
            \begin{equation}
                \Pr_{\vepsilon\sim\gN(\sigma)} [F_0(\vx_0 + \vepsilon) = y_0 \,|\, \|\vepsilon\|_2 \le T'] = 1.
                \label{eq:pf-lem-thm-2-2-2}
            \end{equation}
            With \Cref{eq:pf-lem-thm-2-2-1},
            from \Cref{prop:thm-2-1}, we have
            \begin{equation}
                T'^2 = 2\sigma^2 \GammaCDF_{d/2}^{-1}(\Pcon).
            \end{equation}
            Thus, \Cref{eq:pf-lem-thm-2-2-2} implies
            \begin{equation}
                \small
                \Pr_{\vepsilon\sim\gN(\sigma)} [F_0(\vx_0 + \vepsilon) = y_0 \,|\, \|\vepsilon\|_2 \le \sqrt{2\sigma^2\GammaCDF_{d/2}^{-1}(\Pcon)}] = 1.
            \end{equation}
            Notice that $\gN(\sigma)$ has finite and positive density anywhere within $\{\vepsilon: \|\vepsilon\|_2 \le T'\}$.
            Thus, $F_0(\vx_0+\vepsilon)=y_0$ for any $\vepsilon$ with $\|\vepsilon\|_2 \le T'$ unless a zero-measure set.
            
            Now, we consider $\gQ$.
            $\gQ = \gNgtrunc(k,T,\sigma)$ where $T = T'$, and $\gNtrunc$ has finite and positive density anywhere within $\{\vepsilon: \|\vepsilon\|_2 \le T'\} \setminus \{\mathbf{0}\}$.
            Thus,
            \begin{align*}
                Q_A = & \Pr_{\vepsilon\sim\gNgtrunc(k,T,\sigma)} [F_0(\vx_0 + \vepsilon) = y_0] \\
                = & \Pr_{\vepsilon\sim \gNg(k,\sigma)} [F_0(\vx_0+\vepsilon)=y_0 \,|\, \|\vepsilon\|_2 \le T] \\
                \overset{(*)}{=} & \Pr_{\vepsilon\sim \gNg(k,\sigma)}
                [F_0(\vx_0+\vepsilon)=y_0 \,|\, \|\vepsilon\|_2 \le T, \vepsilon \neq \mathbf{0}] \\
                = & 1.
            \end{align*}
            In the above equations, $(*)$ is because
            \begin{align*}
                & \Pr_{\vepsilon\sim\gNg(k,\sigma)} [\vepsilon = \mathbf{0} \,|\, \|\vepsilon\|_2 \le T] \\
                \le & \lim_{r\to 0} \Pr_{\vepsilon\sim\gNg(k,\sigma)} [\|\vepsilon\|_2 \le r \,|\, \|\vepsilon\|_2 \le T] \\
                = & \lim_{r\to 0}
                C \int_0^r
                x^{-2k} \exp\left(-\frac{x^2}{2\sigma'^2}\right) d x^{d-1} \dif x \\
                = & C\lim_{r\to 0} \int_0^r d x^{d-2k-1} \exp\left(-\frac{x^2}{2\sigma'^2}\right) \dif x 
                = 0
            \end{align*}
            where $C$ is a constant, and the last equality is due to $d /2 > k \Rightarrow d - 2k - 1 \ge 0$.
        \end{proof}
        
        With $Q_A = 1$, we can have a lower bound of $\rDSRS$ as stated in the following lemma.
        
        \begin{lemma}
            \label{lem:thm-2-3}
            Under the same condition as in \Cref{lem:thm-2-2},
            we let $\BetaCDF_{\frac{d-1}{2}}$ be the CDF of distributon $\Beta(\frac{d-1}{2},\frac{d-1}{2})$,
            let
            \begin{equation}
                \begin{aligned}
                    & r_0 = \max u \\
                    \mathrm{s.t.} & \E_{t\sim\Gamma(\frac{d}{2}-k)}
                    \BetaCDF_{\frac{d-1}{2}} \left(
                    \dfrac{T^2 - (\sigma' \sqrt{2t} - u)^2}{4u\sigma' \sqrt{2t}} \right) \ge 0.5,
                \end{aligned}
                \label{eq:lem-thm-2-3-main}
            \end{equation}
            and let $\rDSRS$ be the tightest possible certified radius in \shortApproach under $\ell_2$ when smoothing distribution $\gP = \gNg(k,\sigma)$, then 
            \begin{equation}
                \rDSRS \ge r_0.
            \end{equation}
        \end{lemma}
        
        \begin{proof}[Proof of \Cref{lem:thm-2-3}]
            The proof shares the same core methodology as \shortApproach computational method introduced in \Cref{sec:dsrs}.
            Basically, according to \Cref{eq:check-cond}, for any radius $r$, 
            let $\vdelta = (r,0,\dots,0)^\T$, if $\tC_{\vdelta}(P_A, Q_A) > 0.5$, then $\rDSRS \ge r$, where $Q_A = 1$ according to \Cref{lem:thm-2-2}, and by the $(\sigma,\Pcon)$-concentration property
            \begin{equation}
                P_A \ge \Pr_{\vepsilon\sim\gNg(k,\sigma)} [\|\vepsilon\|_2 \le T].
            \end{equation}
            Therefore, to prove the lemma, we only need to show that when
            \begin{equation}
                \E_{t\sim\Gamma(\frac{d}{2}-k)}
                    \BetaCDF_{\frac{d-1}{2}} \left(
                    \dfrac{T^2 - (\sigma' \sqrt{2t} - u)^2}{4u\sigma' \sqrt{2t}} \right) \ge 0.5,
            \end{equation}
            for any $\vdelta = (u,0,\dots,0)^\T$, $\tC_{\vdelta}(P_A, Q_A) > 0.5$.
            
            By definition~(\Cref{eq:primal}),
            \begin{align*}
                    & \tC_{\vdelta} (P_A,Q_A) \\
                    = & \min_{f} \E_{\vepsilon\sim\gP} [f(\vepsilon+\vdelta)] \quad\mathrm{s.t.} \\
                    & \hspace{1em} \E_{\vepsilon\sim\gP} [f(\vepsilon)] = P_A,
                    \E_{\vepsilon\sim\gQ} [f(\vepsilon)] = Q_A, \\
                    & \hspace{1em}
                    0 \le f(\vepsilon) \le 1 \quad \forall \epsilon\in\sR^d \\
                    \ge & \min_{f} \E_{\vepsilon\sim\gP}
                    [f(\vepsilon+\vdelta)] \quad\mathrm{s.t.} \\
                    & \hspace{1em} 
                    \E_{\vepsilon\sim\gQ} [f(\vepsilon)] = 1, \\
                    & \hspace{1em}
                    0 \le f(\vepsilon) \le 1 \quad \forall \epsilon\in\sR^d \\
                    = & \E_{\vepsilon\sim\gP} [f(\vepsilon+\vdelta)] \quad\mathrm{where}\quad  
                    f(\vepsilon) = \left\{
                    \begin{array}{lr}
                        1, & \|\vepsilon\|_2 \le T \\
                        0. & \|\vepsilon\|_2 > T
                    \end{array}
                    \right. \\
                    =: & V.
            \end{align*}
            
            We now compute $V$:
            \begin{align*}
                V = & \E_{\vepsilon\sim\gP} [\|\vepsilon + \vdelta\|_2 \le T] \\
                = & \int_{\sR^d} p(\vx) \1[\|\vx + \vdelta\|_2 \le T] \dif \vx \\
                \overset{(1)}{=} & \int_0^{\infty} y \dif y \int_{\substack{p(\vx) = y\\ \1[\|\vx + \vdelta\|_2 \le T}} 
                \dfrac{\dif \vx}{\|\nabla p(\vx)\|_2} \\
                \overset{(2)}{=} & \int_0^\infty y\dif y \int_{\substack{p(\vx) = y\\ \1[\|\vx + \vdelta\|_2 \le T}}
                - \dfrac{\dif \vx}{r_p'(r_p^{-1}(y))} \\
                \overset{(3)}{=} & \int_0^\infty y \dif y
                \dfrac{2\pi^{d/2}}{\Gamma(\frac d 2)} r_p^{-1}(y)^{d-1}
                \cdot
                \left(- \dfrac{1}{r_p'(r_p^{-1}(y))} \right)
                \cdot \\
                & \hspace{2em}
                \Pr[\|\vx + \vdelta\|_2 \le T \,|\, p(\vx) = y] \\
                \overset{(4)}{=} & \int_0^\infty
                r_p(t)\dif t \dfrac{2\pi^{d/2}}{\Gamma(\frac d 2)} t^{d-1} 
                \Pr[\|\vx + \vdelta\|_2 \le T \,|\, \|\vx\|_2 = t] \\
                \overset{(5)}{=} & \int_0^\infty 
                \dfrac{1}{(2\sigma'^2)^{d/2-k}\pi^{d/2}}\cdot
                \dfrac{\Gamma(d/2)}{\Gamma(d/2-k)}
                t^{-2k} \exp\left(-\frac{t^2}{2\sigma'^2}\right) \cdot\\ 
                & \hspace{2em} \dfrac{2\pi^{d/2}}{\Gamma(\frac d 2)} t^{d-1} \Pr[\|\vx + \vdelta\|_2 \le T \,|\, \|\vx\|_2 = t] \dif t \\
                = & \dfrac{1}{\Gamma(\frac d 2 - k)} \int_0^\infty 
                t^{d/2-k-1} \exp(-t) \cdot \\
                & \hspace{2em} \Pr[\|\vx + \vdelta\|_2 \le T \,|\, \|\vx\|_2 = \sigma'\sqrt{2t}] \dif t \\
                = & \E_{t\sim\Gamma(\frac d 2 -k)} \Pr[\|\vx + \vdelta\|_2 \le T \,|\, \|\vx\|_2 = \sigma'\sqrt{2t}].
            \end{align*}
            In above equations, $(1)$ follows from level-set sliced integration extended from \citep{yang2020randomized} and $p(\vx)$ is the density of distribution $\gP=\gNg(k,\sigma)$ at point $\vx$;
            in $(2)$ we define $r_p(\|\vx\|_2) := p(\vx)$ noting that $\gP$ is $\ell_2$ symmetric and all $\vx$ with same $\ell_2$ length having the same $p(\vx)$, and we have $\|\nabla p(\vx)\|_2 = -r'_p(r_p^{-1}(y))$ since $y = r_p(\|\vx\|_2)$ and $r_p$ is monotonically decreasing;
            $(3)$ uses 
            \begin{equation}
                \begin{aligned}
                    \mathrm{Vol}(\{\vx: p(\vx)=y\}) & = \mathrm{Vol}(\{\vx: \|\vx\|_2 = r_p^{-1}(y)\}) \\
                    & = \dfrac{2\pi^{d/2}}{\Gamma(\frac d 2)} r_p^{-1}(y)^{d-1};
                \end{aligned}
            \end{equation}
            $(4)$ changes the integration variable from $y$ to $t = r_p^{-1}(y)$;
            and $(5)$ injects the concrete expression of $r_p$.
            
            Now, we inspect $\Pr[\|\vx + \vdelta\|_2 \le T \,|\, \|\vx\|_2 = \sigma'\sqrt{2t}]$:
            When $\|\vx\|_2 = \sigma'\sqrt{2t}$, 
            $\sum_{i=1}^d x_i^2 = 2t\sigma'^2$.
            Meanwhile, $\|\vx + \vdelta\|_2 \le T$ means $(x_1+u)^2 + \sum_{i=2}^d x_i^2 \le T^2$.
            Thus, when $\|\vx\|_2 = \sigma'\sqrt{2t}$,
            \begin{equation}
                \|\vx + \vdelta\|_2 \le T
                \iff
                \dfrac{x_1}{\sigma'\sqrt{2t}} \le \dfrac{T^2-u^2-2t\sigma'^2}{2u\sigma'\sqrt{2t}}.
            \end{equation}
            According to \citep[Lemma I.23]{yang2020randomized}, for $\vx$ uniformly sampled from sphere with radius $\sigma'\sqrt{2t}$, the component coordinate $\dfrac{1 + \frac{x_1}{\sigma'\sqrt{2t}}}{2} \sim \Beta(\frac{d-1}{2}, \frac{d-1}{2})$.
            Thus,
                \begin{align}
                    & \Pr[\|\vx + \vdelta\|_2 \le T \,|\, \|\vx\|_2 = \sigma'\sqrt{2t}] \nonumber \\
                    = & \BetaCDF_{\frac{d-1}{2}} \left(
                        \dfrac{1 + \frac{T^2-u^2-2t\sigma'^2}{2u\sigma'\sqrt{2t}}}{2}
                    \right) \nonumber \\
                    = & \BetaCDF_{\frac{d-1}{2}} \left(
                    \dfrac{T^2 - (\sigma'\sqrt{2t} - u)^2}{4u\sigma'\sqrt{2t}}
                    \right).
                \end{align}
            Finally, we get
            \begin{equation}
                V = \E_{t\sim\Gamma(\frac{d}{2}-k,1)}
                    \BetaCDF_{\frac{d-1}{2}} \left(
                    \dfrac{T^2 - (\sigma' \sqrt{2t} - u)^2}{4u\sigma' \sqrt{2t}} \right).
                \label{eq:V}
            \end{equation}
            In other words,
            when
            \begin{equation}
                \E_{t\sim\Gamma(\frac{d}{2}-k,1)}
                    \BetaCDF_{\frac{d-1}{2}} \left(
                    \dfrac{T^2 - (\sigma' \sqrt{2t} - u)^2}{4u\sigma' \sqrt{2t}} \right) \ge 0.5,
            \end{equation}
            we have $V \ge 0.5$, and thus $\tC_{\delta}(P_A,Q_A) \ge V \ge 0.5$, $\rDSRS \ge u$, which concludes the proof.
        \end{proof}
        
        We require the following property of $\BetaCDF$.
        
        \begin{lemma}
            \label{lem:thm-2-4}
            There exists $d_0 \in \sN_+$, for any $d \ge d_0$, 
            \begin{equation}
                \BetaCDF_{\frac{d-1}{2}}(0.6) \ge 0.999.
            \end{equation}
        \end{lemma}
        
        \begin{proof}[Proof of \Cref{lem:thm-2-4}]
            We let $\rv \sim \Beta(\frac{d-1}{2},\frac{d-2}{2})$, then
            \begin{align}
                & \E [\rv] = 1/2, \\
                & \Var [\rv] = \dfrac{(\frac{d-1}{2})^2}{(\frac{d-1}{2} + \frac{d-1}{2})^2(\frac{d-1}{2} + \frac{d-1}{2} + 1)} = \frac{1}{4d}.
            \end{align}
            Now, applying Chebyshev's inequality, we have
            \begin{equation}
                \Pr[|\rv-0.5| \ge 0.1] \le \dfrac{1}{0.04d}.
            \end{equation}
            Therefore, 
            \begin{align}
                & \BetaCDF_{\frac{d-1}{2}}(0.6) \\
                = & \Pr[\rv < 0.6] \\
                = & 1 - \Pr[\rv \ge 0.6] \\
                \ge & 1 - \Pr[|\rv - 0.5| \ge 0.1] 
                \ge 1 - \dfrac{1}{0.04d}.
            \end{align}
            Thus, when $d \ge 25000$, $\BetaCDF_{\frac{d-1}{2}}(0.6) \ge 0.999$.
        \end{proof}
        
        We also require the following properties of the gamma distribution.
        
        \begin{adxproposition}
            \label{prop:thm-2-5}
            For any $\Pcon \in (0,1)$, there exists $d_0 \in \sN_+$, for any $d \ge d_0$,
            \begin{equation}
                \GammaCDF^{-1}_{d/2}(\Pcon) \ge 0.99 \cdot \dfrac{d}{2}.
            \end{equation}
        \end{adxproposition}
        
        \begin{proof}[Proof of \Cref{prop:thm-2-5}]
            We let $\rv \sim \Gamma(d/2)$, then
            \begin{align}
                & \E[\rv] = d/2, \\
                & \Var[\rv] = d/2.
            \end{align}
            We now apply Chebyshev's inequality and get
            \begin{align*}
                & \Pr[\rv < 0.99 \cdot d/2] \\
                \le & \Pr[|\rv - d/2| > 0.01 \cdot d/2] \\
                \le & \dfrac{20000}{d}.
            \end{align*}
            Thus, for any $\Pcon \in (0,1)$, when $d \ge \dfrac{20000}{\Pcon}$, 
            \begin{equation}
                \GammaCDF_{d/2}\left(0.99\cdot \frac d 2\right) \le \dfrac{20000}{d} \le \Pcon,
            \end{equation}
            i.e.,
            $
                \GammaCDF^{-1}_{d/2}(\Pcon) \ge 0.99\cdot \frac d 2.
            $
        \end{proof}
        
        \begin{adxproposition}
            \label{prop:thm-2-6}
            When $d/2 - 15 \le k < d/2$, $k, d \in \sN_+$,
            \begin{equation}
                \Pr_{t\sim\Gamma(d/2-k)} \left[t \le 0.98 \left(\frac d 2 - k\right)\right] \ge \frac{0.5}{0.999}.
            \end{equation}
        \end{adxproposition}
        
        \begin{proof}[Proof of \Cref{prop:thm-2-6}]
            We prove the proposition by enumeration.
            Notice that $d/2 - k \in \{0.5, 1.0, \cdots, 14.5, 15.0\}$, we enumerate $\GammaCDF_{d/2-k}(0.98(d/2-k))$ for each $(d/2-k)$ and get the following table.
            \resizebox{\linewidth}{!}{
            \begin{tabular}{cc|cc}
                \toprule
                $\frac{d}{2}-k$ & $\GammaCDF_{d/2-k}(0.98(d/2-k))$ & 
                $\frac{d}{2}-k$ & $\GammaCDF_{d/2-k}(0.98(d/2-k))$ \\
                \midrule
                0.5 & 0.6778 &
                  8.0 & 0.5245 \\
                1.0 & 0.6247 &
                8.5 & 0.5224 \\
                1.5 & 0.5990 &
                9.0 & 0.5204 \\
                2.0 & 0.5831 &
                9.5 & 0.5186 \\
                2.5 & 0.5718 &
                10.0 & 0.5168 \\
                3.0 & 0.5632 &
                10.5 & 0.5152 \\
                3.5 & 0.5564 &
                11.0 & 0.5136 \\
                4.0 & 0.5507 &
                11.5 & 0.5121 \\
                4.5 & 0.5459 &
                12.0 & 0.5107 \\
                5.0 & 0.5418 &
                12.5 & 0.5093 \\
                5.5 & 0.5381 &
                13.0 & 0.5080 \\
                6.0 & 0.5349 &
                13.5 & 0.5068 \\
                6.5 & 0.5319 &
                14.0 & 0.5056 \\
                7.0 & 0.5292 &
                14.5 & 0.5044 \\
                7.5 & 0.5268 &
                15.0 & 0.5033 \\
                \bottomrule
            \end{tabular}
            }
            On the other hand, $\frac{0.5}{0.999} \le 0.5001$, which concludes the proof.
        \end{proof}
        
        Now we are ready to prove the main theorem.
        
        \begin{proof}[Proof of \Cref{thm:concentration-sqrt-d}]
             According to \Cref{lem:thm-2-3},
             we only need to show that for $u = 0.02\sigma\sqrt d$,
             \Cref{eq:lem-thm-2-3-main} holds.
             For sufficiently large $d$, indeed,
             \begin{align*}
                 & \E_{t\sim\Gamma(d/2-k)} \BetaCDF_{(d-1)/2} \left(
                 \dfrac{T^2 - (\sigma'\sqrt{2t} - u)^2}{4u\sigma'\sqrt{2t}} \right) \\
                 & \overset{\text{\Cref{lem:thm-2-4}}}{\ge} 0.999 \E_{t\sim\Gamma(d/2-k)} \1\left[
                 \dfrac{T^2 - (\sigma'\sqrt{2t} - u)^2}{4u\sigma'\sqrt{2t}} \ge 0.6
                 \right] \\
                 & \overset{(*)}{\ge} 0.999 \E_{t\sim\Gamma(d/2-k)} \1\left[
                 t \le 0.98 \left(\frac d 2 - k\right)
                 \right] \\
                 & \overset{\text{\Cref{prop:thm-2-6}}}{\ge} 0.999 \cdot \dfrac{0.5}{0.999} = 0.5.
             \end{align*}
            Thus, from \Cref{lem:thm-2-3} we have $\rDSRS \ge u = 0.02\sqrt d$.
            
            The inequality $(*)$ follows from \Cref{fact:thm-2-7}.
        \end{proof}
        
        \begin{fact}
            \label{fact:thm-2-7}
            Under the condition of \Cref{thm:concentration-sqrt-d}, 
            for sufficiently large $d$,
            \begin{equation}
                t \le 0.98 \left(\frac d 2 - k\right) \Rightarrow
                \dfrac{T^2 - (\sigma'\sqrt{2t} - u)^2}{4u\sigma'\sqrt{2t}} \ge 0.6.
            \end{equation}
        \end{fact}
        
        \begin{proof}[Proof of \Cref{fact:thm-2-7}]
            \begin{align*}
                & \dfrac{T^2 - (\sigma'\sqrt{2t}-u)^2}{4u\sigma'\sqrt{2t}} \ge 0.6 \\
                \iff & T^2 - (\sigma'\sqrt{2t} - u)^2 \ge 2.4u\sigma'\sqrt{2t} \\
                \overset{\text{$x:=\sigma'\sqrt{2t}$}}{\iff} &
                T^2 - (x-u)^2 \ge 2.4 ux \\
                \iff & x^2 + 0.4ux + u^2 - T^2 \le 0.
            \end{align*}
            From \Cref{prop:thm-2-5}, we have
            \begin{equation}
                \begin{aligned}
                    & u^2 - T^2 \\
                    = & 0.0004 d\sigma^2 - 2\sigma^2 \GammaCDF^{-1}_{d/2}(\Pcon) \\
                    \le & 0.0004 d\sigma^2 - 0.99d\sigma^2 < 0.
                \end{aligned}
            \end{equation}
            Thus,
            \begin{align*}
                & x^2 + 0.4ux + u^2 - T^2 \le 0 \\
                \iff & x \le \dfrac{-0.4u + \sqrt{0.16u^2 - 4(u^2-T^2)}}{2} \\
                & \hspace{3em} = -0.2u + \sqrt{T^2 - 0.96u^2} \\
                \iff & t \le \dfrac{(-0.2u + \sqrt{T^2 - 0.96u^2})^2}{2\sigma'^2}.
            \end{align*}
            Again, from \Cref{prop:thm-2-5},
            \begin{equation}
                \begin{aligned}
                    & T^2 - 0.96u^2 \\
                    = & 2\sigma^2\GammaCDF_{d/2}^{-1}(\Pcon) - 0.96 \times 0.0004 d\sigma^2 \\
                    \ge & (0.99 - 0.96 \times 0.0004) d\sigma^2,
                \end{aligned}
            \end{equation}
            and therefore
            \begin{equation}
                \begin{aligned}
                    & (-0.2u + \sqrt{T^2 - 0.96u^2})^2 \\
                    \ge & d\sigma^2 \left(-0.004 + \sqrt{0.99 - 0.96 \times 0.0004}\right)^2 \approx 0.9816 d\sigma^2  \\
                    \ge & 0.98 d\sigma^2.
                \end{aligned}
            \end{equation}
            Then,
            \begin{align*}
                & t \le \dfrac{(-0.2u + \sqrt{T^2 - 0.96u^2})^2}{2\sigma'^2} \\
                \Longleftarrow & t \le \dfrac{0.98d\sigma^2}{2\sigma^2} \cdot \dfrac{d-2k}{d} \\
                \iff & t \le 0.98 \left( \frac d 2 - k \right).
            \end{align*}
        \end{proof}
    
    \subsection{\texorpdfstring{\Cref{thm:concentration-constant-k-forbid-sqrt-d}}{Theorem 6}}
    
        \label{adxsec:concentration-constant-k-forbid-sqrt-d}
    
        \begin{theorem}
            \label{thm:concentration-constant-k-forbid-sqrt-d}
            Let $d$ be the input dimension and $F_0$ be the base classifier.
            For an input point $\vx_0\in \sR^d$ with true class $y_0$, suppose $F_0$ satisfies $(\sigma,\Pcon)$-Concentration property and $\Pr_{\vepsilon \sim \gN(\sigma)} [F_0(\vx_0 + \vepsilon) = y_0] = \Pcon$ where $\Pcon < 1$.
            The smoothed classifier $\wF_0^{\gP'}$ is constructed from $F_0$ and smoothed by generalized Gaussian $\gP'=\gNg(k_0,\sigma)$ where $k_0$ is a constant independent of input dimension $d$.
            Then,
            for any constant $c > 0$, there exists $d_0$, such that when input dimension $d \ge d_0$, \textbf{any method cannot certify $\ell_2$ radius $c\sqrt{d}$},
            where $T = \sigma\sqrt{2\GammaCDF_{d/2}^{-1}(\Pcon)}$ and $\GammaCDF_{d/2}$ is the CDF of gamma distribution $\Gamma(d/2,1)$.
        \end{theorem}
        
        We defer the proof to \Cref{newadx:sub:concentration-constant-k-forbid-sqrt-d}.
        This theorem suggests that, if we use generalized Gaussian whose $k$ is a constant with respect to input dimension $d$ or use standard Gaussian~(whose $k=0$ is a constant) for smoothing, we cannot achieve $\Omega(\sqrt{d})$ certified radius rate from \shortApproach and any other certification method. 
    
    \subsection{Proof of \texorpdfstring{\Cref{thm:concentration-constant-k-forbid-sqrt-d}}{Theorem 3}}
        \label{newadx:sub:concentration-constant-k-forbid-sqrt-d}
    
        The proof of \Cref{thm:concentration-constant-k-forbid-sqrt-d} is based on three lemmas listed below.
        
        \begin{lemma}
            Given $k_0 \in \sN$, for any $\epsilon > 0$, there exists $d_0$, such that when $d > d_0$, 
            \begin{equation}
                \Pr_{\rt\sim\Gamma\left(\frac d 2 - k_0,1\right)} \left[\rt\le (1-\epsilon)\left(\frac d 2 - k_0\right)\right] \le \frac{0.48}{0.99}.
                \label{eq:thm-3-pf-1}
            \end{equation}
            \label{lem:thm-3-pf-1}
        \end{lemma}
        
        \begin{lemma}
            Given $\Pcon \ge 0$, for any $\epsilon > 0$, there exists $d_0$, such that when $d > d_0$, 
            \begin{equation}
                T := \sigma\sqrt{2\GammaCDF^{-1}_{d/2}(\Pcon)} \le \sigma\sqrt{(1+\epsilon)d}.
                \label{eq:thm-3-pf-2}
            \end{equation}
            \label{lem:thm-3-pf-2}
        \end{lemma}
        
        \begin{lemma}
            For any $\epsilon > 0$, there exists $d_0$, such that when $d > d_0$,
            \begin{equation}
                \BetaCDF_{\frac{d-1}{2}} (0.5-\epsilon) \le 0.01.
                \label{eq:thm-3-pf-3}
            \end{equation}
            \label{lem:thm-3-pf-3}
        \end{lemma}
        
        Proofs of these lemmas are based on Chebyshev's inequality.
        
        \begin{proof}[Proof of \Cref{lem:thm-3-pf-1}]
            For $\rt\sim\Gamma(d/2 - k_0,1)$, we have
            \begin{equation}
                \E[\rt] = d/2 - k_0, \, \Var[\rt] = d/2 - k_0.
            \end{equation}
            By Chebyshev's inequality,
                \begin{align}
                    \hspace{-2em}\Pr \left[ \rt \le (1 - \epsilon)\left(\frac d 2 - k_0\right) \right] \le & \Pr \left[ |\rt - \E[\rt]| \ge \epsilon\left(\frac d 2 - k_0\right) \right] \nonumber \\
                    \le & \dfrac{1}{\epsilon^2 (\frac d 2 - k_0)}.
                \end{align}
            Picking $d_0 = 2\left( \frac{0.99}{0.48\epsilon^2} + k_0 \right)$ concludes the proof.
        \end{proof}
        
        \begin{proof}[Proof of \Cref{lem:thm-3-pf-2}]
            We define random variable $\rv \sim \Gamma(d/2,1)$, so $\E[\rv]=d/2, \Var[\rv]=d/2$.
            By Chebyshev's inequality,
            \begin{equation}
                \begin{aligned}
                    & \Pr[\rv \le (1+\epsilon)d/2] \\
                    \ge & 1 - \Pr[\rv \ge (1+\epsilon)d/2] \\
                    \ge & 1 - \Pr[|\rv - \E[\rv]| \ge \epsilon d/2] \\
                    \ge & 1 - \dfrac{2}{d\epsilon^2}.
                \end{aligned}
            \end{equation}
            Let $d_0 = \frac{2}{\epsilon^2 (1 - \Pcon)}$.
            Thus, when $d > d_0$,
            $\Pr[\rv \le (1+\epsilon)d/2] \ge 1 - \dfrac{2}{d_0\epsilon^2} = \Pcon$, which implies that $\GammaCDF_{d/2}((1+\epsilon)d/2) \ge \Pcon$ and $\GammaCDF_{d/2}^{-1}(\Pcon) \le (1+\epsilon)d/2$ and concludes the proof.
        \end{proof}
        
        \begin{proof}[Proof of \Cref{lem:thm-3-pf-3}]
            We define random variable $\rv\sim\Beta(\frac{d-1}{2},\frac{d-1}{2})$, and we have $\E[\rv] = 1/2, \Var[\rv] = \frac{1}{4d}$.
            By Chebyshev's inequality,
            \begin{equation}
                \Pr[\rv\le 0.5-\epsilon] \le \Pr[|\rv-\E[\rv]|\ge \epsilon] 
                \le \dfrac{1}{4d\epsilon^2}.
            \end{equation}
            Let $d_0 = \frac{25}{d\epsilon^2}$, when $d > d_0$, $\Pr[\rv \le 0.5-\epsilon] \le 0.01$ and hence $\BetaCDF_{\frac{d-1}{2}} (0.5-\epsilon) \le 0.01$.
        \end{proof}
        
        Now we are ready to prove the main theorem.
        
        \begin{proof}[Proof of \Cref{thm:concentration-constant-k-forbid-sqrt-d}]
            According to the above three lemmas, we pick $d_0$, such that when $d > d_0$, the followings hold simultaneously.
            \begin{align}
                & \Pr_{\rt\sim\Gamma\left(\frac d 2 - k_0,1\right)} \left[\rt\le \left(1-\frac{c^2}{8\sigma^2}\right)\left(\frac d 2 - k_0\right)\right] \le \frac{0.48}{0.99}, \label{eq:thm-3-pf-4} \\
                & T = \sigma\sqrt{2\GammaCDF^{-1}_{d/2}(\Pcon)} \le \sigma\sqrt{\left(1+\frac{c^2}{8\sigma^2}\right)d}, \label{eq:thm-3-pf-5} \\
                & \BetaCDF_{\frac{d-1}{2}} \left(0.5-\frac{c}{8\sigma}\right) \le 0.01. \label{eq:thm-3-pf-6}
            \end{align}
        
            We define vector $\vdelta = (c\sqrt d, 0, 0, \dots, 0)^\T$.
            Since $F_0$ satisfies $(\sigma,\Pcon)$-concentration property and $\Pr_{\vepsilon\sim\gN(\sigma)} [F_0(\vx_0+\vepsilon)=y_0] = \Pcon$, up to a set of zero measure, the region $\{\vepsilon: F_0(\vx_0+\vepsilon)=y_0\}$ and region $\{\vepsilon: \|\vepsilon\|_2\le T\}$ coincide.
            
            We now show that $\E_{\vepsilon\sim\gP'=\gNg(k_0,\sigma)}[F_0(\vx_0+\vdelta+\vepsilon)=y_0] < 0.5$ when $c \le \sigma \sqrt{8/7}$.
                \begin{align}
                    & \E_{\vepsilon\sim\gP'} [F_0(\vx_0+\vdelta+\vepsilon)=y_0] \nonumber \\
                    = & \Pr_{\vepsilon\sim\gP'} [\|\vdelta + \vepsilon\|_2 \le T] \nonumber \\
                    = & \E_{\rt\sim\Gamma(\frac d 2 - k_0,1)} \BetaCDF_{\frac{d-1}{2}} \left( \dfrac{T^2 - (\sigma'\sqrt{2\rt} - c\sqrt{d})^2}{4\sigma'\sqrt{2\rt}\cdot c\sqrt{d}} \right) \nonumber \\
                    & \hspace{5em} \text{(from \Cref{eq:V})} \nonumber \\
                    \le & 0.99 \E_{\rt\sim\Gamma(\frac d 2 - k_0,1)} \1\left[ \dfrac{T^2 - (\sigma'\sqrt{2\rt} - c\sqrt{d})^2}{4\sigma'\sqrt{2\rt}\cdot c\sqrt{d}} \ge 0.5 - \frac{c}{8\sigma} \right] \nonumber \\
                    & + 0.01 \hspace{3em} \text{(by \Cref{eq:thm-3-pf-6} and $\BetaCDF_{\cdot}(\cdot) \le 1$)} \nonumber \\
                    = & 0.99 \Pr_{\rt\sim\Gamma(\frac d 2 - k_0,1)} \left[ \dfrac{T^2 - (\sigma'\sqrt{2\rt} - c\sqrt{d})^2}{4\sigma'\sqrt{2\rt}\cdot c\sqrt{d}} \ge 0.5 - \frac{c}{8\sigma} \right] \nonumber \\
                    & + 0.01. \label{eq:thm-3-pf-7}
                \end{align}
                
            Since
            \begin{align}
                & \dfrac{T^2 - (\sigma'\sqrt{2\rt} - c\sqrt{d})^2}{4\sigma'\sqrt{2\rt}\cdot c\sqrt{d}} \ge 0.5 - \frac{c}{8\sigma} \nonumber \\
                \iff & T^2 - 2\rt \sigma^2 \frac{d}{d-2k_0} - dc^2 + \frac{c^2d}{2}\sqrt{\frac{2\rt}{d-2k_0}} \ge 0 \\
                \overset{\text{\Cref{eq:thm-3-pf-5}}}{\Longrightarrow} & \left(1+\frac{c^2}{8\sigma^2}\right)d\sigma^2 - 2\rt \sigma^2 \frac{d}{d-2k_0} - dc^2 \label{eq:thm-3-pf-8} \\
                & \hspace{3em} + \frac{c^2d}{2}\sqrt{\frac{2\rt}{d-2k_0}} \ge 0.  \nonumber
            \end{align}
            We now inject $\rt=0$ and $\rt=\left(1-\frac{c^2}{8\sigma^2}\right)\left(\frac{d}{2} - k_0\right)$ to the LHS of \Cref{eq:thm-3-pf-8}.
            \begin{itemize}
                \item When $\rt=0$, 
                    \begin{align*}
                        \text{LHS of \Cref{eq:thm-3-pf-8}} & = \left(1+\frac{c^2}{8\sigma^2}\right) d\sigma^2 - dc^2 \\
                        & = d\left(\sigma^2 - \frac{7}{8}c^2\right) \ge 0.
                    \end{align*}
                
                \item When $\rt=\left(1-\frac{c^2}{8\sigma^2}\right)\left(\frac{d}{2} - k_0\right)$,
                    \begin{align*}
                        & \text{LHS of \Cref{eq:thm-3-pf-8}} \\
                        = & \left( 1 + \frac{c^2}{8\sigma^2} \right) d\sigma^2 - \frac{2d\sigma^2}{d-2k_0} \left(1-\frac{c^2}{8\sigma^2}\right) \left(\frac{d}{2} - k_0\right) \\
                        & - dc^2 + \frac{dc^2}{2} \sqrt{ 1 - \frac{c^2}{8\sigma^2} } \\
                        = & d\sigma^2 + \frac{dc^2}{8} - d\sigma^2 + \dfrac{dc^2}{8} - dc^2 + \frac{dc^2}{2} \sqrt{ 1 - \frac{c^2}{8\sigma^2} } \\
                        \le & \frac{dc^2}{4} - dc^2 + \dfrac{dc^2}{2} < 0.
                    \end{align*}
            \end{itemize}
            Notice that the LHS of \Cref{eq:thm-3-pf-8} is a parabola with negative second-order coefficient.
            Thus,
            \begin{equation}
                \text{\Cref{eq:thm-3-pf-8}} \Longrightarrow 
                \rt \in \left[0, \left(1-\frac{c^2}{8\sigma^2}\right)\left(\frac{d}{2} - k_0\right)\right]
            \end{equation}
            and hence
            \begin{equation}
                \begin{aligned}
                    & \Pr_{\rt\sim\Gamma(\frac d 2 - k_0,1)} \left[ \dfrac{T^2 - (\sigma'\sqrt{2\rt} - c\sqrt{d})^2}{4\sigma'\sqrt{2\rt}\cdot c\sqrt{d}} \ge 0.5 - \frac{c}{8\sigma} \right] \\
                    \le & \Pr_{\rt\sim\Gamma(\frac d 2 - k_0,1)} \left[\rt \le  \left(1-\frac{c^2}{8\sigma^2}\right)\left(\frac{d}{2} - k_0\right) \right] \\
                    \le & \dfrac{0.48}{0.99}. \hspace{5em} \text{(by \Cref{eq:thm-3-pf-4})}
                \end{aligned}
            \end{equation}
            Plugging this inequality to \Cref{eq:thm-3-pf-7}, we get
            \begin{equation}
                \E_{\vepsilon\sim\gP'} [F_0(\vx_0+\vdelta+\vepsilon) = y_0] \le 0.99 \cdot \frac{0.48}{0.99} + 0.01 = 0.49.
            \end{equation}
            As a result, when $c \le \sigma\sqrt{8/7}$, the smoothed classifier $\wF_0^{\gP'}$ is not robust given the perturbation $\vdelta=(c\sqrt{d},0,0,\cdots,0)^\T$, since there may exist another $y'\neq y_0$ with $\E_{\vepsilon\sim\gP'} [F_0(\vx_0+\vdelta+\vepsilon) = y'] \ge 0.51$ so $\wF_0^{\gP'}(\vx_0+\vdelta) =y' \neq y_0$.
            
            When $c > \sigma\sqrt{8/7}$, within the $\ell_2$ radius ball $c\sqrt{d}$, there exists perturbation vector $\vdelta=(c'\sqrt{d},0,0,\cdots,0)^\T$ fooling smoothed classifier $\wF_0^{\gP'}$ where $c'=\sigma\sqrt{8/7}$.
            Hence, for any $c > 0$, there exists a perturbation within $\ell_2$ ball with radius $c\sqrt{d}$, such that smoothed classifier $\wF_0^{\gP'}$ can be fooled, and then any robustness certification method cannot certify $\ell_2$ radius $c\sqrt{d}$ since the smoothed classifier itself is not robust.
        \end{proof}

\section{Proofs of \shortApproach Computational Method}

    \label{newadx:det-computional-method}
    
    \subsection{Proof of Strong Duality~(\texorpdfstring{\lowercase{\Cref{thm:strong-duality}}}{Theorem 3})}
        \label{newadx:sub:strong-duality-proof}
        
        \begin{proof}[Proof of \Cref{thm:strong-duality}]
            We write down the Lagrangian dual function of  \Cref{eq:primal-opt}:
            \begin{equation}
                \begin{aligned}
                    \Lambda(f,\,\lambda_1,\,\lambda_2) 
                    := & \E_{\vepsilon\sim\gP} [f(\vdelta + \vepsilon)] 
                    - \lambda_1 \left( \E_{\vepsilon\sim\gP} [f(\vepsilon)] - P_A\right) \\
                    &
                    - \lambda_2 \left( \E_{\vepsilon\sim\gQ} [f(\vepsilon)] - Q_A\right).
                \end{aligned}
                \label{eq:new-proof-thm-3-lambda-def}
            \end{equation}
            Then, from $\tC$'s expression~(\Cref{eq:primal}), we have
                \begin{align}
                    & \tC_{\vdelta}(P_A, Q_A) \nonumber \\
                    = & \min_f \E_{\vepsilon\sim\gP} [f(\vepsilon + \vdelta)] \,\mathrm{s.t.}\, 0\le f(\vepsilon) \le 1 \, \forall \vepsilon\in\sR^d, \nonumber \\
                    & \hspace{3em}  \E_{\vepsilon\sim\gP} [f(\vepsilon)] = P_A, \E_{\vepsilon\sim\gQ} [f(\vepsilon)] = Q_A  \nonumber\\
                    = & \min_{f: \sR^d \to [0,1]} \max_{\lambda_1,\lambda_2\in \sR} \Lambda(f,\,\lambda_1,\,\lambda_2) \nonumber \\
                    \overset{(i)}{\ge} & \max_{\lambda_1,\lambda_2\in \sR} \min_{f: \sR^d \to [0,1]}
                    \Lambda(f,\,\lambda_1,\,\lambda_2) \nonumber \\
                    \overset{(ii)}{=} & \max_{\lambda_1,\lambda_2\in\sR}
                    \Pr_{\vepsilon\sim\gP} [p(\vepsilon) < \lambda_1 p(\vepsilon+\vdelta) + \lambda_2 q(\vepsilon+\vdelta)] \nonumber \\
                    & \hspace{2em} - \lambda_1 \Pr_{\vepsilon\sim\gP} [p(\vepsilon-\vdelta) < \lambda_1 p(\vepsilon) + \lambda_2 q(\vepsilon)] \nonumber  \\
                    & \hspace{2em} - \lambda_2 \Pr_{\vepsilon\sim\gQ} [p(\vepsilon-\vdelta) < \lambda_1 p(\vepsilon) + \lambda_2 q(\vepsilon)] \nonumber  \\
                    & \hspace{2em} + \lambda_1 P_A + \lambda_2 Q_A. \label{eq:new-proof-thm-3-ineq}
                \end{align}
                %
            In the above equation, $(i)$ is from the min-max inequality.
            For completeness, we provide the proof as such:
            Define $g: \sR^2 \to \sR$ such that $g(\lambda_1,\lambda_2) := \min_{f:\sR^d \to [0,1]} \Lambda(f,\,\lambda_1,\lambda_2)$.
            As a result, for any $\lambda_1,\lambda_2\in\sR$ and any $f:\sR^d \to [0,1]$,
            $g(\lambda_1,\lambda_2) \le \Lambda(f,\lambda_1,\lambda_2)$.
            So for any $f: \sR^d \to [0,1]$, $\max_{\lambda_1,\lambda_2\in\sR} g(\lambda_1,\lambda_2) \le \max_{\lambda_1,\lambda_2\in \sR} \Lambda(f,\lambda_1,\lambda_2)$, which implies
            \begin{equation}
                \max_{\lambda_1,\lambda_2\in\sR} g(\lambda_1,\lambda_2) \le \min_{f:\sR\to[0,1]} \max_{\lambda_1,\lambda_2\in\sR} \Lambda(f,\lambda_1,\lambda_2),
                \label{eq:new-proof-thm-3-ineq-i}
            \end{equation}
            where LHS is the RHS of $(i)$ and RHS is the LHS of $(i)$.
            
            In above equation, $(ii)$ comes from a closed-form solution of $f$ for $\Lambda(f,\lambda_1,\lambda_2)$ given $(\lambda_1,\lambda_2)\in \sR^2$.
            Notice that we can rewrite $\Lambda(f,\lambda_1,\lambda_2)$ as an integral over $\sR^d$:
            \begin{align}
                & \Lambda(f,\,\lambda_1,\,\lambda_2) \nonumber \\
                = & \E_{\vepsilon\sim\gP} [f(\vdelta+\vepsilon)] - \lambda_1 \E_{\vepsilon\sim\gP} [f(\vepsilon)] - \lambda_2 \E_{\vepsilon\sim\gQ} [f(\vepsilon)] \nonumber \\
                & \hspace{2em} 
                + \lambda_1 P_A + \lambda_2 Q_A \nonumber \\
                = & \int_{\sR^d} f(\vx) \cdot \left( p(\vx - \vdelta) - \lambda_1 p(\vx) - \lambda_2 q(\vx) \right) \dif\vx \nonumber \\
                & \hspace{2em} 
                + \lambda_1 P_A + \lambda_2 Q_A.
                \label{eq:new-proof-them-3-ineq-ii}
            \end{align}
            
            We would like to minimize over $f: \sR^d \to [0,\,1]$ in \Cref{eq:new-proof-them-3-ineq-ii} and simple greedy solution reveals that we should choose
            \begin{equation}
                f(\vx) = \left\{
                \begin{aligned}
                    1, \quad & p(\vx - \vdelta) - \lambda_1p(\vx) - \lambda_2q(\vx) < 0 \\
                    0. \quad & p(\vx - \vdelta) - \lambda_1p(\vx) - \lambda_2q(\vx) \ge 0
                \end{aligned}
                \right.
                \label{eq:proof-thm-3-new-optimal-f}
            \end{equation}
            We inject this $f$ into \Cref{eq:new-proof-them-3-ineq-ii} and get
            \begin{equation}
                \begin{aligned}
                    & \min_{f:\sR^d\to[0,1]} \Lambda(f,\,\lambda_1,\,\lambda_2) \\
                    = & \Pr_{\vepsilon\sim\gP} [p(\vepsilon) < \lambda_1 p(\vepsilon + \vdelta) + \lambda_2 q(\vepsilon + \vdelta)] \\
                    & \hspace{1em} + \lambda_1 \left( P_A - \Pr_{\vepsilon\sim\gP} [p(\vepsilon - \vdelta) < \lambda_1p(\vepsilon) + \lambda_2q(\vepsilon)] \right) \\
                    & \hspace{1em} + \lambda_2 \left( Q_A - \Pr_{\vepsilon\sim\gQ} [p(\vepsilon - \vdelta) < \lambda_1p(\vepsilon) + \lambda_2q(\vepsilon)] \right). \\
                \end{aligned}
                \label{eq:new-proof-thm-3-ineq-ii-2}
            \end{equation}
            Hence $(ii)$ holds.
            
            On the other hand, we know that $\tD_\vdelta(P_A,Q_A)$~(defined by \Cref{eq:dual}) is feasible by theorem statement.
            Denote $(\lambda_1^*, \lambda_2^*)\in \sR^2$ to a feasible solution to $\tD_\vdelta(P_A,Q_A)$ and $d^*$ to the objective value, then from the constraints of $(\tD)$ we know
            \begin{equation}
                \begin{aligned}
                    \Pr_{\vepsilon\sim\gP} [p(\vepsilon-\vdelta) < \lambda_1^* p(\vepsilon) + \lambda_2^* q(\vepsilon)] = P_A, \\
                    \Pr_{\vepsilon\sim\gQ} [p(\vepsilon-\vdelta) < \lambda_1^* p(\vepsilon) + \lambda_2^* q(\vepsilon)] = Q_A, \\
                    \Pr_{\vepsilon\sim\gP} [p(\vepsilon) < \lambda_1^* p(\vepsilon+\vdelta) + \lambda_2^* q(\vepsilon+\vdelta)] = d^*.
                \end{aligned}
                \label{eq:new-proof-thm-3-equalities}
            \end{equation}
            Plugging in these equalities into \Cref{eq:new-proof-thm-3-ineq}, we have
            \begin{equation}
                \tC_\vdelta(P_A,Q_A) 
                \ge d^* - \lambda_1P_A - \lambda_2Q_A + \lambda_1P_A + \lambda_2 Q_A 
                = d^*.
                \label{eq:new-proof-thm-3-oneside}
            \end{equation}
            At the same time, we define function $f^*: \sR^d \to [0,1]$ such that 
            \begin{equation}
                f^*(\vx) = \1[p(\vx-\vdelta) - \lambda_1^* p(\vx) - \lambda_2^* q(\vx) < 0].
            \end{equation}
            From \Cref{eq:new-proof-thm-3-equalities}, $f^*$ satisfies the constraints of ($\tC$)~(\Cref{eq:primal-pa-qa,eq:primal-f-cons}).
            Since $(\tC)$ minimizes over all possible functions $f: \sR^d\to[0,1]$,
            we have
            \begin{equation}
                \tC_\vdelta(P_A,Q_A) \le \E_{\vepsilon\sim\gP} [f^*(\vepsilon+\vdelta)] = d^*.
                \label{eq:new-proof-thm-3-anotherside}
            \end{equation}
            Combining  \Cref{eq:new-proof-thm-3-oneside,eq:new-proof-thm-3-anotherside}, we get $\tC_{\vdelta}(P_A,Q_A) = d^*$ and hence the strong duality holds.
        \end{proof}

    \subsection{Proofs of \texorpdfstring{\lowercase{\Cref{prop:convex,prop:mono,thm:alg-2-good}}}{Results in Section 5.2}}
    
        \label{newadx:sub:preprocessing}
    
        
        \begin{proof}[Proof of \Cref{prop:convex}]
            Suppose $f_1$ is the optimal solution to $\tC_\vdelta(P_A^1,\,Q_A^1)$ and $f_2$ is the optimal solution to $\tC_\vdelta(P_A^2,\,Q_A^2)$.
            Due to the linearity of expectation, $(f_1 + f_2)/2$ satisfies all the constraints of \Cref{eq:primal} for $P_A = (P_A^1 + P_A^2) / 2$ and $Q_A = (Q_A^1 + Q_A^2) / 2$, i.e., $(f_1 + f_2)/2$ is feasible for $P_A = (P_A^1 + P_A^2) / 2$ and $Q_A = (Q_A^1 + Q_A^2) / 2$ with objective value $\left( \tC_\vdelta(P_A^1,\,Q_A^1) + \tC_\vdelta(P_A^2,\,Q_A^2) \right) / 2$.
            Thus, we have 
            \begin{equation}
                \begin{aligned}
                    & \tC_\vdelta\left( \frac{P_A^1 + P_A^2}{2},\, \frac{Q_A^1 + Q_A^2}{2} \right) \le \\
                    & \hspace{2em} 
                    \dfrac{1}{2} \left( \tC_\vdelta(P_A^1,\,Q_A^1) + \tC_\vdelta(P_A^2,\,Q_A^2) \right)
                \end{aligned}
            \end{equation}
            since $\tC$ is a minimization problem.
            By definition, $\tC_\vdelta (P_A,\,Q_A)$ is convex.
        \end{proof}
        
        \begin{remark}
        Since $\tC_\vdelta(P_A,\,Q_A)$ is defined on a compact $\sR^2$ subspace, the convexity implies continuity.
        The continuity property is used in the following proof of \Cref{thm:alg-2-good}.
        \end{remark}
        
        \begin{proof}[Proof of \Cref{prop:mono}]
            Here, we only prove the monotonicity for functions $x \mapsto \min_y \tC_\vdelta (x,\,y)$ and $x \mapsto \argmin_{y} \tC_\vdelta (x,\,y)$.
            The same statement for $y \mapsto \min_x \tC_\vdelta(x,\,y)$ and $y \mapsto \min_x \tC_\vdelta(x,\,y)$ is then straightforward due to the symmetry.
            
            For simplification, we define $\tC'_\vdelta: x \mapsto \min_y \tC_\vdelta(x,\,y)$ and let $\tilde\tC_\vdelta: x\mapsto \argmin_y \tC_\vdelta(x,\,y)$.
            We notice that both functions can be exactly mapped to the constrained optimization problem ($\tC'$) which removes the second constraint in \Cref{eq:primal-pa-qa} in ($\tC$):
            \begin{subequations}
                \begin{align}
                    \underset{f}{\mathrm{minimize}} \quad & \E_{\vepsilon\sim\gP} [f(\vdelta + \vepsilon)] \\
                    \mathrm{s.t.} \quad & \E_{\vepsilon\sim\gP}[f(\vepsilon)] = x,\,\\
                    & 0 \le f(\vepsilon) \le 1 \quad \forall \vepsilon\sim\sR^d.
                \end{align}
            \end{subequations}
            $\tC'_\vdelta(x)$ is the optimal objective to ($\tC'$) and $\bar\tC_\vdelta(x)$ is $\E_{\vepsilon\sim\gQ}[f^*(\vepsilon)]$ where $f^*$ is the optimal solution.
            
            Either based on Neyman-Pearson lemma~[\citeyear{neyman1933ix}] or strong duality, $(\tC')$ is equivalent to $(\tD')$ defined as such:
            \begin{subequations}
                \begin{align}
                    & \Pr_{\vepsilon\sim\gP} [p(\vepsilon) < \lambda p(\vepsilon + \vdelta)] \\
                    \mathrm{s.t.} \quad & \Pr_{\vepsilon\sim\gP} [p(\vepsilon - \vdelta) < \lambda p(\vepsilon)] = x. \label{eq:proof-prop-2-d-prime-cons}
                \end{align}
            \end{subequations}
            For a given $x$, we only need to find $\lambda$ satisfying \Cref{eq:proof-prop-2-d-prime-cons}.
            Then, 
            \begin{align}
                \tC'_\vdelta(x) & = \Pr_{\vepsilon\sim\gP} [p(\vepsilon) < \lambda p(\vepsilon + \vdelta)], \label{eq:proof-prop-2-alter-C-prime} \\
                \bar\tC_\vdelta(x) & = \Pr_{\vepsilon\sim\gQ} [p(\vepsilon - \vdelta) < \lambda p(\vepsilon)]. \label{eq:proof-prop-2-alter-C-bar}
            \end{align}
            
            Now the monotonicity~(what we would like to prove) is apparent.
            For $x_1 < x_2$, from \Cref{eq:proof-prop-2-d-prime-cons}, we have $\lambda_1 < \lambda_2$, since the probability density function $p$ is non-negative.
            Thus, we inject $\lambda_1$ and $\lambda_2$ into \Cref{eq:proof-prop-2-alter-C-prime} and \Cref{eq:proof-prop-2-alter-C-bar}, and yield
            \begin{equation}
                \tC'_\vdelta(x_1) \le \tC'_\vdelta(x_2),\, 
                \bar \tC_\vdelta(x_1) \le \bar \tC_\vdelta(x_2),
            \end{equation}
            which concludes the proof.
        \end{proof}
        
        \begin{proof}[Proof of \Cref{thm:alg-2-good}]
            We discuss the cases according to the branching statement in the algorithm~(\Cref{alg:determine-p-a-q-a}).
            
            If $\underline{q} > \underline{Q_A}$,
            \begin{itemize}
                \item if $\underline{q} \le \overline{Q_A}$, by definition we have $\tC_\vdelta(\underline{P_A},\,\underline{q}) \le \tC_\vdelta(\underline{P_A},\,y)$ for arbitrary $y$.
                According to \Cref{prop:mono}, we also have $\tC_\vdelta(\underline{P_A},\,\underline{q}) \le \tC_\vdelta(x,\,y)$ for arbitrary $x \ge \underline{P_A}$ and arbitrary $y$.
                Given that $(\underline{P_A},\,\underline{q}) \in [\underline{P_A},\,\overline{P_A}] \times [\underline{Q_A},\,\overline{Q_A}]$,
                $(\underline{P_A},\,\underline{q})$ solves \Cref{eq:preprocess-q-a-q-a};
                
                \item if $\underline{q} > \overline{Q_A}$, by convexity, $\tC_\vdelta(\underline{P_A},\,\overline{Q_A}) \le \tC_\vdelta(\underline{P_A},\,y)$ for $y \in [\underline{Q_A},\,\overline{Q_A}]$.
                
                We further show that $\tC_\vdelta(\underline{P_A},\,\overline{Q_A}) \le \tC_\vdelta(x,\,\overline{Q_A})$ for $x \in [\underline{P_A},\,\overline{P_A}]$:
                assume that this is not true, by \Cref{prop:convex}, the function $y \mapsto \argmin_{x} \tC_\vdelta(x,\,y)$ has function value larger than $\underline{P_A}$ at $y = \overline{Q_A}$.
                Since $\tC_\vdelta(0,\,0) = 0$ is the global minimum of $\tC_\vdelta$, the function value at $y = 0$ is $x = 0$.
                By \Cref{prop:mono}, there exists $y_0 \in [0,\,\overline{Q_A}]$ such that $\underline{P_A} = \argmin_x \tC_\vdelta(x,\,y_0)$.
                Then, we get 
                \begin{equation}
                    \begin{aligned}
                        & \tC_\vdelta(\underline{P_A},\,y_0) \\
                        \overset{(i.)}{\le} & \tC_\vdelta(\argmin_x \tC_\vdelta(x,\,\overline{Q_A}),\,\overline{Q_A}) \\
                        \overset{(ii.)}{\le} & \tC_\vdelta(\underline{P_A},\,\overline{Q_A}),
                    \end{aligned}
                    \label{eq:proof-thm-4-1}
                \end{equation}
                where $(i.)$ follows from \Cref{prop:mono} for $y \mapsto \argmin_x \tC_\vdelta(x,\,y)$; $(ii.)$ is implied in the meaning of $\argmin_x \tC_\vdelta(x,\,\overline{Q_A})$.
                Since $y_0 \in [0,\,\overline{Q_A}]$, \Cref{eq:proof-thm-4-1} implies that $\underline{q}$ should be in $[0,\,\overline{Q_A}]$ as well, which violates the branching condition.
                Thus, $\tC_\vdelta(\underline{P_A},\,\overline{Q_A}) \le \tC_\vdelta(x,\,\overline{Q_A})$ for $x\in [\underline{P_A},\,\overline{P_A}]$.
                
                Using \Cref{prop:mono} for function $x \mapsto \argmin_y \tC_\vdelta(x,\,y)$ in interval $[\underline{P_A},\,\overline{P_A}]$ together with \Cref{prop:convex}, we get $\tC_\vdelta(x,\,\overline{Q_A}) \le \tC_\vdelta(x,\,y)$ for $x \in [\underline{P_A},\,\overline{P_A}]$ and $y \in [\underline{Q_A},\,\overline{Q_A}]$.
                Thus, $(\underline{P_A},\,\overline{Q_A})$ solves \Cref{eq:preprocess-q-a-q-a}.
            \end{itemize}
            
            If $\underline{q} \le \underline{Q_A}$,
            \begin{itemize}
                \item if $\underline{p} \le \overline{P_A}$,
                by definition we have $\tC_\vdelta(\max\{\underline{p},\underline{P_A}\},\,\underline{Q_A}) \le \tC_\vdelta(x,\,\underline{Q_A})$ for $x\in[\underline{P_A},\,\overline{P_A}]$.
                According to \Cref{prop:mono} and condition $\underline{q} \le \underline{Q_A}$, we further have $\tC_\vdelta(\max\{\underline{p},\underline{P_A}\},\,\underline{Q_A}) \le \tC_\vdelta(\max\{\underline{p},\underline{P_A}\},y) \le \tC_\vdelta(x,\,y)$ for arbitrary $x \in [\underline{P_A}, \overline{P_A}]$ and $y \in [\underline{Q_A},\overline{Q_A}]$.
                Given that $(\max\{\underline{p},\underline{P_A}\},\,\underline{Q_A}) \in [\underline{P_A},\,\overline{P_A}] \times [\underline{Q_A},\,\overline{Q_A}]$,
                $(\underline{p},\,\underline{Q_A})$ solves \Cref{eq:preprocess-q-a-q-a};
                
                \item if $\underline{p} > \overline{P_A}$,
                according to \Cref{prop:convex}, $\tC_\vdelta(\overline{P_A},\,\underline{Q_A}) \le \tC_\vdelta(x,\,\underline{Q_A})$ for $x\in [\underline{P_A},\,\overline{P_A}]$.
                
                We further show that $\tC_\vdelta(\overline{P_A},\, \underline{Q_A}) \le \tC_\vdelta(\overline{P_A},\, y)$ for $y \in [\underline{Q_A},\,\overline{Q_A}]$:
                assume that this is not true, by \Cref{prop:convex}, the function $x \mapsto \argmin_y \tC_\vdelta(x,\,y)$ has function value larger than $\underline{Q_A}$ at $x = \overline{P_A}$.
                Since $\tC_\vdelta(0,\,0)$ is the global minimum, by \Cref{prop:mono} on $x \mapsto \argmin_y \tC_\vdelta(x,\,y)$, there exists $x_0 \in [0,\,\overline{P_A}]$ such that $\underline{Q_A} = \argmin_y \tC_\vdelta(x_0,\,y)$.
                Then, we get
                \begin{equation}
                    \begin{aligned}
                         \tC_\vdelta(x_0,\,\underline{Q_A}) \\
                        \le & \tC_\vdelta(\overline{P_A},\,\argmin_y \tC_\vdelta(\overline{P_A},\,y)) \\
                        \le & \tC_\vdelta(\overline{P_A},\,\underline{Q_A})
                    \end{aligned}
                    \label{eq:proof-thm-4-2}
                \end{equation}
                following the similar deduction as in \Cref{eq:proof-thm-4-1}.
                Since $x_0 \in [0,\,\overline{P_A}]$, \Cref{eq:proof-thm-4-2} implies that $\underline{p}$ should be in $[0,\,\overline{P_A}]$ as well, which violates the branching condition.
                Thus, $\tC_\vdelta(\overline{P_A},\,\underline{Q_A}) \le \tC_\vdelta(\overline{P_A},\,y)$ for $y \in [\underline{Q_A},\,\overline{Q_A}]$.
                
                Using \Cref{prop:mono} for function $y \mapsto \argmin_x \tC_\vdelta(x,\,y)$ in interval $[\underline{Q_A},\,\overline{Q_A}]$ together with \Cref{prop:convex}, we get $\tC_\vdelta(\overline{P_A},\,y) \le \tC_\vdelta(x,\,y)$ for $y\in [\underline{Q_A},\,\overline{Q_A}]$ and $x\in [\underline{P_A},\,\overline{P_A}]$.
                Thus, $(\overline{P_A},\,\underline{Q_A})$ solves \Cref{eq:preprocess-q-a-q-a}.
            \end{itemize}
        \end{proof}
    
    \subsection{Proof of \texorpdfstring{\lowercase{\Cref{thm:concrete-equations}}}{Theorem 5}}
    
        \label{newadx:sub:proof-concrete-equations}
        
        \begin{proof}[Proof of \Cref{thm:concrete-equations}]
            We first define $r_p(\|\vepsilon\|_2) = p(\vepsilon)$ and $r_q(\|\vepsilon\|_2) = q(\vepsilon)$, then easily seen the concrete expressions of $r_p$ and $r_q$ are:
            \begin{align}
                r_p(t) & = \dfrac{1}{(2\sigma'^2)^{d/2-k} \pi^{d/2}} \cdot \dfrac{\Gamma(d/2)}{\Gamma(d/2-k)}, \label{eq:pf-thm5-1} \\
                r_q(t) & = \dfrac{\nu}{(2\sigma'^2)^{d/2-k} \pi^{d/2}} \cdot \dfrac{\Gamma(d/2)}{\Gamma(d/2-k)},
                \label{eq:pf-thm5-2}
            \end{align}
            where
            \begin{equation}
                \nu := \dfrac{\Gamma(d/2-k)}{\gamma(d/2-k, \frac{T^2}{2\sigma'^2})} > 1
            \end{equation}
            and $\gamma$ is the lower incomplete Gamma function.
        
            Now we use level-set integration similar as the proof in \Cref{lem:thm-2-3} to get the expressions of $P$, $Q$, and $R$ respectively.
            Since $\gP$ and $\gQ$ are $\ell_2$-symmetric, without loss of generality, we let $\vdelta = (r,0,\dots,0)^\T$.
            
            \textbf{(P)}.
            
            Suppose $P_T = r_p(T)$.
            \begin{align*}
                & P(\lambda_1,\lambda_2) \\
                = & \Pr_{\vepsilon\sim\gP = \gNg(k,\sigma)} [p(\vepsilon-\vdelta) < \lambda_1 p(\vepsilon) + \lambda_2 q(\vepsilon)] \\
                = & \int_{\sR^d} \1[p(\vx - \vdelta) < \lambda_1 p(\vx) + \lambda_2 q(\vx)] p(\vx) \dif \vx \\
                = & \int_0^{P_T} y \dif y \int_{\substack{p(\vx)=y\\ p(\vx-\vdelta) < \lambda_1p(\vx)}} \dfrac{\dif \vx}{\|\nabla p(\vx)\|_2} 
                + \\
                & \hspace{1em}
                \int_{P_T}^\infty y \dif y \int_{\substack{p(\vx)=y\\ p(\vx-\vdelta) < (\lambda_1 + \lambda_2\nu) p(\vx)}} \dfrac{\dif \vx}{\|\nabla p(\vx)\|_2} \\
                = & \int_0^{P_T} y\dif y \dfrac{2\pi^{d/2}}{\Gamma(d/2)} r_p^{-1}(y)^{d-1} \left( - \dfrac{1}{r_p'(r_p^{-1}(y))} \right) \cdot \\
                & \hspace{1em} \Pr[ p(\vx-\vdelta) \le \lambda_1 p(\vx) \,|\, p(\vx)=y ] + \\
                & \hspace{1em} \int_{P_T}^{\infty} y\dif y \dfrac{2\pi^{d/2}}{\Gamma(d/2)} r_p^{-1}(y)^{d-1} \left( - \dfrac{1}{r_p'(r_p^{-1}(y))} \right) \cdot \\
                & \hspace{1em} \Pr[p(\vx-\vdelta) < (\lambda_1+\lambda_2\nu) p(\vx) \,|\, p(\vx) = y] \\
                \overset{y=r_p(t)}{=} &
                \int_T^\infty r_p(t)\dif t \dfrac{2\pi^{d/2}}{\Gamma(d/2)} t^{d-1} \cdot \\
                & \hspace{1em} \Pr[p(\vx-\vdelta) < \lambda_1p(\vx) \,|\, \|\vx\|_2 = t] + \\
                & \hspace{1em}
                \int_0^T r_p(t)\dif t \dfrac{2\pi^{d/2}}{\Gamma(d/2)} t^{d-1} \cdot \\
                & \hspace{1em} \Pr[p(\vx-\vdelta) < (\lambda_1 + \lambda_2\nu) p(\vx) \,|\, \|\vx\|_2 = t] \\
                = & \dfrac{1}{\Gamma(d/2-k)} \int_{T^2/(2\sigma'^2)}^\infty t^{d/2-k-1} \exp(-t) \dif t \cdot \\
                & \hspace{1em} \Pr[p(\vx-\vdelta) < \lambda_1p(\vx) \,|\, \|\vx\|_2 = \sigma'\sqrt{2t}] + \\
                & \hspace{1em} \dfrac{1}{\Gamma(d/2-k)} \int_{0}^{T^2/(2\sigma'^2)} t^{d/2-k-1} \exp(-t) \dif t \cdot \\
                & \hspace{1em} \Pr[p(\vx-\vdelta) < (\lambda_1 + \lambda_2\nu) p(\vx) \,|\, \|\vx\|_2 = \sigma'\sqrt{2t}] \\
                = & \E_{t\sim\Gamma(d/2-k,1)}
                \left\{
                \begin{array}{lr}
                    u_3(t,\lambda_1), & t \ge T^2 / (2\sigma'^2) \\
                    u_3(t,\lambda_1+\lambda_2\nu). & t < T^2 / (2\sigma'^2)
                \end{array}
                \right.
            \end{align*}
            Here, $u_3(t,\lambda) = \Pr[p(\vx-\vdelta) < \lambda p(\vx) \,|\, \|\vx\|_2 = \sigma'\sqrt{2t}]$.
            
            \textbf{(Q)}.
            
            Similarly,
            \begin{align*}
                & Q(\lambda_1,\lambda_2) \\
                = & \Pr_{\vepsilon\sim\gQ = \gNgtrunc(k,T,\sigma)} [p(\vepsilon-\vdelta) < \lambda_1p(\vepsilon) + \lambda_2q(\vepsilon)] \\
                = & \int_{\|\vx\|_2\le T} \1[p(\vx-\vdelta) < \lambda_1p(\vx) + \lambda_2 q(\vx)] q(\vx) \dif \vx \\
                = & \int_{\nu P_T}^\infty y \dif y \int_{\substack{q(\vx) = y\\ p(\vx-\vdelta) < (\lambda_1 + \lambda_2\nu) p(\vx)}} \dfrac{\dif \vx}{\|\nabla q(\vx)\|_2}  \\
                \overset{y = r_q(t)}{=} & \int_0^T
                r_q(t) \dif t \dfrac{2\pi^{d/2}}{\Gamma(d/2)} t^{d-1} \cdot \\
                & \hspace{1em}
                \Pr[p(\vx-\vdelta) < (\lambda_1 + \lambda_2 \nu) p(\vx) \,|\, \|\vx\|_2=t] \\
                = & \nu \E_{t\sim \Gamma(d/2-k,1)} u_3(t,\lambda_1 + \lambda_2\nu) \cdot \1\left[ t \le \frac{T^2}{2\sigma'^2} \right].
            \end{align*}
            
            \textbf{(R)}.
            
            Now, for $R$:
            \begin{align*}
                & R(\lambda_1,\lambda_2) \\
                = & \Pr_{\vepsilon\sim\gP=\gNg(k,\sigma)} [p(\vepsilon) < \lambda_1 p(\vepsilon + \vdelta) + \lambda_2 q(\vepsilon + \vdelta)] \\
                = & \int_{\sR^d} \1[p(\vx) < \lambda_1 p(\vx + \vdelta) + \lambda_2 q(\vx + \vdelta)] p(\vx) \dif \vx \\
                = & \E_{t\sim\Gamma(d/2-k,1)} u_4(t,\lambda_1,\lambda_2).
            \end{align*}
            Here, $u_4(t,\lambda_1,\lambda_2)=\Pr[\lambda_1 p(\vx + \vdelta) + \lambda_2 q(\vx + \vdelta) > r_p(\sigma'\sqrt{2t}) \,|\, \|\vx\|_2 = \sigma'\sqrt{2t}]$.
            
            Plugging \Cref{lem:thm-5-1} into $P(\lambda_1,\lambda_2)$ and $Q(\lambda_1,\lambda_2)$, and then plugging \Cref{lem:thm-5-2} into $R(\lambda_1,\lambda_2)$, we yield the desired expressions in theorem statement.
        \end{proof}
        
        \begin{lemma}
            Under the condition of \Cref{thm:concrete-equations},
            let $\vdelta = (r,0,\dots,0)^\T$,
            \begin{equation}
                \begin{aligned}
                    u_3(t,\lambda) & := \Pr[p(\vx-\vdelta) < \lambda p(\vx) \,|\, \|\vx\|_2 = \sigma'\sqrt{2t}] \\
                    & \hspace{-3em} = \BetaCDF_{\frac{d-1}{2}} \left(
                        \dfrac{(r + \sigma'\sqrt{2t})^2}{4r\sigma'\sqrt{2t}} - 
                        \dfrac{2k\sigma'^2 W(\frac t k e^{\frac t k} \lambda^{-\frac 1 k})}{4r\sigma'\sqrt{2t}}
                    \right),
                \end{aligned}
            \end{equation}
            where $W$ is the principal branch of Lambert $W$ function.
            \label{lem:thm-5-1}
        \end{lemma}
        
        \begin{proof}[Proof of \Cref{lem:thm-5-1}]
            $p(\vx-\vdelta) < \lambda p(\vx)$ means that
            $r_p(\|\vx-\vdelta\|_2) < \lambda r_p(\|\vx\|_2)$ and therefore
            $\|\vx-\vdelta\|_2 > r_p^{-1}(\lambda r_p(\|\vx\|_2))$.
            Given that $\|\vx\|_2 = \sigma'\sqrt{2t}$,
            we have
            \begin{equation}
                \left\{
                \begin{aligned}
                    & x_1^2 + \sum_{i=2}^d x_i^2 = 2t\sigma'^2, \\
                    & (x_1-r)^2 + \sum_{i=2}^d x_i^2 \ge r_p^{-1}(\lambda r_p(\sigma'\sqrt{2t}))^2.
                \end{aligned}
                \right.
            \end{equation}
            This is equivalent to
            \begin{equation}
                x_1 \le \dfrac{2t\sigma'^2 + r^2 - r_p^{-1}(\lambda r_p(\sigma' \sqrt{2t}))^2}{2r}.
            \end{equation}
            From the expression of $r_p$~(\Cref{eq:pf-thm5-1}),
            we have
            \begin{equation}
                r_p^{-1}(\lambda r_p(\sigma'\sqrt{2t}))^2 = 2\sigma'^2 k W\left( \frac t k e^{\frac t k} \lambda^{-\frac 1 k} \right).
                \label{eq:pf-lem-5-1-2}
            \end{equation}
            Thus,
            when $\|\vx\|_2 = \sigma'\sqrt{2t}$ and $\vx$ uniformly sampled from this sphere,
            \begin{equation}
                \begin{aligned}
                    & p(\vx-\vdelta) < \lambda p(\vx) \\
                    \iff & x_1 \le \dfrac{2t\sigma'^2 + r^2 - 2\sigma'^2k W(\frac t k e^{\frac t k} \lambda^{-\frac 1 k})}{2r} \\
                    \iff & 
                    \dfrac{1 + \frac{x_1}{\sigma'\sqrt{2t}}}{2} \le \dfrac{(r + \sigma'\sqrt{2t})^2 - 2\sigma'^2 k W\left( \frac t k e^{\frac t k} \lambda^{-\frac 1 k} \right)}{4r\sigma'\sqrt{2t}}.
                \end{aligned}
                \label{eq:pf-lem-5-1-1}
            \end{equation}
             According to \citep[Lemma I.23]{yang2020randomized}, for $\vx$ uniformly sampled from sphere with radius $\sigma'\sqrt{2t}$, the component coordinate $\dfrac{1 + \frac{x_1}{\sigma'\sqrt{2t}}}{2} \sim \Beta(\frac{d-1}{2}, \frac{d-1}{2})$.
             Combining \Cref{eq:pf-lem-5-1-1} with this result concludes the proof.
        \end{proof}
        
        \begin{lemma}
            \label{lem:thm-5-2}
            Under the condition of \Cref{thm:concrete-equations},
            let $\vdelta = (r,0,\dots,0)^\T$,
            \begin{equation}
                \begin{aligned}
                    u_4(t,\lambda_1,\lambda_2) & :=\Pr[\lambda_1 p(\vx + \vdelta) + \lambda_2 q(\vx + \vdelta) > r_p(\sigma'\sqrt{2t}) \\
                    & \hspace{2em} \,|\, \|\vx\|_2 = \sigma'\sqrt{2t}] \\
                    = & \left\{ \begin{array}{lr}
                u_1(t), & \lambda_1 \le 0 \\
                u_1(t) + u_2(t), & \lambda_1 > 0
                \end{array} \right.
                \end{aligned}
            \end{equation}
            where
            \begin{equation*}
                \scriptsize
                \begin{aligned}
                u_1(t) = & \BetaCDF_{\frac{d-1}{2}} \left(
                    \dfrac{\min\{T^2, 2\sigma'^2 k W(\frac{t}{k}e^{\frac t k} (\lambda_1 + \nu \lambda_2)^{\frac 1 k}) \}}{4r\sigma'\sqrt{2t}} \right. \\
                    & \hspace{3em} \left. - 
                    \dfrac{(\sigma'\sqrt{2t} - r)^2}
                    {4r\sigma'\sqrt{2t}}
                \right), \\
                u_2(t) = & \max\left\{ 0,
                    \BetaCDF_{\frac{d-1}{2}} \left(
                    \dfrac{2\sigma'^2k W(\frac{t}{k} e^{\frac t k} \lambda_1^{\frac 1 k}) - (\sigma'\sqrt{2t} - r)^2}{4r\sigma'\sqrt{2t}} 
                    \right) 
                \right. \\
                & \hspace{3em} \left. -\BetaCDF_{\frac{d-1}{2}} \left(
                    \dfrac{T^2 - (\sigma'\sqrt{2t} - r)^2}{4r\sigma'\sqrt{2t}}
                \right) \right\}, \\
                \end{aligned}
            \end{equation*}
        \end{lemma}
        
        \begin{proof}[Proof of \Cref{lem:thm-5-2}]
            Under the condition that $\|\vx\|_2 = \sigma'\sqrt{2t}$, we separate two cases: $q(\vx + \vdelta) > 0$ and $q(\vx + \vdelta) = 0$, which corresponds to $\|\vx + \vdelta\|_2 \le T$ and $\|\vx + \vdelta\|_2 > T$.
            
            \textbf{(1)} $q(\vx + \vdelta) > 0$:
            
            Notice that
            \begin{equation}
                \begin{aligned}
                    & q(\vx + \vdelta) > 0 \\
                    \iff & \|\vx + \vdelta\|_2^2 \le T^2 \\
                    \iff & x_1 \le \dfrac{T^2 - 2t\sigma'^2 - r^2}{2r}.
                \end{aligned}
            \end{equation}
            From \Cref{eq:pf-thm5-2}, $q(\vx + \vdelta) = \nu p(\vx + \vdelta)$.
            Thus,
            \begin{equation}
                \begin{aligned}
                    & \lambda_1p(\vx + \vdelta ) + \lambda_2 q(\vx + \vdelta) > r_p(\sigma'\sqrt{2t}) \\
                    \iff & (\lambda_1 + \nu\lambda_2) p(\vx + \vdelta) \ge r_p(\sigma'\sqrt{2t}) \\
                    \iff & \|\vx + \vdelta\|_2^2 \le r_p^{-1}\left(
                    \dfrac{r_p(\sigma'\sqrt{2t})}{\lambda_1 + \nu\lambda_2} \right)^2 \\
                    \iff & x_1 \le
                    \dfrac{
                    r_p^{-1}\left(
                    \frac{r_p(\sigma'\sqrt{2t})}{\lambda_1 + \nu\lambda_2} \right)^2
                    - 2t\sigma'^2 - r^2}{2r}.
                \end{aligned}
            \end{equation}
            
            Therefore,
            \begin{equation}
                \scriptsize
                \begin{aligned}
                    & \Pr[\lambda_1 p(\vx + \vdelta) + \lambda_2 q(\vx + \vdelta) > r_p(\sigma'\sqrt{2t})  \wedge q(\vx + \vdelta) > 0 \,|\ \|\vx\|_2 = \sigma'\sqrt{2t} ] \\
                    & \hspace{-2em} = \Pr\left[ x_1 \le \dfrac{
                    \min\left\{r_p^{-1}\left(
                    \frac{r_p(\sigma'\sqrt{2t})}{\lambda_1 + \nu\lambda_2} \right)^2, T^2\right\}
                    - 2t\sigma'^2 - r^2}{2r} \,\Big|\, \|\vx\|_2 = \sigma'\sqrt{2t} \right].
                \end{aligned}
            \end{equation}
            By \citep[Lemma I.23]{yang2020randomized} and \Cref{eq:pf-lem-5-1-2}, we thus have
            \begin{equation}
                \begin{aligned}
                    & \Pr[\lambda_1 p(\vx + \vdelta) + \lambda_2 q(\vx + \vdelta) > r_p(\sigma'\sqrt{2t}) \\
                    & \hspace{2em} \wedge q(\vx + \vdelta) > 0 \,|\ \|\vx\|_2 = \sigma'\sqrt{2t} ] \\
                    = & \BetaCDF_{\frac{d-1}{2}} \left(
                    \dfrac{\min\{T^2, 2\sigma'^2 k W(\frac{t}{k}e^{\frac t k} (\lambda_1 + \nu \lambda_2)^{\frac 1 k}) \}}{4r\sigma'\sqrt{2t}} \right. \\
                    & \hspace{3em} \left. - 
                    \dfrac{(\sigma'\sqrt{2t} - r)^2}
                    {4r\sigma'\sqrt{2t}} \right) = u_1(t).
                \end{aligned}
            \end{equation}
            
            \textbf{(2)} $q(\vx + \vdelta) = 0$:
            
            Similarly,
            \begin{equation}
                q(\vx + \vdelta) = 0 \iff x_1 > \dfrac{T^2 - 2t\sigma'^2 - r^2}{2r}.
            \end{equation}
            When $\lambda_1 \le 0$, the condition $\lambda_1 p(\vx + \vdelta) + \lambda_2 q(\vx + \vdelta) = \lambda_1 p(\vx + \vdelta) > r_p(\sigma'\sqrt{2t})$ can never be satisfied.
            When $\lambda_1 > 0$, we have
            \begin{equation}
                \begin{aligned}
                    & \lambda_1 p(\vx + \vdelta) + \lambda_2 q(\vx + \vdelta) > r_p(\sigma'\sqrt{2t}) \\
                    \iff & \|\vx + \vdelta\|_2^2 \le r_p^{-1}\left(
                    \dfrac{r_p(\sigma'\sqrt{2t})}{\lambda_1}\right)^2 \\
                    \iff & x_1 \le 
                    \dfrac{
                    r_p^{-1}\left(
                    \frac{r_p(\sigma'\sqrt{2t})}{\lambda_1} \right)^2
                    - 2t\sigma'^2 - r^2}{2r}.
                \end{aligned}
            \end{equation}
            Therefore, when $\lambda_1 \le 0$,
            \begin{equation}
                \begin{aligned}
                    & \Pr[\lambda_1 p(\vx + \vdelta) + \lambda_2 q(\vx + \vdelta) > r_p(\sigma'\sqrt{2t}) \\
                    & \hspace{2em} \wedge q(\vx + \vdelta) = 0 \,|\ \|\vx\|_2 = \sigma'\sqrt{2t} ] = 0. \\
                \end{aligned}
            \end{equation}
            When $\lambda_1 > 0$, the condition that $lambda_1 p(\vx + \vdelta) + \lambda_2 q(\vx + \vdelta) > r_p(\sigma'\sqrt{2t})$ is equivalent to
            \begin{equation}
                x_1 \in \left( \dfrac{T^2 - 2t\sigma'^2 - r^2}{2r},
                \dfrac{
                    r_p^{-1}\left(
                    \frac{r_p(\sigma'\sqrt{2t})}{\lambda_1} \right)^2
                    - 2t\sigma'^2 - r^2}{2r}
                \right].
            \end{equation}
            Again, by \citep[Lemma I.23]{yang2020randomized} and \Cref{eq:pf-lem-5-1-2}, we have
            \begin{equation}
                \small
                \begin{aligned}
                    & \Pr[\lambda_1 p(\vx + \vdelta) + \lambda_2 q(\vx + \vdelta) > r_p(\sigma'\sqrt{2t}) \\
                    & \hspace{2em} \wedge q(\vx + \vdelta) = 0 \,|\ \|\vx\|_2 = \sigma'\sqrt{2t} ] \\
                    = & \max\left\{ 0,
                    \BetaCDF_{\frac{d-1}{2}} \left(
                    \dfrac{2\sigma'^2k W(\frac{t}{k} e^{\frac t k} \lambda_1^{\frac 1 k}) - (\sigma'\sqrt{2t} - r)^2}{4r\sigma'\sqrt{2t}} 
                    \right) 
                    \right. \\
                    & \hspace{3em} \left. -\BetaCDF_{\frac{d-1}{2}} \left(
                        \dfrac{T^2 - (\sigma'\sqrt{2t} - r)^2}{4r\sigma'\sqrt{2t}}
                    \right) \right\} \\
                    = & u_2(t).
                \end{aligned}
            \end{equation}
            
            \textbf{(3)} Combining the two cases:
            
            Now we are ready to combine the two cases.
            \begin{equation}
                \begin{aligned}
                    & \Pr[\lambda_1 p(\vx + \vdelta) + \lambda_2 q(\vx + \vdelta) > r_p(\sigma'\sqrt{2t}) \\
                        & \hspace{2em} \,|\ \|\vx\|_2 = \sigma'\sqrt{2t} ] \\
                    = & \Pr[\lambda_1 p(\vx + \vdelta) + \lambda_2 q(\vx + \vdelta) > r_p(\sigma'\sqrt{2t}) \\
                    & \hspace{2em} \wedge q(\vx + \vdelta) > 0 \,|\ \|\vx\|_2 = \sigma'\sqrt{2t} ] + \\
                    & \Pr[\lambda_1 p(\vx + \vdelta) + \lambda_2 q(\vx + \vdelta) > r_p(\sigma'\sqrt{2t}) \\
                    & \hspace{2em} \wedge q(\vx + \vdelta) = 0 \,|\ \|\vx\|_2 = \sigma'\sqrt{2t} ] \\
                    = & \left\{ \begin{array}{lr}
                    u_1(t), & \lambda_1 \le 0 \\
                    u_1(t) + u_2(t). & \lambda_1 > 0
                    \end{array} \right.
                \end{aligned}
            \end{equation}
        \end{proof}
        
    \subsection{Discussion on Uniqueness of Feasible Pair}
        \label{adxsubsec:main-approach-uniqueness}
        
        As we sketched in \Cref{subsec:joint-binary-search}, in general cases, the pair that satisfies constraints in \Cref{eq:dual-2} is unique.
        We formally state this finding and prove it in \Cref{thm:truncated-q-uniqueness}.
        
        \begin{theorem}[Uniqueness of Feasible Solution in \Cref{eq:dual-2}]
            Under the same setting of \Cref{thm:concrete-equations}, if $Q_A \in (0, 1)$ and $P_A \in (Q_A/\nu, 1-(1-Q_A)/\nu)$, then there is the pair $(\lambda_1, \lambda_2)$ that satisfies both $P(\lambda_1,\lambda_2)=P_A$ and $Q(\lambda_1,\lambda_2)=Q_A$ is unique.
            \label{thm:truncated-q-uniqueness}
        \end{theorem}
        
        We prove the theorem in the end of this section, which is based on the strict monotonicity of two auxiliary functions: $g(\lambda_1 + \nu\lambda_2) := Q(\lambda_1,\lambda_2)$ and $h(\lambda_1)$~(defined in \Cref{subsec:joint-binary-search}).
        For other types of smoothing distributions $\gP$ and $\gQ$, in \Cref{thm:uniqueness} we characterize and prove a sufficient condition that guarantees the uniqueness of feasible pair.
        
        We observe that the feasible region of $(P_A, Q_A)$ is
        \begin{equation}
            \gR = \{(x,y): y/\nu \le x \le 1 - (1-y)/\nu, 0 \le y \le 1\}.
        \end{equation}
        Therefore,
        the theorem states that when $(P_A, Q_A)$ is an internal point of $\gR$, the feasible solution is unique and we can use our proposed method to find out such a feasible solution and thus solve the dual problem $(\tD)$.
        Now, the edge cases are that $(P_A,Q_A)$ lies on the boundary of $\gR$.
        We discuss all these cases and show that these boundary cases correspond to degenerate problems that are easy to solve respectively:
        \begin{itemize}
            \item $Q_A = 0$:\\
            When $Q_A = 0$, and $P_A \in (0,1)$~(otherwise, trivially $R(\lambda_1, \lambda_2) = P_A \in \{0,1\}$ solves $(\tD)$), we have
            $\lambda_1 + \nu\lambda_2 \to 0^+$ and thus $R(\lambda_1, \lambda_2) = \E_{t\sim\Gamma(d/2-k,1)} u_2(t)$.
            Since $u_2(t)$ only involves $\lambda_1$, we only require $\lambda_1$ to be unique to deploy the method.
            Since $P_A = P(\lambda_1,\lambda_2) = \E_{t\sim\Gamma(d/2-k,1)} u_3(t,\lambda_1) \cdot \1[t \ge T^2/(2\sigma'^2)]$ and $P_A \in (0,1)$, by similar arguments as in \Cref{thm:truncated-q-uniqueness}, we know $\lambda_1$ is unique.
            Hence, all feasible pairs give the same $R(\lambda_1,\lambda_2)$, i.e., have the same objective value and the proposed method that computes a feasible one is sufficient for solving $(\tD)$.
            
            \item $Q_A = 1$:\\
            When $Q_A = 1$ and $P_A \in (0,1)$, we observe that $u_1(t) \le \BetaCDF_{\frac{d-1}{2}} \left( \frac{T^2}{4r\sigma'\sqrt{2t}} - \frac{(\sigma'\sqrt{2t} - r)^2}{4r\sigma'\sqrt{2t}} \right)$ where equality is feasible with the selected $(\lambda_1+\nu\lambda_2) \to +\infty$ and hence the maximum of $\E_{t\sim\Gamma(d/2-k,1)} u_1(t)$ among all feasible $(\lambda_1,\lambda_2)$ is a constant.
            On the other hand, since $u_2(t)$ only involves $\lambda_1$ that is unique as discussed in ``$Q_A=0$'' case, all feasible pairs give the same value of $\E_{t\sim\Gamma(d/2-k,1)} u_2(t)$.
            As a result, the maximum of $R(\lambda_1,\lambda_2)$ can be computed by adding the unique value of $\E_{t\sim\Gamma(d/2-k,1)} u_2(t)$ and the constant corresponding to the maximum of $\E_{t\sim\Gamma(d/2-k,1)} u_1(t)$ among all feasible $(\lambda_1,\lambda_2)$, which solves the dual problem $(\tD)$.
            
            \item $P_A = Q_A / \nu$: \\
            We assume $P_A, Q_A \in (0,1)$~(otherwise covered by former cases).
            In this case, $\lambda_1$ satisfies that $h(\lambda_1) = P_A - Q_A/\nu = 0$, so $\lambda_1 \to 0^+$.
            As a result, $u_2(t) = 0$ for all $t > 0$ and $R(\lambda_1,\lambda_2) = \E_{t\sim\Gamma(d/2-k,1)} u_1(t)$.
            We observe that $u_1(t)$ is only related to $(\lambda_1 + \nu\lambda_2)$ where $(\lambda_1 + \nu\lambda_2)$ satisfying $Q(\lambda_1,\lambda_2) = Q_A$ is unique since $Q_A \in (0,1)$.
            Thus, any feasible $(\lambda_1,\lambda_2)$ would have the same $(\lambda_1 + \nu\lambda_2)$ and hence leads to the same $R(\lambda_1,\lambda_2)$.
            So the proposed method that finds one feasible $(\lambda_1,\lambda_2)$ suffices for solving $(\tD)$.
            
            \item $P_A = 1 - \frac{1-Q_A}{\nu}$:\\
            We again assume $P_A, Q_A \in (0,1)$ (otherwise covered by former cases).
            In this case, $\lambda_1$ satisfies that $h(\lambda_1) = 1 - 1/\nu$.
            Since $\E_{t\sim\Gamma(d/2-k,1)} \1[t < T^2/(2\sigma'^2)] = 1/\nu$, we know $u_3(t,\lambda_1)=1$ for $t \ge T^2 / (2\sigma'^2)$ except a zero-measure set, and thus $\lambda_1 \to +\infty$.
            As a result, $\E_{t\sim\Gamma(d/2-k,1)} u_2(t) = 1 - \E_{t\sim\Gamma(d/2-k,1)} \BetaCDF_{\frac{d-1}{2}} \left( \frac{T^2 - (\sigma'\sqrt{2t} - r)^2}{4r\sigma'\sqrt{2t}} \right)$ is a constant.
            Similar as ``$P_A = Q_A / \nu$'' case, feasible $(\lambda_1 + \nu\lambda_2)$ is unique.
            Therefore, feasible $(\lambda_1,\lambda_2)$ leads to a unique $\E_{t\sim\Gamma(d/2-k,1)} u_1(t)$.
            We compute out these two quantities $\E_{t\sim\Gamma(d/2-k,1)} u_2(t)$ and $\E_{t\sim\Gamma(d/2-k,1)} u_1(t)$ so as to obtain the unique $R(\lambda_1,\lambda_2)$ that solves $(\tD)$.
        \end{itemize}
        
        We remark that in practice, we never observe any instances that correspond to these edge cases though we implemented these techniques for solving them.

        \begin{proof}[Proof of \Cref{thm:truncated-q-uniqueness}]
            The high-level proof sketch is implied in the derivation of our feasible $(\lambda_1,\lambda_2)$ finding method introduced in \Cref{subsec:joint-binary-search}.
            We first show $h(\lambda_1)$ is monotonically strictly increasing in a neighborhood of $\lambda_1$ where $h(\lambda_1) = P_A - Q_A / \nu$, so the $\lambda_1$ that satisfies $P_A - Q_A / \nu$ is unique.
            We then define $g(\gamma) = \nu\E_{t\sim\Gamma(d/2-k,1)} u_3(t,\gamma) \cdot \1[t \le T^2 / (2\sigma'^2)]$, and show its strict monotonicity around the neighborhood of $\lambda^*$ that satisfies $g(\gamma^*) = Q_A$, which indicates that $(\lambda_1 + \nu\lambda_2)$ that satisfies $Q_A$ is unique.
            Combining this two arguments, we know feasible $(\lambda_1,\lambda_2)$ is unique.
            Note that the key in this proof is the local \emph{strict} monotonicity, and the global (nonstrict) monotonicity for both $h(\cdot)$ and $g(\cdot)$ is apparent from their expressions.
            We now present the proofs for the local strict monotonicity of these two functions $h$ and $g$.
            
            (1)~$h(\lambda_1)$ is monotonically strictly increasing in a neighborhood of $\lambda_1$ where $h(\lambda_1) = P_A - Q_A / \nu$.
            
            From the theorem condition, we know that $h(\lambda_1) \in (0, 1 - 1/\nu)$.
            Thus, there exists $t_0 \ge T^2 / (2\sigma'^2)$, such that $u_3(t_0,\lambda_1) \in (0, 1)$~(otherwise $h(\lambda_1) = 0 \text{ or } 1 - 1/\nu$).
            From the closed-form equation of $u_3$, for any neighboring $\lambda_1' \neq \lambda_1$, we will have $W(t_0/k e^{t_0/k} \lambda_1^{-1/k}) \neq W(t_0/k e^{t_0/k} \lambda_1'^{-1/k})$ by the monotonicity of Lambert W function.
            On the other hand, since $u_3(t_0,\lambda_1) \in (0, 1)$ and $\BetaCDF_{\frac{d-1}{2}}(\cdot)$ is also strict monotonic in the neighborhood, we have $u_3(t_0,\lambda_1') \neq u_3(t_0,\lambda_1)$.
            Same happens to $t_0$'s neighborhood, i.e., $\exists \delta > 0$, s.t., $\forall t \in (t_0-\delta, t_0+\delta)$, $u_3(t,\lambda_1') \neq u_3(t,\lambda_1)$ and $\sgn(u_3(t,\lambda_1')-u_3(t,\lambda_1))$ is consistent for any $t\in(t_0-\delta, t_0 + \delta)$.
            As a result, $h(\lambda_1) \neq h(\lambda_1')$. By definition and the fact that $h(\cdot)$ is monotonically non-decreasing, the argument is proved.
            
            (2)~$g(\gamma)$ is monotonically strictly increasing in a neighborhood of $\gamma^*$ where $g(\gamma^*) = Q_A$.
            
            Since $Q_A \in (0,1)$ by the theorem condition, we know that there exists $t_0 \in (0, T^2/(2\sigma'^2))$, such that $u_3(t_0,\gamma^*) \in (0,1)$~(otherwise $g(\gamma^*) = 0 \text{ or } 1$, which contradicts the theorem condition).
            Following the same reasoning as in (1)'s proof, for any $\gamma' \neq \gamma^*$ that lies in a sufficiently small neighborhood of $\gamma^*$, we have $u_3(t_0, \gamma^*) \neq u_3(t_0, \gamma')$, and $\exists \delta > 0$, s.t., $\forall t \in (t_0 - \delta, t_0 + \delta)$, $u_3(t,\gamma^*) \neq u_3(t,\gamma')$ and $\sgn(u_3(t,\gamma^*) - u_3(t,\gamma'))$ is consistent for any $t \in (t_0 - \delta, t_0 + \delta)$.
            As a result, $g(\gamma^*) \neq g(\gamma')$.
            By definition and the fact that $g(\cdot)$ is monotonically non-decreasing, the argument is proved.
        \end{proof}

\section{Extensions of \shortApproach Computational Methods}
    
    In this appendix, we exemplify a few extensions of \shortApproach computational method.
    
    \subsection{Certification with Standard and Truncated Standard Gaussian}
    
        \label{newadx:sub:cert-std-truncated-gaussian}
        
        In main text and \Cref{thm:concrete-equations}, we focus on \shortApproach certification with generalized Gaussian as $\gP$ and truncated generalized Gaussian as $\gQ$, which has theoretical advantages~(\Cref{thm:concentration-sqrt-d}).
        On the other hand, \shortApproach can also be applied to other distributions.
        Concretely, to certify robustness with standard Gaussian as $\gP$ and truncated standard Gaussian as $\gQ$, we can directly plug the following theorem's numerical integration expressions into the described \shortApproach algorithm~(\Cref{alg:DSRS-pipeline}).
    
        \begin{theorem}
            \label{thm:concrete-equations-std-gaussian}
            In $\tD_{\vdelta}(P_A,Q_A)$, let $r = \|\vdelta\|_2$,    
            when $\gP = \gN(\sigma)$ and $\gQ = \gNtrunc(T,\sigma)$,
            let $\nu := \GammaCDF_{d/2}(T^2/(2\sigma^2))^{-1}$,

            \begin{subequations}
                \begin{equation*}
                    \scriptsize
                    \begin{aligned}
                        R(\lambda_1,\lambda_2) := & \Pr_{\vepsilon\sim\gP} [p(\vepsilon) < \lambda_1 p(\vepsilon + \vdelta) + \lambda_2 q(v\epsilon + \vdelta)] \\
                        = & \left\{ \begin{array}{lr}
                        \E_{t\sim\Gamma(d/2,1)} u_1(t), & \lambda_1 \le 0 \\
                        \E_{t\sim\Gamma(d/2,1)} u_1(t) + u_2(t), & \lambda_1 > 0
                        \end{array} \right. \mathrm{where} \\
                        u_1(t) = & \BetaCDF_{\frac{d-1}{2}} \left( \dfrac
                        {\min\{T^2, 2t\sigma^2 + 2\sigma^2\ln(\lambda_1 + \nu\lambda_2)\}}{4r\sigma\sqrt{2t}} \right. \\
                        & \hspace{5em} \left. - \dfrac{(\sigma\sqrt{2t} - r)^2}{4r\sigma\sqrt{2t}}
                         \right), \\
                        u_2(t) = & \max\left\{0, \BetaCDF_{\frac{d-1}{2}} \left( \dfrac{2t\sigma^2 + 2\sigma^2 \ln \lambda_1 - (\sigma\sqrt{2t}-r)^2}{4r\sigma\sqrt{2t}} \right) \right. \\
                        & \hspace{5em} \left. - \BetaCDF_{\frac{d-1}{2}} \left( \dfrac{T^2 - (\sigma\sqrt{2t}-r)^2}{4r\sigma\sqrt{2t}}\right) \right\}.
                    \end{aligned}
                \end{equation*}
                \begin{equation*}
                    \scriptsize
                    \begin{aligned}
                        P(\lambda_1,\lambda_2) & := \Pr_{\vepsilon\sim\gP} [p(\vepsilon-\vdelta) < \lambda_1 p(\vepsilon) + \lambda_2 q(\vepsilon)] \\
                        & = \E_{t\sim\Gamma(d/2,1)} \left\{ \begin{array}{lr}
                        u_3(t,\lambda_1), & t \ge T^2 / (2\sigma^2) \\
                        u_3(t,\lambda_1 + \nu\lambda_2), & t < T^2 / (2\sigma^2).
                        \end{array} \right. \mathrm{where} \\
                        u_3(t,\lambda) & = \BetaCDF_{\frac{d-1}{2}} \left( \dfrac{1}{2} + \dfrac{r^2 + 2\sigma^2\ln\lambda}{4r\sigma\sqrt{2t}} \right).
                    \end{aligned}
                \end{equation*}
                \begin{equation*}
                    \scriptsize
                    \begin{aligned}
                        Q(\lambda_1,\lambda_2) & := \Pr_{\vepsilon\sim\gQ} [p(\vepsilon-\vdelta) < \lambda_1 p(\vepsilon) + \lambda_2 q(\vepsilon)] \\
                        & = \nu \E_{t\sim\Gamma(d/2,1)} u_3(t,\lambda_1 + \nu\lambda_2) \cdot \1[t \le T^2/(2\sigma^2)].
                    \end{aligned}
                \end{equation*}
                In above equations, $\Gamma(d/2,1)$ is gamma distribution and $\GammaCDF_{d/2}$ is its CDF, and $\BetaCDF_{\frac{d-1}{2}}$ is the CDF of distribution $\Beta(\frac{d-1}{2}, \frac{d-1}{2})$.
            \end{subequations}
        \end{theorem}
        
        \begin{proof}[Proof of \Cref{thm:concrete-equations-std-gaussian}]
            The proof largely follows the same procedure as the proof of \Cref{thm:concrete-equations}.
            The only difference is that, since $\gP = \gN(\sigma)$, let $r_p(\|\vepsilon\|_2) = p(\vepsilon)$, different from \Cref{eq:pf-lem-5-1-2},
            now
            \begin{equation}
            \begin{aligned}
                & r_p^{-1}(\lambda r_p(\sigma\sqrt{2t}))^2 \\
                = & -2\sigma^2 \ln(\lambda r_p(\sigma\sqrt{2t}) \cdot (2\pi\sigma^2)^{d/2}) \\
                = & -2\sigma^2 \ln\left( \dfrac{\lambda}{(2\pi\sigma^2)^{d/2}} \exp\left(-\dfrac{2t\sigma^2}{2\sigma^2}\right) \cdot (2\pi\sigma^2)^{d/2} \right) \\
                = & -2\sigma^2 (\ln\lambda - t) = 2t\sigma^2 - 2\sigma^2 \ln\lambda.
            \end{aligned}
            \end{equation}
            By plugging this equation into the proof of \Cref{thm:concrete-equations}, we prove \Cref{thm:concrete-equations-std-gaussian}.
        \end{proof}
    
        In practice, \shortApproach with standard Gaussian and truncated standard Gaussian as smoothing distributions gives marginal improvements over Neyman-Pearson-based certification.
        This is because, for standard Gaussian distribution, the noise magnitude is particularly concentrated on a thin shell as reflected by the green curve in \Cref{fig:smooth-model-landscape-and-magnitude-density}.
        As a result, the truncated standard Gaussian as $\gQ$ either has a tiny density overlap with $\gP$ or provides highly similar information~(i.e., $Q_A \approx P_A$).
        In either case, $\gQ$ provides little additional information.
        Therefore, in practice, we do not use standard Gaussian and truncated standard Gaussian as $\gP$ and $\gQ$, which is also justified by \Cref{thm:concentration-sqrt-d}, though \shortApproach can provide certification for this setting.
    
    \subsection{Certification with Generalized Gaussian with Different Variances}
    
        \label{newadx:sub:cert-gaussian-diff-var}

\begin{table*}[!ht]
    \centering
    \caption{The numerical integration expression for $P$, $Q$, and $R$~(see definition in \Cref{thm:concrete-equations}).
    See \Cref{newadx:sub:cert-gaussian-diff-var} for notation description.}
    \resizebox{\linewidth}{!}{
    \begin{tabular}{cl}
        \toprule
        $P(\lambda_1,\,\lambda_2) =$ & 
        $\underset{x\sim\Gamma(\frac{d}{2}-k)}{\E}
    \mathrm{BetaCDF}_{\frac{d-1}{2}} \left( \dfrac{(r+\sigma'\sqrt{2x})^2 - g^{-1}(\lambda_1 g(\sigma'\sqrt{2x}) + \lambda_2 h(\sigma'\sqrt{2x}))^2}{4r\sigma'\sqrt{2x}} \right) 
    $
    \\                    
    $Q(\lambda_1,\,\lambda_2) =$ &                    $\underset{x\sim\Gamma(\frac{d}{2}-k)}{\E}                \mathrm{BetaCDF}_{\frac{d-1}{2}} \left( \dfrac{(r+\beta'\sqrt{2x})^2 - g^{-1}(\lambda_1 g(\beta'\sqrt{2x}) + \lambda_2 h(\beta'\sqrt{2x}))^2}{4r\beta'\sqrt{2x}} \right)$
        \\
        $R(\lambda_1,\,\lambda_2) =$ & 
        $\underset{x\sim\Gamma(\frac{d}{2}-k)}{\E}
    \mathrm{BetaCDF}_{\frac{d-1}{2}} \left( \dfrac{1}{2} + \dfrac{\overline{(\lambda_1 g + \lambda_2 h)^{-1}} (g(\sigma'\sqrt{2x}))^2 - r^2 - 2x\sigma'^2}{4r\sigma'\sqrt{2x}} \right) - 
    \mathrm{BetaCDF}_{\frac{d-1}{2}} \left( \dfrac{1}{2} + \dfrac{\underline{(\lambda_1 g + \lambda_2 h)^{-1}} (g(\sigma'\sqrt{2x}))^2 - r^2 - 2x\sigma'^2}{4r\sigma'\sqrt{2x}} \right)$ \\
        \bottomrule
    \end{tabular}
    }
    \label{tab:ell-2-numerical-integration}
\end{table*}

\begin{table*}[!t]
    \centering
    \caption{Simplified terms in numerical integration for $P$ and $Q$, where $W$ is the real-valued branch of the Lambert $W$ function.}
    \resizebox{\linewidth}{!}{
    \begin{tabular}{c|c|cc}
        \toprule
        \multicolumn{2}{c|}{Functions} & $P(\lambda_1,\,\lambda_2)$ & $Q(\lambda_1,\,\lambda_2)$ \\
        \hline\hline
        \multicolumn{2}{c|}{Term to Simplify} &
        $g^{-1}\left(\lambda_1 g(\sigma'\sqrt{2x}) + \lambda_2 h(\sigma'\sqrt{2x})\right)^2$ &
        $g^{-1}\left(\lambda_1 g(\beta'\sqrt{2x}) + \lambda_2 h(\beta'\sqrt{2x})\right)^2$ \\
        \hline
        Simplified & $k = 0$ &
        $
        -2\sigma'^2 \ln\left( \lambda_1\exp(-x) + \lambda_2 \left(\frac{\sigma'}{\beta'}\right)^d\exp\left(-\frac{\sigma'^2}{\beta'^2}x\right) \right)
        $
        &
        $
        -2\sigma'^2 \ln\left( \lambda_1\exp\left(-\frac{\beta'^2}{\sigma'^2}x\right) + \lambda_2 \left(\frac{\sigma'}{\beta'}\right)^d \exp(-x)\right)
        $
        \\
        \cline{2-4}
        Terms & $k > 0$ & 
        $
        2k\sigma'^2 W\left(
            \dfrac{x}{k}
            \left( 
                \lambda_1\exp(-x) + \lambda_2 \left(\frac{\sigma'}{\beta'}\right)^{d-2k} \exp\left( -\frac{\sigma'^2}{\beta'^2}x\right)
            \right)^{-1/k}
        \right)
        $
        &
        $
        2k\sigma'^2 W\left(
            \dfrac{x}{k} \cdot \dfrac{\beta'^2}{\sigma'^2}
            \left(
                \lambda_1 \exp\left( -\frac{\beta'^2}{\sigma'^2}\right)
                +
                \lambda_2 \left(\frac{\sigma'}{\beta'}\right)^{d-2k} \exp(-x)
            \right)^{-1/k}
        \right)
        $
        \\
        \midrule
    \end{tabular}
    }
    \label{tab:numerical-integration-simplified}
\end{table*}
        
        We now consider the robustness certification with smoothing distribution $\gP = \gNg(k,\sigma)$ and additional smoothing distribution $\gQ = \gNg(k,\beta)$ where $\sigma$ and $\beta$ are different~(i.e., different variance).
        
        \subsubsection{Computational Method Description}
        
        Hereinafter, for this $\gP$ and $\gQ$
        we define the radial density function $g(r) := p(\vx)$ and $h(r) := q(\vx)$ for any $\|\vx\|_2=r$, where $p$ and $q$ are the density functions of $\gP$ and $\gQ$ respectively.
        $\overline{(\lambda_1 g+\lambda_2 h)^{-1}}(x) := \max y \quad \mathrm{s.t.} \, \lambda_1 g(y) + \lambda_2 h(y) = x$
        and similarly
        $\underline{(\lambda_1 g+\lambda_2 h)^{-1}}(x) := \min y \quad \mathrm{s.t.} \, \lambda_1 g(y) + \lambda_2 h(y) = x$.
        
        In this case, we can still have the numerical expression for $P(\lambda_1,\lambda_2)$, $Q(\lambda_1,\lambda_2)$, and $R(\lambda_1,\lambda_2)$ as shown in \Cref{thm:numerical-integration-dif-var}.

        \begin{theorem}
            When the smoothing distributions 
            $\gP = \gNg(k,\sigma)$ and additional smoothing distribution
            $\gQ = \gNg(k,\beta)$,
            let $P(\lambda_1,\lambda_2)$, $Q(\lambda_1,\lambda_2)$, and $R(\lambda_1,\lambda_2)$ be as defined in \Cref{thm:concrete-equations}, then $P$, $Q$, and $R$ can be computed by expressions in \Cref{tab:ell-2-numerical-integration}.
            \label{thm:numerical-integration-dif-var}
        \end{theorem}

        In \Cref{tab:ell-2-numerical-integration}, the numerical integration requires the computation of several inverse functions.
        In this subsection, we simplify the numerical integration expressions for $P$ and $Q$ by deriving the closed forms of these inverse functions, as shown in \Cref{tab:numerical-integration-simplified}.
        In the actual implementation of numerical integration, for $P$ and $Q$, we use these simplified expressions to compute;
        for $R$, benefited from the unimodality~(\Cref{lem:R-unimodal}), we deploy a simple binary search to compute.
        
        \begin{theorem}
            When the smoothing distributions $\gP=\gNg(k,\sigma)$ and $\gQ=\gNg(k,\beta)$,
            the terms $g^{-1}(\lambda_1 g(\sigma'\sqrt{2x}) + \lambda_2h(\sigma'\sqrt{2x}))^2$ and $g^{-1}(\lambda_1 g(\beta'\sqrt{2x}) + \lambda_2 h(\beta'\sqrt{2x}))^2$ in $P(\lambda_1,\,\lambda_2)$ and $Q(\lambda_1,\,\lambda_2)$'s computational expressions~(see \Cref{tab:ell-2-numerical-integration}) are equivalent to those shown in \Cref{tab:numerical-integration-simplified}.
            \label{thm:numerical-integration-simplified}
        \end{theorem}
        
        With the method to compute $P$, $Q$, and $R$ for given $\lambda_1$ and $\lambda_2$, now the challenge is to solve $\lambda_1$ and $\lambda_2$ such that $P(\lambda_1,\lambda_2) = P_A$ and $Q(\lambda_1,\lambda_2) = Q_A$.
        
        Luckily, as \Cref{thm:uniqueness} shows, for given $P_A$ and $Q_A$, such $(\lambda_1,\lambda_2)$ pair is unique.
        Indeed, such uniqueness holds not only for this $\gP$ and $\gQ$ but also for a wide range of smoothing distributions.
        
        \begin{theorem}[Uniqueness]
            \label{thm:uniqueness}
            Suppose distributions $\gP$ and $\gQ$'s are $\ell_p$-spherically symmetric, i.e., there exists radial density functions $g$ and $h$ such that $p(\vx) = g(\|\vx\|_p)$ and $q(\vx) = h(\|\vx\|_p)$,
            then if $g$ and $h$ are continuous and $\frac{g}{h}$ is continuous and strictly monotonic almost everywhere, for any given $(P_A,\,Q_A) \in \R^2_+$, there is at most one $(\lambda_1,\,\lambda_2)$ pair satisfying constraint of \Cref{eq:dual-2}.
        \end{theorem}
        \vspace{-0.5em}
        The proof is shown in the next subsection, which is based on Cauchy's mean value theorem of the probability integral.
        
        With \Cref{thm:uniqueness} and \Cref{prop:mono}, we can use joint binary search as shown in \Cref{alg:joint-binary-search} to find $\lambda_1$ and $\lambda_2$ that can be viewed as the intersection of two curves.
        At a high level,
        Each time, we leverage the monotonicity to get a point $(\lambda_1^{mid},\,\lambda_2^{mid})$ on the $P$'s curve, then compute corresponding $Q$, and update the binary search intervals based on whether $Q(\lambda_1^{mid},\,\lambda_2^{mid}) > Q_A$.
        We shrink the intervals for both $\lambda_1$ and $\lambda_2$~(Lines $5$ and $7$) in \Cref{alg:joint-binary-search} to accelerate the search.
        The algorithm is plugged into Line $8$ of \Cref{alg:DSRS-pipeline}.

        \begin{algorithm}[H]
            \SetCommentSty{footnotesize}
            \footnotesize\setcounter{AlgoLine}{0}
            \SetAlgoLined
                \KwData{Query access to $P(\cdot,\,\cdot)$ and $Q(\cdot,\,\cdot)$, $P_A$ and $Q_A$}
                \KwResult{$\lambda_1$ and $\lambda_2$ satisfying constraints $P(\lambda_1,\lambda_2) = P_A, Q(\lambda_1,\lambda_2) = Q_A$}
                $\lambda_1^L \gets -M, \lambda_1^U \gets +M, \lambda_2^L \gets -M, \lambda_2^U \gets +M$
                \tcc*[r]{$M$ is a large positive number}
                \While{$\lambda_1^U - \lambda_1^L > eps$}{
                    $\lambda_1^{mid} \gets (\lambda_1^L + \lambda_1^U) / 2$\;
                    Binary search for $\lambda_2^{mid} \in [\lambda_2^L,\,\lambda_2^U]$ such that $P(\lambda_1^{mid},\,\lambda_2^{mid}) = P_A$
                    \tcc*[r]{$(\lambda_1^{mid},\,\lambda_2^{mid})$ lies on red curve}
                    \uIf{$Q(\lambda_1^{mid},\,\lambda_2^{mid}) < Q_A$}
                    {
                        $\lambda_1^U \gets \lambda_1^{mid}, \lambda_2^L \gets \lambda_2^{mid}$
                        \tcc*[r]{$(\lambda_1^{mid},\,\lambda_2^{mid})$ is right to intersection}
                    }
                    \uElse
                    {
                        $\lambda_1^L \gets \lambda_1^{mid}, \lambda_2^U \gets \lambda_2^{mid}$
                        \tcc*[r]{$(\lambda_1^{mid},\,\lambda_2^{mid})$ is left to intersection}
                    }
                }
                \Return{$(\lambda_1^L,\,\lambda_2^L)$} \tcc*[r]{for soundness: $R(\lambda_1^L,\,\lambda_2^L)$ lower bounds $(\tD)$}
            \caption{\textsc{DualBinarySearch} for $\lambda_1$ and $\lambda_2$.}
            \label{alg:joint-binary-search}
        \end{algorithm}
        
        \subsubsection{Proofs}
        
        \begin{proof}[Proof of \Cref{thm:numerical-integration-dif-var}]
            The $\ell_2$-radial density functions of $p(\vx)$ and $q(\vx)$ have these expressions: $g(r) \propto r^{-k} \exp(-r^2/(2\sigma'^2))$ and $h(r) \propto r^{-k} \exp(-r^2/(2\beta'^2))$.
            When $r$ increases, $r^{-k}$, $\exp(-r^2/(2\sigma'^2))$, and $\exp(-r^2/(2\beta'^2))$ decrease so that $g$ and $h$ are both strictly decreasing.
            Therefore, they have inverse functions, which are denoted by $g^{-1}$ and $h^{-1}$.
            Now we are ready to derive the expressions.
            
            \paragraph{(I.)}
            We start with $P$:
            %
                \begin{align*}
                    & P(\lambda_1,\,\lambda_2) \\
                    = & \int_{\sR^d} \1\left[p(\vx - \vdelta) < \lambda_1p(\vx) + \lambda_2q(\vx)\right] p(\vx) \dif\vx \\
                    \overset{(1.)}{=} & \int_0^\infty y \dif y \int_{\substack{p(\vx)=y \\ p(\vx-\vdelta) < \lambda_1 p(\vx) + \lambda_2 h(\vx)}} 
                    \dfrac{\dif \vx}{\|\nabla p(\vx)\|_2} \\
                    \overset{(2.)}{=} & \int_0^\infty y \dif y \int_{\substack{p(\vx)=y \\ p(\vx-\vdelta) < \lambda_1 p(\vx) + \lambda_2 h(\vx)}} - \dfrac{\dif \vx}{g'\left(g^{-1}(y)\right)} \\
                    \overset{(3.)}{=} & \int_0^\infty y \dif y \cdot
                    \dfrac{\pi^{d/2} d g^{-1}(y)^{d-1}}{(d/2)!} 
                    \left(-\dfrac{\dif \vx}{g'(g^{-1}(y))}\right) \cdot \\
                    & \hspace{1em}
                    \Pr\left[ p(\vx - \vdelta) < \lambda_1p(\vx) + \lambda_2q(\vx) \,\big|\, p(\vx) = y\right] \\
                    \overset{(4.)}{=} & \int_0^\infty g(t) \dif t \cdot
                    \dfrac{\pi^{d/2} d t^{d-1}}{(d/2)!} \cdot \\
                    & \hspace{1em} 
                    \Pr\left[ p(\vx - \vdelta) < \lambda_1 g(t) + \lambda_2 h(t) \,\big|\, \|\vx\|_2 = t\right] \\
                    \overset{(5.)}{=} & \int_0^\infty 
                    \dfrac{1}{(2\sigma'^2)^{\frac{d}{2}-k}\pi^{\frac{d}{2}}} \cdot \dfrac{\Gamma(d/2)}{\Gamma(d/2-k)} \cdot \\
                    & \hspace{1em} 
                    t^{-2k} \exp\left( - \dfrac{t^2}{2\sigma'^2} \right) \cdot \dfrac{\pi^{d/2} dt^{d-1}}{(d/2)!} \cdot \\
                    & \hspace{1em}
                     \Pr\left[ p(\vx - \vdelta) < \lambda_1g(t) + \lambda_2h(t) \,\big|\, \|\vx\|_2 = t \right] \\
                    = & \int_0^\infty \dfrac{2}{(2\sigma'^2)^{\frac{d}{2}-k} \Gamma(\frac{d}{2}-k)} t^{d-2k-1} \exp\left( - \dfrac{t^2}{2\sigma'^2} \right) \cdot \\
                    & \hspace{1em}
                    \Pr\left[ p(\vx - \vdelta) < \lambda_1g(t) + \lambda_2h(t) \,\big|\, \|\vx\|_2 = t \right] \\
                    \overset{(6.)}{=} & \dfrac{1}{(2\sigma'^2)^{\frac{d}{2}-k} \Gamma(\frac{d}{2}-k)} \int_0^\infty t^{\frac{d}{2}-k - 1} \exp\left(-\frac{t}{2\sigma'^2}\right) \cdot \\
                    & \hspace{1em} 
                    \Pr\left[ p(\vx-\vdelta) < \lambda_1g(\sqrt t) + \lambda_2h(\sqrt t) \,\big|\, \|\vx\|_2 = \sqrt t\right] \\
                    \overset{(7.)}{=} &
                    \dfrac{1}{\Gamma(\frac{d}{2}-k)} \int_0^\infty t^{\frac{d}{2}-2k-1} \exp(-t) \cdot \\
                    & \hspace{1em} 
                    \Pr\left[ p(\vx-\vdelta) < \lambda_1g(\sigma'\sqrt{2t}) + \lambda_2h(\sigma'\sqrt{2t}) \right. \\
                    & \hspace{8em} \left. 
                    \,\big|\, \|\vx\|_2 = \sigma' \sqrt{2t} \right] \\
                    \overset{(8.)}{=} &
                    \E_{t\sim\Gamma(\frac{d}{2}-k)} \Pr\left[ p(\vx-\vdelta) < \lambda_1g(\sigma'\sqrt{2t}) + \lambda_2h(\sigma'\sqrt{2t}) \right. \\
                    & \hspace{8em} \left. 
                    \,\big|\, \|\vx\|_2 = \sigma' \sqrt{2t} \right].
                \end{align*}
            As a reminder, $d$ is the input dimension.
            The $\Gamma(\cdot)$ here refer to either the Gamma distribution~(in $t\sim\Gamma(\frac{d}{2}-k)$) or Gamma function~(in some denominators).
            In the above integration:
            $(1.)$ uses level-set sliced integration as first proposed in \citep{yang2020randomized};
            $(2.)$ leverages the fact that $p(\vx)$ is $\ell_2$-symmetric and $g'(\cdot) < 0$;
            $(3.)$ injects the surface area of $\ell_2$-sphere with radius $g^{-1}(y)$;
            $(4.)$ alters the integral variable: $t = g^{-1}(y)$, which yields $\dif t = \dif y / g'(t) = \dif y / g'(g^{-1}(y))$ and $y = g(t)$;
            $(5.)$ injects the expression of $g(t)$;
            $(6.)$ alters the integral variable from $t$ to $t^2$;
            $(7.)$ does re-scaling;
            and $(8.)$ observes that the integral can be expressed by expectation over Gamma distribution.
            
            Due to the isotropy, let $r = \|\vdelta\|_2$, we can deem $\vdelta = (r,\,0,\,\dots,\,0)^\T$ by rotating the axis.
            Then we simplify the probability term by observing that
            \begin{equation*}
                \small
                \begin{aligned}
                    & \left\{
                    \begin{aligned}
                        & \|\vx\|_2 = \sigma'\sqrt{2t} \\
                        & p(\vx - \vdelta) < \lambda_1 g(\sigma'\sqrt{2t}) + \lambda_2 h(\sigma'\sqrt{2t}) \\
                    \end{aligned}
                    \right. \\
                    \iff & \left\{
                    \begin{aligned}
                        & x_1^2 + \sum_{i=2}^d x_i^2 = 2t\sigma'^2 \\
                        & (x_1-r)^2 + \sum_{i=2}^d x_i^2 \ge 
                        g^{-1}\left(\lambda_1g(\sigma'\sqrt{2t}) + \lambda_2h(\sigma'\sqrt{2t})\right)^2
                    \end{aligned}
                    \right. \\
                    \Longrightarrow & x_1 \le \dfrac{2t\sigma'^2 - g^{-1}\left(\lambda_1 g(\sigma'\sqrt{2t}) + \lambda_2 h(\sigma'\sqrt{2t}) \right)^2 + r^2}{2r}.
                \end{aligned}
                \label{eq:thm-6-proof-1}
            \end{equation*}
            
            \begin{lemma}[Lemma I.23; \citep{yang2020randomized}]
                If $(x_1, \cdots, x_d)$ is sampled uniformly from the
                unit sphere $\sS^{d-1} \subseteq \sR^d$, then
                \begin{equation}
                    \dfrac{1+x_1}{2} \text{ is distributed as } \mathrm{Beta}\left(\frac{d-1}{2},\,\frac{d-1}{2}\right).
                \end{equation}
                \label{lem:grag-beta}
            \end{lemma}
            Combining \Cref{lem:grag-beta} and \Cref{eq:thm-6-proof-1}, we get
            \begin{equation}
                \small
                \begin{aligned}
                    & \Pr\left[ p(\vx - \vdelta) < \lambda_1 g(\sigma'\sqrt{2t}) + \lambda_2 h(\sigma'\sqrt{2t}) \,\big|\, \|\vx\|_2 = \sigma'\sqrt{2t} \right] \\
                    = & \mathrm{BetaCDF}_{\frac{d-1}{2}} \left( \dfrac{(r + \sigma'\sqrt{2t})^2}{4r\sigma'\sqrt{2t}} \right. \\
                    & \hspace{2em} \left. 
                    - \dfrac{g^{-1}\left(\lambda_1 g(\sigma'\sqrt{2t}) + \lambda_2 h(\sigma'\sqrt{2t})\right)^2}{4r\sigma'\sqrt{2t}} \right).
                \end{aligned}
                \label{eq:thm-6-proof-2}
            \end{equation}
            Injecting \Cref{eq:thm-6-proof-2} into $(8.)$ yields the expression shown in \Cref{tab:ell-2-numerical-integration}.
            
            \paragraph{(II.)} The integration for $Q$ is similar:
            \begin{align*}
                & Q(\lambda_1,\,\lambda_2) \\
                = & \int_{\sR^d} \1[p(\vx-\vdelta) < \lambda_1p(\vx) + \lambda_2 q(x)] q(\vx) \dif \vx \\
                = & \int_0^\infty y \dif y \int_{\substack{q(\vx) = y\\ p(\vx - \vdelta) < \lambda_1 p(\vx) + \lambda_2 q(\vx)}} \dfrac{\dif \vx}{\|\nabla q(\vx)\|_2} \\
                = & \int_0^\infty y \dif y \int_{\substack{q(\vx) = y\\ p(\vx - \vdelta) < \lambda_1 p(\vx) + \lambda_2 q(\vx)}} - \dfrac{\dif \vx}{h'(h^{-1}(y))} \\
                = & \int_0^\infty h(t) \dif t \cdot \dfrac{\pi^{d/2} d t^{d-1}}{(d/2)!} \cdot \\
                & \hspace{1em}
                \Pr\left[ p(\vx-\vdelta) < \lambda_1 p(\vx) + \lambda_2 q(\vx) \,\big|\, \|\vx\|_2 = t \right] \\
                = & \dfrac{1}{(2\beta'^2)^{\frac{d}{2}-k} \Gamma(\frac{d}{2}-k)} \int_0^\infty t^{\frac{d}{2}-k-1} \exp\left(-\dfrac{t}{2\beta'^2}\right)\cdot \\
                & \hspace{1em} 
                \Pr\left[ p(\vx - \vdelta) < \lambda_1 g(\sqrt{t}) + \lambda_2 h(\sqrt{t}) \,\big|\, \|\vx\|_2 = \sqrt t\right] \\
                = & \E_{t\sim\Gamma(\frac{d}{2}-k)} \Pr\left[ p(\vx-\vdelta) < \lambda_1 g(\beta'\sqrt{2t}) + \lambda_2h(\beta'\sqrt{2t}) \right. \\
                & \hspace{8em} \left. 
                \,\big|\, \|\vx\|_2 = \beta'\sqrt{2t} \right],
            \end{align*}
            where
            \begin{equation}
                \begin{aligned}
                    & \hspace{-2em} \Pr\left[ p(\vx - \vdelta) < \lambda_1 g(\beta'\sqrt{2t}) + \lambda_2 h(\beta'\sqrt{2t}) \,\big|\, \|\vx\|_2 = \beta'\sqrt{2t} \right] \\
                    = & \mathrm{BetaCDF}_{\frac{d-1}{2}} \left(
                    \dfrac{(r + \beta'\sqrt{2t})^2}{4r\beta'\sqrt{2t}} \right. \\
                    & \hspace{2em} \left. 
                    - \dfrac{g^{-1}\left(\lambda_1 g(\beta'\sqrt{2t}) + \lambda_2 h(\beta'\sqrt{2t})\right)^2}{4r\beta'\sqrt{2t}}\right).
                \end{aligned}
            \end{equation}
            
            \paragraph{(III.)} Finally, we derive the integration for $R$:
            \begin{align*}
                & R(\lambda_1,\,\lambda_2) \\
                = & \int_{\sR^d} \1\left[p(\vx-\vdelta) < \lambda_1p(\vx) + \lambda_2 q(\vx)\right] p(\vx - \vdelta) \dif \vx \\
                = & \int_{\sR^d} \1\left[p(\vx) < \lambda_1p(\vx + \vdelta) + \lambda_2 q(\vx + \vdelta)\right] p(\vx) \dif \vx \\
                = & \int_0^\infty y \dif y \int_{\substack{p(\vx) = y\\ p(\vx) < \lambda_1p(\vx+\vdelta) + \lambda_2q(\vx+\vdelta)}} \dfrac{\dif \vx}{\|\nabla p(\vx)\|_2} \\
                = & \int_0^\infty y \dif y \int_{\substack{p(\vx) = y\\ p(\vx) < \lambda_1p(\vx+\vdelta) + \lambda_2q(\vx+\vdelta)}} -\dfrac{\dif \vx}{g'(g^{-1}(y))} \\
                = & \int_0^\infty g(t)\dif t \cdot \dfrac{\pi^{d/2}dt^{d-1}}{(d/2)!} \cdot \\
                & \hspace{1em} 
                \Pr\left[ \lambda_1 p(\vx+\vdelta) + \lambda_2 q(\vx+\vdelta) > p(\vx) \,\big|\, \|\vx\|_2 = t \right] \\
                = & \dfrac{1}{(2\sigma'^2)^{\frac{d}{2}-k} \Gamma(\frac{d}{2}-k)} \int_0^\infty t^{\frac{d}{2}-k-1} \exp\left(-\frac{t}{2\sigma'^2}\right) \cdot \\
                & \hspace{1em} 
                \Pr\left[ \lambda_1 p(\vx+\vdelta) + \lambda_2 q(\vx + \vdelta) > g(\sqrt t) \,\big|\, \|\vx\|_2 = \sqrt t \right] \\
                \overset{(9.)}{=} & \E_{t\sim\Gamma(\frac{d-k}{2})} \Pr\left[ \lambda_1 p(\vx+\vdelta) + \lambda_2 q(\vx+\vdelta) > g(\sigma' \sqrt{2t}) \right. \\
                & \hspace{8em} \left. 
                \,\big|\, \|\vx\|_2 = \sigma'\sqrt{2t} \right].
            \end{align*}
            To simplify the probability term, this time we have
            \begin{align}
                & \left\{
                \begin{aligned}
                    & \|\vx\|_2 = \sigma'\sqrt{2t}, \\
                    & \lambda_1p(\vx+\vdelta) + \lambda_2q(\vx+\vdelta) > g(\sigma'\sqrt{2t})
                \end{aligned}
                \right. \\
                \iff & \left\{
                \begin{aligned}
                    & x_1^2 + \sum_{i=2}^d x_i^2 = 2t\sigma'^2 \\
                    & (\lambda_1 g + \lambda_2 h)\left( \sqrt{(x_1 + r)^2 + \sum_{i=2}^d x_i^2} \right) > g(\sigma' \sqrt{2t}) \\
                \end{aligned}
                \right. \\
                \Longrightarrow & (\lambda_1g + \lambda_2h)\left(\sqrt{2t\sigma'^2 + r^2 + 2x_1r}\right) > g(\sigma'\sqrt{2t}),
                \label{eq:thm-6-proof-3}
            \end{align}
            where $r = \|\vdelta\|_2$ and we deem $\vdelta = (r,\,0,\,\dots,\,0)^\T$ by rotating the axis.
            
            \begin{lemma}
                The function $(\lambda_1g + \lambda_2h)$ is unimodal in its domain $(0,\,+\infty)$.
                \label{lem:R-unimodal}
            \end{lemma}
            
            \begin{proof}[Proof of \Cref{lem:R-unimodal}]
                We expand $(\lambda_1 g + \lambda_2h)$ then consider its derivative.
                \begin{equation}
                    \begin{aligned}
                        & \lambda_1 g(r) + \lambda_2 h(r) \\
                        = & C_1 \lambda_1 r^{-2k} \exp\left(-\frac{r^2}{2\sigma'^2}\right) + C_2 \lambda_2 r^{-2k} \exp\left(-\frac{r^2}{2\beta'^2}\right),
                    \end{aligned}
                \end{equation}
                where $C_1$ and $C_2$ is the constant normalization coefficient of $g$ and $h$ respectively.
                \begin{align}
                    & (\lambda_1 g + \lambda_2 h)'(r) \\
                    = & - \left( C_1 \lambda_1 2k r^{-2k-1} \exp\left(-\frac{r^2}{2\sigma'^2}\right) \right. \nonumber \\
                    & \hspace{1em} 
                    + C_1 \lambda_1 \cdot \dfrac{r^{-2k+1}}{\sigma'^2} \exp\left(-\frac{r^2}{2\sigma'^2}\right) \nonumber \\
                    & \hspace{1em} + C_2\lambda_2 2kr^{-2k-1} \exp\left(-\frac{r^2}{2\beta'^2}\right) \nonumber \\
                    & \hspace{1em} 
                    + \left. C_2\lambda_2 \cdot \frac{r^{-2k+1}}{\beta'^2} \exp\left(-\frac{r^2}{2\beta'^2}\right) \right). \nonumber
                \end{align}
                Now we show that
                \begin{equation}
                    (\lambda_1g + \lambda_2h)'(r) = 0
                \end{equation}
                has at most one solution.
                
                When either $\lambda_1 = 0$ or $\lambda_2 = 0$, since both $g$ and $h$ are monotonic, $\lambda_1g + \lambda_2$ is monotonic and there is no solution.
                Thus, we assume $\lambda_1,\,\lambda_2 \neq 0$.
                We observe that 
                \begin{align}
                    & (\lambda_1g + \lambda_2h)'(r) = 0 \\
                    \iff & C_1 \lambda_1 \left( 2k + \dfrac{r^2}{\sigma'^2} \right) \exp\left(-\frac{r^2}{2\sigma'^2}\right) \nonumber \\
                    & \hspace{1em} 
                    + C_2\lambda_2 \left( 2k + \dfrac{r^2}{\beta'^2} \right) \exp\left(-\dfrac{r^2}{2\beta'^2}\right) = 0 \\
                    \iff & - \dfrac{C_1\lambda_1}{C_2\lambda_2} \cdot \dfrac{2k + \frac{r^2}{\sigma'^2}}{2k + \frac{r^2}{\beta'^2}} =
                    \exp\left(\frac{r^2}{2\sigma'^2} -\frac{r^2}{2\beta'^2}\right) \\
                    \overset{x := r^2}{\iff} & - \dfrac{C_1\lambda_1}{C_2\lambda_2} \cdot \dfrac{2k + \frac{x}{\sigma'^2}}{2k + \frac{x}{\beta'^2}} =
                    \exp\left(\frac{x}{2\sigma'^2} -\frac{x}{2\beta'^2}\right) \\
                    \iff & - \dfrac{C_2\lambda_2}{C_1\lambda_1} \cdot \dfrac{2k + \frac{x}{\beta'^2}}{2k + \frac{x}{\sigma'^2}} =
                    \exp\left(\frac{x}{2\beta'^2} -\frac{x}{2\sigma'^2}\right).
                    \label{eq:lem-R-unimodel-pf-1}
                \end{align}
                We focus on \Cref{eq:lem-R-unimodel-pf-1}.
                Without loss of generality, we assume $\sigma' > \beta'$, 
                then both function $x \mapsto \frac{2k + x/\beta'^2}{2k + x/\sigma'^2}$ and function $x \mapsto \exp\left(x / (2\beta'^2) - x / (2\sigma'^2)\right)$ are monotonically increasing.
                If $\lambda_1\lambda_2 > 0$, the LHS and RHS of \Cref{eq:lem-R-unimodel-pf-1} are continuous and monotonic in opposite directions.
                Thus, there is at most one solution to \Cref{eq:lem-R-unimodel-pf-1}.
                If $\lambda_1\lambda_2 < 0$, the LHS is monotonic increasing but the derivative is decreasing because
                \begin{equation}
                    - \dfrac{C_2\lambda_2}{C_1\lambda_1} \cdot \dfrac{2k+\frac{x}{\beta'^2}}{2k + \frac{x}{\sigma'^2}}
                    = - \dfrac{C_2\lambda_2}{C_1\lambda_1} \cdot \dfrac{1 + \frac{x}{2k\beta'^2}}{1 + \frac{x}{2k\sigma'^2}}
                \end{equation}
                where the numerator $1 + \frac{x}{2k\beta'^2}$ is linearly increasing and the denominator $1 + \frac{x}{2k\sigma'^2}$ is also linearly increasing.
                On the other hand, the RHS is monotonic increasing and the derivative is also increasing because
                \begin{equation}
                    \frac{1}{2\beta'^2} - \frac{1}{2\sigma'^2} > 0.
                \end{equation}
                As a result, the difference function between RHS and LHS is monotone and there is at most one solution.
                Thus, we have shown \Cref{eq:lem-R-unimodel-pf-1} has at most one solution.
                
                Given that $(\lambda_1 g + \lambda_2 h)'$ is also continuous, we thus know the function $(\lambda_1 g + \lambda_2 h)$ is unimodal.
            \end{proof}
            Moreover, since $g$ and $h$ approach $0$ when $r \to \infty$, $(\lambda_1g + \lambda_2h)$ approaches $0$ when $r \to \infty$.
            
            We define
            \begin{align}
                & \overline{(\lambda_1g + \lambda_2h)^{-1}}(y) := \max y' \nonumber \\
                & \hspace{2em} 
                \quad \mathrm{s.t.} \quad \lambda_1 g(y') + \lambda_2 h(y') = x, \\
                & \underline{(\lambda_1g + \lambda_2h)^{-1}}(y) := \min y' \nonumber \\
                & \hspace{2em} 
                \quad \mathrm{s.t.} \quad \lambda_1 g(y') + \lambda_2 h(y') = x.
            \end{align}
            Then, \Cref{lem:R-unimodal} and $(\lambda_1g + \lambda_2h) \to 0$ when $r\to\infty$ imply that,
            when $y > 0$,
            \begin{equation}
            \begin{aligned}
                & y_0 \in \left(\underline{(\lambda_1g + \lambda_2h)^{-1}}(y),\, \overline{(\lambda_1g + \lambda_2h)^{-1}}(y)\right) \\
                \iff & \lambda_1g(y_0) + \lambda_2h(y_0) > y.
            \end{aligned}
            \end{equation}
            We simplify \Cref{eq:thm-6-proof-3} by observing that $g(\sigma'\sqrt{2t}) > 0$:
            \begin{align}
                & \text{\Cref{eq:thm-6-proof-3}} \nonumber \\
                \iff &
                \sqrt{2t\sigma'^2 + r^2 + 2x_1r} \nonumber \\ 
                & \in 
                \left( 
                \underline{(\lambda_1g + \lambda_2h)^{-1}}(g(\sigma'\sqrt{2t})), \right. \nonumber \\
                & \hspace{4em} \left.
                \overline{(\lambda_1g + \lambda_2h)^{-1}}(g(\sigma'\sqrt{2t})) \right) \\
                \iff & \dfrac{\underline{(\lambda_1g + \lambda_2h)^{-1}}(g(\sigma'\sqrt{2t}))^2 - 2t\sigma'^2 - r^2}{2r} \le x_1 \nonumber \\
                & \hspace{2em} 
                \le \dfrac{\overline{(\lambda_1g + \lambda_2h)^{-1}}(g(\sigma'\sqrt{2t}))^2 - 2t\sigma'^2 - r^2}{2r}. \label{eq:thm-6-proof-4}
            \end{align}
            Combining \Cref{lem:grag-beta} with \Cref{eq:thm-6-proof-4} we get
            \begin{equation}
                \small
                \resizebox{\linewidth}{!}{
                $
                \begin{aligned}
                    & \Pr\left[ \lambda_1p(\vx + \vdelta) + \lambda_2 q(\vx + \vdelta) > g(\sigma'\sqrt{2t}) \,\big|\, \|\vx\|_2 = \sigma'\sqrt{2t} \right] \\
                    = & \mathrm{BetaCDF}_{\frac{d-1}{2}}
                    \left(
                        \dfrac{1}{2} + 
                        \dfrac{\overline{(\lambda_1g + \lambda_2h)^{-1}}(g(\sigma'\sqrt{2t}))^2 - 2t\sigma'^2 - r^2}{4r\sigma'\sqrt{2t}}
                    \right) \\
                    & \hspace{4em} 
                    - \mathrm{BetaCDF}_{\frac{d-1}{2}}
                    \left(
                        \dfrac{1}{2} + 
                        \dfrac{\underline{(\lambda_1g + \lambda_2h)^{-1}}(g(\sigma'\sqrt{2t}))^2 - 2t\sigma'^2 - r^2}{4r\sigma'\sqrt{2t}}
                    \right).
                \end{aligned}
                $
                }
            \end{equation}
            Combining the above equation with $(9.)$ yields the expression shown in \Cref{tab:ell-2-numerical-integration}.
        \end{proof}
    
        \begin{proof}[Proof of \Cref{thm:numerical-integration-simplified}]
            To prove the theorem, the main work we need to do is deriving the closed-form expression for the inverse function $g^{-1}$, where
            \begin{equation}
                g(r) = \dfrac{1}{(2\sigma'^2)^{\frac{d}{2}-k} \pi^{\frac d 2}} \cdot \dfrac{\Gamma(\frac d 2)}{\Gamma(\frac{d}{2}-k)} r^{-2k} \exp\left(-\dfrac{r^2}{2\sigma^2}\right).
                \label{eq:thm-7-pf-def-g}
            \end{equation}
            
            \paragraph{(I.)} When $k=0$,
            \begin{align}
                & 
                \text{\Cref{eq:thm-7-pf-def-g}} \nonumber \\
                \iff & y = \dfrac{1}{(2\sigma'^2\pi)^{\frac d 2}} \exp\left(
                - \dfrac{g^{-1}(y)^2}{2\sigma'^2} \right) \\
                \iff & g^{-1}(y)^2 = -2\sigma'^2 \ln\left( (2\sigma'^2\pi)^{\frac d 2} y \right).
                \label{eq:thm-7-pf-1}
            \end{align}
            
            \paragraph{(II.)} When $k>0$,
            we notice that the Lambert $W$ function $W$ is the inverse function of $w \mapsto we^w$, i.e., $W(x) \exp(W(x)) = x$.
            We let the normalizing coefficient of $g(r)$ be
            \begin{equation}
                C := \dfrac{1}{(2\sigma'^2)^{\frac{d}{2}-k} \pi^{\frac d 2}} \cdot \dfrac{\Gamma(\frac d 2)}{\Gamma(\frac{d}{2}-k)}.
            \end{equation}
            Then
            \begin{align}
                & \text{\Cref{eq:thm-7-pf-def-g}} \nonumber \\
                \iff & y = Cg^{-1}(y)^{-2k} \exp\left(-\dfrac{g^{-1}(y)^2}{2\sigma'^2}\right) \\
                \iff & \dfrac{C}{y} = g^{-1}(y)^2k \exp\left( \dfrac{g^{-1}(y)^2}{2\sigma'^2} \right) \\
                \iff & \left(\dfrac{C}{y}\right)^{\frac{1}{k}} = g^{-1}(y)^2 \exp\left( \dfrac{g^{-1}(y)^2}{k\sigma'^2} \right) \\
                \iff & \dfrac{1}{2k\sigma^2} \left(\dfrac{C}{y}\right)^{\frac{1}{k}} = W^{-1}\left(\dfrac{g^{-1}(y)^2}{2k\sigma'^2}\right) \\
                \iff & g^{-1}(y)^2 = 2\sigma'^2 k W\left( 
                \dfrac{1}{2k\sigma'^2} \left(\frac{C}{y}\right)^{\frac 1 k}\right).
                \label{eq:thm-7-pf-2}
            \end{align}
            
            Plugging in \Cref{eq:thm-7-pf-1,eq:thm-7-pf-2} to $g^{-1}(\lambda_1 g(\sigma'\sqrt{2x}) + \lambda_2h(\sigma'\sqrt{2x}))^2$ and $g^{-1}(\lambda_1 g(\beta'\sqrt{2x}) + \lambda_2 h(\beta'\sqrt{2x}))^2$
            for $k = 0$ and $k > 0$ case, then results in \Cref{tab:numerical-integration-simplified} follow from algebra.
        \end{proof}

        \begin{proof}[Proof of \Cref{thm:uniqueness}]
            We prove the theorem by contradiction.
            Suppose that the $(\lambda_1,\,\lambda_2)$ that satisfy the constraint of  \Cref{eq:dual-2} are not unique, and we let $(\lambda_1^a,\,\lambda_2^a)$ and $(\lambda_1^b,\,\lambda_2^b)$ be such two pairs.
            Without loss of generality, we assume $\lambda_1^a \neq \lambda_1^b$.
            
            If $\lambda_2^a = \lambda_2^b$, we have $P(\lambda_1^a,\,\lambda_2^a) = P(\lambda_1^b,\,\lambda_2^b)$, i.e., 
            the region 
            \begin{equation}
                \begin{aligned}
                    \left\{\vx - \vdelta:\, p(\vx - \vdelta) \in \right. &
                    \left[\min\{\lambda_1^a,\,\lambda_1^b\} p(\vx) + \lambda_2^a q(\vx), \right. \\
                    &
                    \left.\left.\max\{\lambda_1^a,\,\lambda_1^b\} p(\vx) + \lambda_2^a q(\vx)\right)
                    \right\}
                \end{aligned}
            \end{equation}
            has zero mass under distribution $\gP$.
            Given that $\gP$ and $\gQ$ have positive and continuous density functions almost everywhere, the volume of the region is non-zero thus the mass under distribution $\gP$ is also non-zero.
            Therefore, we also have $\lambda_2^a \neq \lambda_2^b$.
            Because of the partial monotonicity of $P$ and $Q$ functions~(shown in \Cref{subsec:joint-binary-search}), without loss of generality, we assume
            \begin{equation}
                \lambda_1^a < \lambda_1^b, \quad \lambda_2^a > \lambda_2^b.
                \label{eq:thm-5-assume}
            \end{equation}
            
            \begin{lemma}
                There exists a point $r_0 \ge 0$,
                either (1) or (2) is satisfied.
                \begin{enumerate}[label=(\arabic*),leftmargin=*]
                    \item When $r > r_0$,
                    \begin{align*}
                         \Pr [p(\vepsilon-\vdelta) < \lambda_1^a p(\vepsilon) + \lambda_2^a q(\vepsilon) \,\big|\, \|\vepsilon\|_p = r] \\
                         \ge
                        \Pr [p(\vepsilon-\vdelta) < \lambda_1^b p(\vepsilon) + \lambda_2^b q(\vepsilon) \,\big|\, \|\vepsilon\|_p = r]; \\
                    \end{align*}
                    when $r < r_0$,
                    \begin{align*}
                        \Pr [p(\vepsilon-\vdelta) < \lambda_1^a p(\vepsilon) + \lambda_2^a q(\vepsilon) \,\big|\, \|\vepsilon\|_p = r ] \\
                        \le
                        \Pr [p(\vepsilon-\vdelta) < \lambda_1^b p(\vepsilon) + \lambda_2^b q(\vepsilon) \,\big|\, \|\vepsilon\|_p = r].
                    \end{align*}
                    
                    \item When $r > r_0$,
                    \begin{align*}
                        \Pr [p(\vepsilon-\vdelta) < \lambda_1^a p(\vepsilon) + \lambda_2^a q(\vepsilon) \,\big|\, \|\vepsilon\|_p = r ] \\
                        \le
                        \Pr [p(\vepsilon-\vdelta) < \lambda_1^b p(\vepsilon) + \lambda_2^b q(\vepsilon) \,\big|\, \|\vepsilon\|_p = r]; \\
                    \end{align*}
                    when $r < r_0$,
                    \begin{align*}
                        \Pr [p(\vepsilon-\vdelta) < \lambda_1^a p(\vepsilon) + \lambda_2^a q(\vepsilon) \,\big|\, \|\vepsilon\|_p = r ] \\
                        \ge
                        \Pr [p(\vepsilon-\vdelta) < \lambda_1^b p(\vepsilon) + \lambda_2^b q(\vepsilon) \,\big|\, \|\vepsilon\|_p = r].
                    \end{align*}
                \end{enumerate}
                \label{lem:thm-5-mono}
            \end{lemma}
            Note that in $\Pr[\cdot \,\big|\, \|\vepsilon\|_p = r]$, the vector $\varepsilon\in\sR^d$ is uniformly sampled from the $\ell_p$-sphere of radius $r$.
            \begin{proof}[Proof of \Cref{lem:thm-5-mono}]
                For a given $r$, since $\gP$ and $\gQ$ are both $\ell_p$-spherically symmetric, and the density functions are both positive almost everywhere, as long as 
                \begin{equation}
                    \lambda_1^a g(r) + \lambda_2^a h(r) \le \lambda_1^b g(r) + \lambda_2^b h(r),
                    \label{eq:thm-5-lem-pf-1}
                \end{equation}
                then
                \begin{equation}
                \begin{aligned}
                    & [p(\vepsilon-\vdelta) < \lambda_1^a p(\vepsilon) + \lambda_2^a q(\vepsilon) \,\big|\, \|\vepsilon\|_p = r ] \\
                    \le &
                    [p(\vepsilon-\vdelta) < \lambda_1^b p(\vepsilon) + \lambda_2^b q(\vepsilon) \,\big|\, \|\vepsilon\|_p = r].
                \end{aligned}
                \label{eq:thm-5-lem-pf-2}
                \end{equation}
                It still holds if we change the ``$\le$``'s to ``$\ge$'''s in both \Cref{eq:thm-5-lem-pf-1,eq:thm-5-lem-pf-2}.
                Meanwhile,
                \begin{equation}
                    \lambda_1^a g(r) + \lambda_2^a h(r) \le \lambda_1^b g(r) + \lambda_2^b h(r)
                    \iff \dfrac{g(r)}{h(r)} > \dfrac{\lambda_2^b - \lambda_2^a}{\lambda_1^a - \lambda_1^b}.
                    \label{eq:thm-5-lem-pf-4}
                \end{equation}
                Since $g(r) / h(r)$ is strictly monotonic, there exists \emph{at most} one point $r_0 \ge 0$ that divides 
                \begin{equation}
                    \dfrac{g(r)}{h(r)} > \dfrac{\lambda_2^b - \lambda_2^a}{\lambda_1^a - \lambda_1^b}
                    \quad \text{and} \quad 
                    \dfrac{g(r)}{h(r)} < \dfrac{\lambda_2^b - \lambda_2^a}{\lambda_1^a - \lambda_1^b}.
                    \label{eq:thm-5-lem-pf-3}
                \end{equation}
                If that $r_0$ exists, from \Cref{eq:thm-5-lem-pf-1,eq:thm-5-lem-pf-2,eq:thm-5-lem-pf-4} the lemma statement follows.
                
                Now we only need to show that $r_0$ exists.
                Assume that the point $r_0$ does not exist, it implies that for all $r$, we have either
                \begin{equation}
                    \lambda_1^a g(r) + \lambda_2^a h(r) < \lambda_1^b g(r) + \lambda_2^b h(r)
                \end{equation}
                or
                \begin{equation}
                    \lambda_1^a g(r) + \lambda_2^a h(r) > \lambda_1^b g(r) + \lambda_2^b h(r)
                \end{equation}
                while $P(\lambda_1^a,\,\lambda_2^a) = P(\lambda_1^b,\,\lambda_2^b) > 0$ and $Q(\lambda_1^a,\,\lambda_2^a) = Q(\lambda_1^b,\,\lambda_2^b) > 0$.
                It implies that for almost every $r$, 
                \begin{equation}
                    \Pr [p(\vepsilon - \vdelta) \in (a,\,b) \,\big|\, \|\vepsilon\|_p = r] = 0
                \end{equation}
                where
                \begin{equation}
                    \begin{aligned}
                        a = \min\{ \lambda_1^a g(r) + \lambda_2^a h(r),\, \lambda_1^b g(r) + \lambda_2^b h(r)\}, \\
                        b = \max\{ \lambda_1^a g(r) + \lambda_2^a h(r),\, \lambda_1^b g(r) + \lambda_2^b h(r)\}.
                    \end{aligned}
                \end{equation}
                This violates the continuous assumption on both $\gP$ and $\gQ$.
                Therefore, point $r_0$ exists.
            \end{proof}
            
            With \Cref{lem:thm-5-mono}, we define auxiliary function $D: \sR_+ \to \sR$ such that
            \begin{equation}
                \begin{aligned}
                D(r) = & \Pr[p(\vepsilon-\vdelta) < \lambda_1^a p(\vepsilon) + \lambda_2^a q(\vepsilon) \,|\, \|\vepsilon\|_p = r] \\
                & \hspace{-1em} - \Pr[p(\vepsilon-\vdelta) < \lambda_1^b p(\vepsilon) + \lambda_2^b q(\vepsilon) \,|\, \|\vepsilon\|_p = r].
                \end{aligned}
            \end{equation}
            We let $S(r)$ be the surface area of $\ell_p$ sphere of radius $r$.
            Then the $P$ and $Q$ can be written in integral form:
            \begin{align}
                & P(\lambda_1,\,\lambda_2)
                = \int_0^{\infty} \Pr[p(\varepsilon - \vdelta) < \lambda_1p(\varepsilon) + \lambda_2p(\varepsilon) \,|\, \nonumber \\
                & \hspace{4em}  \|\varepsilon\|_p = r] \cdot g(r) S(r) \dif r, \\
                & Q(\lambda_1,\,\lambda_2)
                = \int_0^{\infty} \Pr[p(\varepsilon - \vdelta) < \lambda_1p(\varepsilon) + \lambda_2p(\varepsilon) \,|\, \nonumber \\
                & \hspace{4em} \|\varepsilon\|_p = r] \cdot h(r) S(r) \dif r.
            \end{align}
            Since $P(\lambda_1^a,\,\lambda_2^a) = P(\lambda_1^b,\,\lambda_2^b)$ and $Q(\lambda_1^a,\,\lambda_2^a) = Q(\lambda_1^b,\,\lambda_2^b)$ by our assumption,
            simple arrangement yields
            \begin{align}
                \int_0^{r_0} D(r)g(r)S(r) \dif r = \int_{r_0}^{\infty} (-D(r))g(r)S(r) \dif r \neq 0, \label{eq:thm-5-proof-1} \\
                \int_0^{r_0} D(r)h(r)S(r) \dif r = \int_{r_0}^{\infty} (-D(r))h(r)S(r) \dif r \neq 0. \label{eq:thm-5-proof-2}
            \end{align}
            As \Cref{lem:thm-5-mono} shows, $D(r)$ where $r\in [0,\,r_0]$ always has the same sign as $-D(r)$ where $r\in [r_0,\,+\infty]$, and the last inequality~($\neq 0$) is again due to the continuity of $p$ and $q$ and the fact that $P(\lambda_1^a,\,\lambda_2^a) > 0$ and $Q(\lambda_1^a,\,\lambda_2^a) > 0$.
            Now we can divide \Cref{eq:thm-5-proof-1} by \Cref{eq:thm-5-proof-2} and apply the Cauchy's mean value theorem, which yields
            \begin{equation}
                \dfrac{D(\xi_1)g(\xi_1)S(\xi_1)}{D(\xi_1)h(\xi_1)S(\xi_1)}
                =
                \dfrac{D(\xi_2)g(\xi_2)S(\xi_2)}{D(\xi_2)h(\xi_2)S(\xi_2)},
            \end{equation}
            where $\xi_1 \in (0,\,r_0)$ and $\xi_2 \in (r_0,\,+\infty)$.
            Apparently, it requires
            \begin{equation}
                \dfrac{g(\xi_1)}{h(\xi_1)} = \dfrac{g(\xi_2)}{h(\xi_2)}.
            \end{equation}
            However, $g/h$ is strictly monotonic.
            By contradiction, there is no distinct pair $(\lambda_1^a,\,\lambda_2^a)$ and $(\lambda_1^b,\,\lambda_2^b)$ satisfying the constraint of \Cref{eq:dual-2} simultaneously.
        \end{proof}
        
        \begin{remark}
            Suppose $\gP$ and $\gQ$ are $\ell_p$-radial extended Gaussian/Laplace distributions, i.e., their density functions $p(\vx)$ and $q(\vx)$ are 
            \begin{equation}
                \begin{aligned}
                    p(\vx) \propto \|\vx\|_p^{-k} \exp\left(-\|\vx\|_p/\sigma\right)^\alpha, \\
                    q(\vx) \propto \|\vx\|_p^{-k} \exp\left(-\|\vx\|_p/\beta\right)^\alpha
                \end{aligned}
            \end{equation}
            with $\alpha > 0$~(for Gaussian $\alpha=2$ and for Laplace $\alpha=1$).
            Note that this is a broader family than the family considered in \shortApproach shown in the main text.
            we have
            \begin{equation}
                \frac{g}{h} \propto
                \exp(r/\beta - r/\sigma)^\alpha
            \end{equation}
            that is strictly monotonic.
            Thus, \Cref{thm:uniqueness} is applicable for this large family of smoothing distributions that are commonly used in the literature~\citep{lecuyer2019certified,cohen2019certified,yang2020randomized,zhang2020black,zhai2019macer,jeong2020consistency}.
        \end{remark}

    \subsection{Discussion on Certification of Other \texorpdfstring{$\ell_p$}{Lp} Norms}
    
        \label{newadx:sub:cert-other-lp}
        
        \shortApproach can also be extended to provide a certified robust radius under $\ell_p$ norm other than $\ell_2$.
        
        Different from the case of $\ell_2$ certification, for other $\ell_p$ norm the challenge is to compute $P(\lambda_1,\lambda_2)$, $Q(\lambda_1,\lambda_2)$, and $R(\lambda_1,\lambda_2)$ as defined in \Cref{thm:concrete-equations}.
        For $\ell_2$ certification, as shown by \Cref{thm:concrete-equations,thm:concrete-equations,thm:numerical-integration-dif-var}, there exist closed-form expressions for these quantities that can be efficiently implemented with numerical integrations.
        However, for other $\ell_p$ norms, finding such closed-form expressions for $P$, $Q$, and $R$ is challenging.
        
        Luckily, we notice that $P$, $Q$, and $R$ are all probability-based definitions, as long as we can effectively sample from $\gP$ and $\gQ$ and effectively compute the density functions $p(\cdot)$ and $q(\cdot)$, we can estimate these function quantities from the empirical means of Monte-Carlo sampling.
        
        Compared to numerical integration based $\ell_2$ certification, Monte-Carlo sampling has \textit{sampling uncertainty} and \textit{efficiency} problems.
        Here we discuss these two problems in detail and how we can alleviate them.
        

        
        \paragraph{Sampling Uncertainty and Mitigations.}
        The empirical means for $P$ and $Q$ are stochastic, which breaks the nice properties of $P$ and $Q$~(shown in \Cref{subsec:joint-binary-search}) with respect to $(\lambda_1,\,\lambda_2)$ as the different queries to $P$ and $Q$ fluctuate around the actual value.
        Therefore, the joint binary search~(\Cref{alg:joint-binary-search}) may fail to return the correct $(\lambda_1,\,\lambda_2)$ pair.
        Thus, we propose a stabilization trick: the \emph{same} set of samples is used when querying $P$ and $Q$ during the joint binary search.
        With this same set of samples, it can be easily verified that the nice properties (\Cref{prop:convex,prop:mono}) still hold even if $P$ and $Q$ are empirical means.
        To guarantee the certification soundness in the context of probabilistic Monte-Carlo sampling,
        we introduce a different set of samples to test whether the solved $(\lambda_1,\,\lambda_2)$ indeed upper bounds the intersection point~(if the test fails we fall back to the classical Neyman-Pearson-based certification though it seldom happens).
        Since the test is also probabilistic, we need to accumulate this additional failing probability and use the lower bound of the confidence interval for soundness, which results in $1-2\alpha = 99.8\%$ certification instead of $1 - \alpha = 99.9\%$ as in classical randomized smoothing certification~\citep{cohen2019certified,yang2020randomized}.
        Note that the existence of additional confidence intervals caused by Monte-Carlo sampling makes \shortApproach based on Monte-Carlo sampling slightly looser than \shortApproach based on numerical integration.
        
        \paragraph{Efficiency Concern and Mitigations.}
        Traditionally
        we need to sample several~(denoted as $M$, in our implementation we set $M = 5\times 10^4$)  high-dimensional vectors for \emph{each} $P$ or $Q$ computation, which induces the efficiency concern.
        With the usage of the aforementioned stabilization trick, now we only sample one set of $M$ samples during the whole joint binary search instead of during each $P$ and $Q$ computation.
        Combining with the testing phase sampling, the whole algorithm needs to sample $2M$ vectors rather than $\gO(M \log^2 M)$ without the stabilization trick, so it greatly solves the efficiency concern.
        Moreover, we notice that for the samples $\{\epsilon_i\}_{i=1}^M$ we only need to care about its densities $\{p(\varepsilon_i)\}_{i=1}^M, \{q(\varepsilon_i)\}_{i=1}^M$ and $\{p(\varepsilon_i-\vdelta)\}_{i=1}^M$.
        Thus, we store only these three quantities instead of the whole $d$-dimensional vectors $\{\epsilon_i\}_{i=1}^M$, reducing the space complexity from $\gO(M\times d)$ to $\gO(M)$.
        
        These techniques significantly improve efficiency in practice.
        Although \shortApproach based on Monte-Carlo sampling is still slightly looser and slower than \shortApproach based on numerical integration,
        \shortApproach based on Monte-Carlo sampling makes certifying robustness under other $\ell_p$ norms feasible.
        In this work, we focus on the $\ell_2$ norm, because additive randomized smoothing is not optimal for other norms~(e.g., $\ell_1$~\citep{levine2021improved}) or the state-of-the-art certification can be directly translated from $\ell_2$ certification~(e.g., $\ell_\infty$~\citep{yang2020randomized} and semantic transformations~\citep{li2020transformation}).
        Moreover, to the best of our knowledge, standard $\ell_2$ certification is the most challenging setting where additive randomized smoothing achieves state-of-the-art and no other work can achieve visibly tighter robustness certification than standard Neyman-Pearson certification~\citep{yang2020randomized,levine2020tight,mohapatra2020higherorder}.

\section{Implementation and Optimizations}

    \label{newadx:impl}
    
    In this appendix, we discuss the implementation tricks and optimizations, along with our simple heuristic for selecting the hyperparameter $T$ in the additional smoothing distribution $\gQ = \gNgtrunc(k,T,\sigma)$.
    
    \subsection{Implementation Details}
        \label{newadxsec:impl-details}
    
        We implement \shortApproach in Python with about two thousand lines of code.
        The tool uses \texttt{PyTorch}\footnote{\url{http://pytorch.org/}} to query a given base classifier with Monte-Carlo sampling in order to derive the confidence intervals $[\underline{P_A},\,\overline{P_A}]$ and $[\underline{Q_A},\,\overline{Q_A}]$.
        Then, the tool builds the whole \shortApproach pipeline on \texttt{SciPy}\footnote{\url{https://scipy.org/scipylib/index.html}} and \texttt{NumPy}\footnote{\url{https://numpy.org/}}.
        Specifically, the numerical integration is implemented with \texttt{scipy.integrate.quad()} method.
        We exploit the full independence across the certification for different input instances and build the tool to be fully parallelizable on CPUs.
        By default, we utilize $10$ processes, and the number of processes can be dynamically adjusted.
        
        The tool is built in a flexible way that adding new smoothing distributions is not only theoretically straightforward but also easy in practice.
        In the future, we plan to extend the tool to 
        1)~provide GPU support; 
        2)~reuse existing certification results from previous instances with similar confidence intervals
        to achieve higher efficiency.
        We will also support more smoothing distributions.
        
        In the implementation, we widely use the logarithmic scale, since many quantities in the computation have varied scales.
        For example, since the input dimension $d$ is typically over $500$ on a real-world dataset, the density functions $p$ and $q$ decay very quickly along with the increase of noise magnitude.
        Another example is the input variable for $\ln(\cdot)$ and $W(\cdot)$ in \Cref{thm:concrete-equations}.
        These variables are exponential with respect to the input dimension $d$ so they could be very large or small.
        To mitigate this, we perform all computations with varied scales in logarithmic scale to improve the precision and floating-point soundness.
        For example, we implement a method $\texttt{lnlogadd}$ to compute $\log(\lambda_1 \exp(x_1) + \lambda_2 \exp(x_2))$ and apply method $\texttt{wlog}$ in \citep{yang2020randomized} to compute $W(\exp(x))$.
        We remark that in the binary search for dual variables~(see \Cref{subsec:joint-binary-search}), we also use the logarithmic scale for $\lambda_1$ and $\lambda_2$.
        
        
        The code, model weights, and all experiment data are publicly available at \url{https://github.com/llylly/DSRS}.

    \subsection{\texorpdfstring{$T$}{T} Heuristics}
    
        As briefly outlined in \Cref{subsec:exp-setup}, we apply a simple yet effective heuristic to determine the hyperparameter $T$ in additional smoothing distribution $\gQ = \gNgtrunc(k,T,\sigma)$.
        
        Concretely,
        we first sample the prediction probability from the original smoothing distribution $\gP$ and get the confidence lower bound $\underline{P_A}$ of $P_A = f^\gP(\vx_0)_{y_0}$.
        Then, we use the following empirical expression to determine $T$ from $\underline{P_A}$:
        \begin{equation}
            \label{eq:t-heuristic}
            T = \sigma \sqrt{\dfrac{2d}{d-2k}\GammaCDF^{-1}_{d/2-k}(p)},
        \end{equation}
        where
        \begin{equation}
            \label{eq:t-heuristic-2}
            p = \max\{-0.08 \ln(1 - \underline{P_A}) + 0.2, 0.5\}.
        \end{equation}
        It can be viewed as we first parameterize $T$ by $p$ and then find a simple heuristic to determine $p$ by $\underline{P_A}$.
        
        The $T$'s parameterization with $p$ is inspired by the probability mass under original smoothing distribution $\gP = \gNg(k,\sigma)$ if true-decision region is concentrated in a $\ell_2$-ball centered at $\vx_0$.
        Concretely,
        from $\gP$'s definition,
        \begin{equation}
            \Pr_{\vepsilon\sim\gP} [\|\vepsilon\|_2 \le T] = p.
        \end{equation}
        Then, we use a randomly sampled CIFAR-10 training set containing $1,000$ points with models trained using Consistency~\citep{jeong2020consistency} under $\sigma=0.50$ to sweep all $p \in \{0.1,0.2,\cdots,0.9\}$.
        We plot the minimum $p$ and maximum $p$ that gives the highest certified robust radius as a segment and find \Cref{eq:t-heuristic-2} fits the general tendency well as shown in \Cref{fig:t-heuristic}.
        Thus, we use this simple heuristic to determine $T$.
        Empirically, this simple heuristic generalizes well and is competitive with more complex methods as shown in \Cref{newadx:add-exp-results}.
        
        \begin{figure}[t]
            \includegraphics[width=\linewidth]{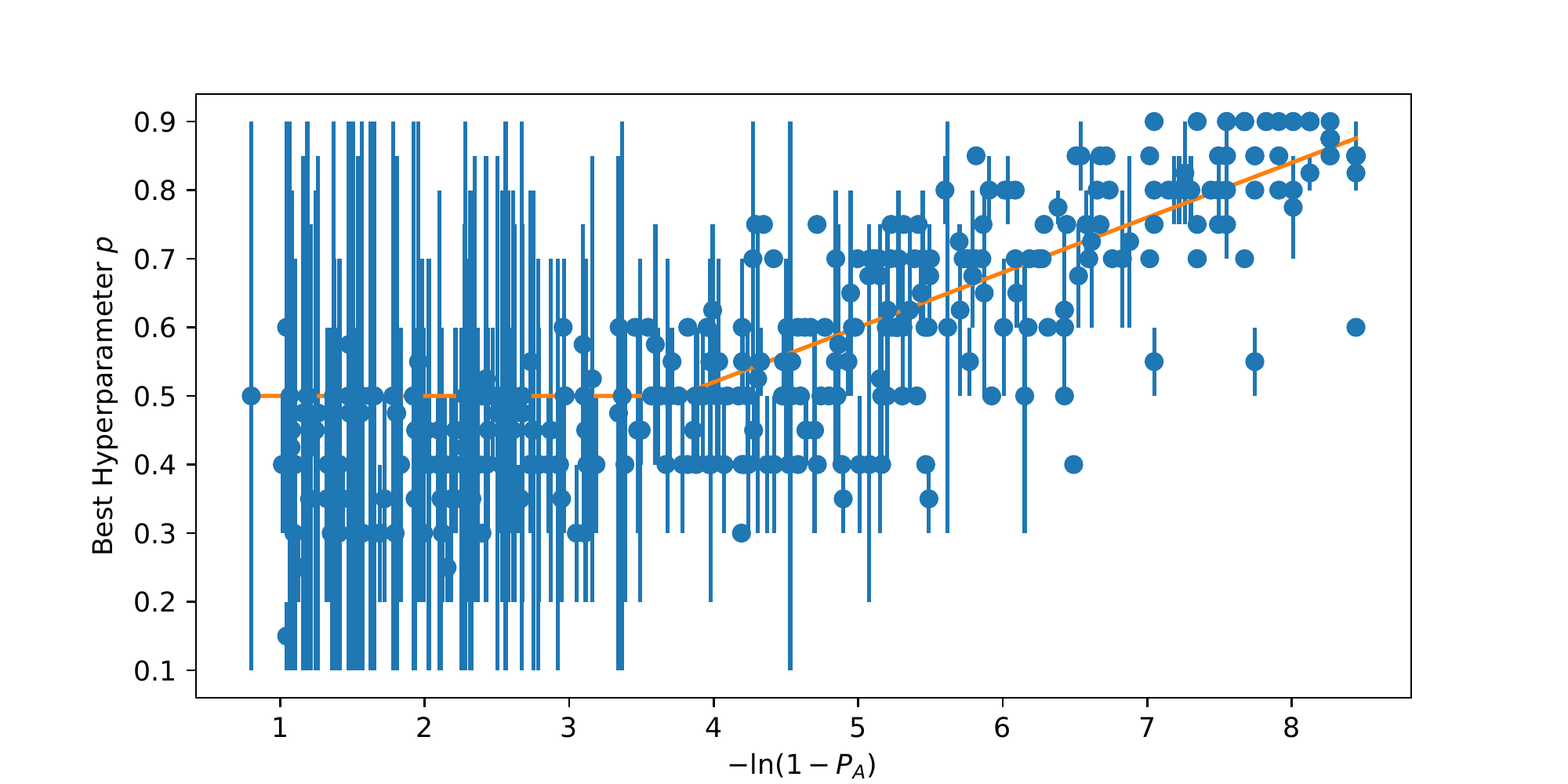}
            \label{fig:t-heuristic}
            \caption{Our $p$ heuristic~(\Cref{eq:t-heuristic-2}) shown as orange curve fits the generalization tendency of optimal $p$ ranges~(shown as blue segments) well.}
        \end{figure}
        
        Another heuristic that we have deployed is the fall-back strategy.
        When the empirical probability $\widetilde{P_A} = 1$, i.e., if for all the sampled $\vepsilon\sim\gP$, $F(\vx_0+\vepsilon)=y_0$, then we fall back to still using $\gP$ instead of using another distribution $\gQ$ for the second round of sampling.
        This strategy is inspired by a finding that, with the fixed sampling budget, if $\underline{P_A}$ is already very high, it is more efficient to use more samples to further increase the confidence interval of $\underline{P_A}$ rather than querying imprecise information under another distribution $\gQ$.
        Notice that such strategy does not break the high-confidence soundness of our certification, because $\Pr[\widetilde{P_A} = n/N \,|\, P_A \le t] \ge \Pr[\widetilde{P_A} = n/N \wedge \widetilde{P_A^{\text{half}}} = 1 \,|\, P_A \le t]$ for any $t < 1$ and Bernoulli distributed sampling~(which is our case), and we use Bernoulli confidence interval that corresponds to $\Pr[\widetilde{P_A} = n/N \,|\, P_A \le t]$.
        Due to the tight experiment time, we only deployed this strategy on ImageNet evaluation but not on MNIST and CIFAR-10 evaluations.

    \subsection{Training Strategy for Generalized Gaussian Smoothing}
    
        We train the base classifiers on each dataset using Gaussian augmentation training~\citep{cohen2019certified}, Consistency training~\citep{jeong2020consistency}, and SmoothMix training~\citep{jeong2021smoothmix}, which are typical training methods for randomized smoothing.
        We do not consider other training approaches such as \citep{zhai2019macer,salman2019provably,li2019certified,carmon2019unlabeled} because:
        (1)~Some training approaches either require additional data~\citep{carmon2019unlabeled,salman2019provably}, or relatively time-consuming~\citep{salman2019provably}, or are reported to be not as effective as later approaches~\citep{li2019certified};
        (2)~Selected training approaches are widely used or achieve state-of-the-art with high training efficiency.
        
        On MNIST, for all training methods, we use a convolutional neural network with 4 convolutional layers and 3 fully connected layers following \citet{cohen2019certified} as the base classifiers' architecture.
        On CIFAR-10, for all training methods, we use ResNet-110~\citep{he2016deep} as the base classifiers' architecture.
        These architecture settings
        follow the prior work on smoothed classifier training~\citep{cohen2019certified,salman2019provably,zhai2019macer}.
        
        On both MNIST and CIFAR-10, for all training methods, 
        we train for $150$ epochs.
        The learning rate is $0.01$ and is decayed by $0.1$ at the $50$th and $100$th epoch.
        For Consistency training, the hyperparameter $\lambda=5$ on MNIST and $\lambda=20$ on CIFAR-10.
        We use two noise vectors per training batch per instance.
        These are the best hyperparameters reported in \citep{jeong2020consistency}.
        The batch size is $256$ on both MNIST and CIFAR-10 following \citep{jeong2020consistency}.
        For SmoothMix training, we directly use the best hyperparameters from their open-source repository:
        \url{https://github.com/jh-jeong/smoothmix/blob/main/EXPERIMENTS.MD}.
        
        On ImageNet, we use ResNet-50~\citep{he2016deep} as the base classifiers' architecture and finetune from the open-source model trained by \citeauthor{cohen2019certified}~\citep{cohen2019certified} with Gaussian smoothing.
        We train for $10$ epochs due to the expensive training cost on ImageNet and we remark that better results can be achieved with a larger training time budget.
        The learning rate is $0.001$ and is decayed at the end of every epoch by $0.1$.
        For Consistency training, the hyperparameter $\lambda=5$ and we use two noise vectors per training batch per instance following the best hyperparameters reported in \citep{jeong2020consistency}.
        For SmoothMix training, since the open-source repository does not contain the best hyperparameters, we select the hyperparameters as suggested in the original paper~\citep{jeong2021smoothmix}.
        
        During training, the training samples are augmented by adding noises following \emph{training smoothing distribution}.
        In typical training approaches~\citep{cohen2019certified,carmon2019unlabeled,salman2019provably,zhai2019macer,jeong2020consistency}, the training smoothing distribution is set to be the same as the original smoothing distribution $\gP$ for constructing the smoothed classifier.
        However, for our generalized Gaussian $\gNg(k,\sigma)$ as $\gP$ with large $k$, we find this strategy gives a poor empirical performance.
        
        To better train the base classifier for our original smoothing distribution $\gP$ with large $k$, we introduce a warm-up stage on the training smoothing distribution.
        Suppose our original smoothing distribution $\gP$ for constructing the smoothed classifier is $\gNg(k_0,\sigma)$
        We let $e_0 = 100$ be the number of warm-up epochs on MNIST and CIFAR-10, or $e_0 = 10000$ be the number of warm-up steps on ImageNet.
        In the first $e_0$ epochs~(on MNIST and CIFAR-10) or steps~(on ImageNet), we use the training smoothing distribution with smaller $k$.
        Formally, in the $e$th epoch/step where $e \le e_0$, the training smoothing distribution $\gP' = \gNg(k,\sigma)$ where
        \begin{equation}
                k = \lceil k_0 - k_0^{1 - e / e_0} \rceil.
        \end{equation}
        For later epochs/steps, we use the original smoothing distribution $\gP$ itself as the training smoothing distribution.
        This strategy gradually increases the $k$ of training smoothing distribution throughout the training, so that the base classifier can be a better fit for the final desired distribution $\gP$.
        Using this strategy, the smoothed classifier constructed from our trained base classifier has similar certified robustness compared to standard Gaussian augmentation under classical Neyman-Pearson-based certification.



\section{Additional Experimental Results}
    
    \label{newadx:add-exp-results}
    In this appendix, we present additional experiment results and studies.

    \subsection{Empirical Study Setup of Concentration Property}
    
        \label{newadx:sub:study-concentration}
        
        
        In \Cref{fig:smooth-model-landscape-and-magnitude-density} in \Cref{newadx:illustration-concentration-fig}, we present our investigation of the decision regions of base classifiers.
        The investigation follows the following protocol:
        (1)~We choose the base classifier from \cite{salman2019provably} on ImageNet trained for $\sigma=0.5$ Gaussian smoothing as the subject.
        The reason is that this base classifier is one of the state-of-the-art certifiably robust classifiers on the large-scale ImageNet dataset and our code uses the same preprocessing parameters so it is easy to adopt their model.
        (2)~We pick every $500$-th image from the test set of the ImageNet dataset to form a subset of $100$ samples.
        (3)~We filter out the samples where the base classifier cannot classify correctly even without adding any noise, which results in $89$ remaining samples.
        (4)~For each of these $89$ samples, for each perturbation magnitude $r$, we uniformly sample $1000$ perturbation vectors from the hypersphere with $\ell_2$-radius $r$  and compute the empirical probability of true-prediction, where the step size of $r$ is $10$.
        
        \Cref{fig:smooth-model-landscape-and-magnitude-density} implies that for a vast majority of samples, the true-prediction samples are highly concentrated on an $\ell_2$ ball around the clean input since there exists apparent $\ell_2$ magnitude thresholds where the true-prediction probability is almost $1$ within the thresholds and almost $0$ beyond the thresholds.
        This implies that the concentration property may be achievable for real-world base classifiers in randomized smoothing.
        
        In \Cref{fig:landscape-for-diff-var}, we follow the same protocol but plot the landscape of base classifiers trained using generalized Gaussian distribution~(instead of standard Gaussian distribution as in \Cref{fig:smooth-model-landscape-and-magnitude-density}).
        By comparing \Cref{fig:landscape-for-diff-var} and \Cref{fig:smooth-model-landscape-and-magnitude-density}, we find that although base classifiers in \Cref{fig:smooth-model-landscape-and-magnitude-density} can achieve better certified robustness using Neyman-Pearson certification and generalized Gaussian smoothing~(compare \Cref{fig:real-data-sampling} and Neyman-Pearson rows in \Cref{tab:variant-compare}), they also sacrifice the concentration property, which can explain why \shortApproach improvements are much smaller on models in \Cref{sec:exp} compared with models in \Cref{fig:real-data-sampling}.
        Thus, as discussed in \Cref{sec:exp}, there may be a large space for exploring training approaches that favor \shortApproach certification by preserving the concentration property.

    \subsection{Effectiveness of \texorpdfstring{$T$}{T}-Heuristics and Attempts on Better Optimization Tricks}
    
        \label{newadx:sub:t-heuristic-attempt-better-opt}
    
    \paragraph{Better Optimization Tricks.} For obtaining a tighter certified radius, we should make the support of $\gNgtrunc(k,T,\sigma)$ more aligned with the decision region while keeping the $Q_A$ large enough. So except using a simple heuristic, we can also turn the search for an appropriate value of $T$ into an optimization problem. In order to make the optimization more stable, here we will construct the optimization based on $\Pcon$ that is more scale-invariant, and then transform it to get the final appropriate $T = \sigma \sqrt{2\GammaCDF_{d/2}^{-1}(\Pcon)}$.
    
    The final optimization objective is now built as $\Pcon = \argmin -Q_A+\frac{\lambda}{2}(\Pcon-P_A)^2$, where $\lambda$ is a hyperparameter that controls the relative weight of the two loss terms. The $Q_A$ here is estimated by sampling from the distribution $\gNgtrunc(k,T,\sigma)$ in which the $T$ is determined by $\Pcon$; however, the actual process of such sampling is implemented by rejecting the sampled noise whose norm is bigger than $T$. Therefore, there will be no gradient obtained for $\Pcon$ through the backward of the loss, namely, the optimized objective. So instead, we attempt to approximate the gradient comes from the first loss term $-Q_A$ with $G_{Q_A} = \E_{\vepsilon\sim\gQ}[\nabla_{||\vepsilon||_2}  \emph{CrossEntropyLoss}(f(\vx_0 + \vepsilon),y_0)]$. Then, we will only have the gradient information, and there is no explicit form of $Q_A$ anymore. The $P_A$ is estimated with the $\underline{P_A}$ which we have already obtained, so the final estimation of the gradient for $\Pcon$ is $G_{Q_A} + \lambda (\Pcon - \underline{P_A})$.
    
    \paragraph{Experiment Setting.} The $\Pcon$ is optimized for different input test images respectively, and the initialized $\Pcon$ is set to $0.7$. For each input, we will update $\Pcon$ $20$ steps on datasets MNIST and CIFAR10 while updating only $10$ steps on dataset ImageNet to reduce time complexity. For each step, we will sample $2,000$ times for estimating the term $G_{Q_A}$, and the learning rate is set to $2,000$. Besides, to avoid the $\Pcon$ being optimized too large or too small, we will clip the final optimized $\Pcon$ within $0.1$ and $0.9$. Since the optimization is a bit time-consuming, we only test it on CIFAR10 with $\sigma = 0.5$ and test it on MNIST and ImageNet with $\sigma = 1.0$. Different $\lambda$ is also tried for different datasets and different training methods for getting better performance. 
    
    \paragraph{Performance.} The final results are shown in~\Cref{tab:t-heuristic-attempt-better-opt}, and the certification approach based on the optimization tricks mentioned above is denoted as ``Opt'' in the table. As we can see, our simple $T$-Heuristics could still be competitive with the method based on complicated optimization but with a cheaper cost.
    
    \begin{table*}[!t]
    \centering
    \caption{\small $\ell_2$ certified robust accuracy w.r.t. different radii $r$ for different certification approaches.}
    
    
        \resizebox{\linewidth}{!}{
    \begin{tabular}{c|c|c|cccccccccccc}
        \toprule
        \multirow{2}{*}{Dataset} & Training & Certification & \multicolumn{12}{c}{Certified Accuracy under Radius $r$} \\
        & Method & Approach &  0.2  &  0.4  &  0.6  &  0.8  &  1.0  &  1.2  &  1.4  &  1.6  &  1.8  &  2.0  &  2.2  &  2.4  \\
        \hline\hline
        \multirow{6}{*}{MNIST} & \multirow{2}{*}{Gaussian Aug.} & Neyman-Pearson & 
         95.5 \% & 93.5 \% & 90.0 \% & 86.1 \% & 80.4 \% & 72.8 \% & 61.4 \% & 50.2 \% & 36.6 \% & 25.2 \% & 14.5 \% & 8.5 \% \\
        &  \multirow{2}{*}{$(\sigma=1.00)$} & DSRS($T$-heuristic) & 
         95.5 \% & 93.5 \% & 90.2 \% & 86.9 \% & 81.4 \% & 74.4 \% & 64.6 \% & 55.2 \% & 42.8 \% & 30.9 \% & 20.3 \% & 11.3 \% \\
        &  &  Opt $(\lambda=7e-05)$ & 
         95.5\% &	 93.5\% &	 90.0\% &	 86.9\% &	 81.7\% &	 74.9\% &	 65.6\% &	 55.8\% &	 43.8\% &	 30.5\% &	 19.1\% &	 10.2\% \\
        \cline{2-15}
        & \multirow{2}{*}{Consistency} & Neyman-Pearson & 
        94.5 \% & 92.6 \% & 89.3 \% & 85.9 \% & 80.7 \% & 74.4 \% & 65.9 \% & 56.9 \% & 44.1 \% & 34.4 \% & 23.3 \% & 12.8 \% \\
        & \multirow{2}{*}{$(\sigma=1.00)$} & DSRS($T$-heuristic) &
        94.5 \% & 92.8 \% & 89.3 \% & 86.3 \% & 81.4 \% & 76.1 \% & 68.3 \% & 59.5 \% & 50.7 \% & 39.8 \% & 30.7 \% & 20.0 \% \\
        &  &  Opt $(\lambda=6e-06)$ & 
         94.5\% &	 92.8\% &	 89.3\% &	 86.2\% &	 81.2\% &	 75.8\% &	 68.2\% &	 59.6\% &	 50.5\% &	 39.6\% &	 30.7\% &	 19.8\% \\
        \hline\hline
        
        \multirow{2}{*}{} &  &  & \multicolumn{12}{c}{Certified Accuracy under Radius $r$} \\
        &  &  &  0.1  &  0.2  &  0.3  &  0.4  &  0.5  &  0.6  &  0.7  &  0.8  &  0.9  &  1.0  &  1.1  &  1.2  \\
        \hline\hline
        
        \multirow{6}{*}{CIFAR-10} 
        &  \multirow{2}{*}{Gaussian Aug.} & Neyman-Pearson & 
        60.4 \% & 55.2 \% & 51.3 \% & 45.9 \% & 40.8 \% & 35.6 \% & 30.1 \% & 24.3 \% & 20.0 \% & 16.7 \% & 13.0 \% & 10.1 \% \\
        & \multirow{2}{*}{$(\sigma=0.50)$} & DSRS($T$-heuristic) & 
        60.6 \% & 55.5 \% & 51.5 \% & 46.8 \% & 42.1 \% & 37.3 \% & 32.5 \% & 27.4 \% & 22.8 \% & 19.3 \% & 16.0 \% & 12.7 \% \\
        &  &  Opt $(\lambda=3e-06)$&
         60.6\% &	 55.5\% &	 51.3\% &	 46.7\% &	 42.0\% &	 37.2\% &	 32.5\% &	 27.4\% &	 23.0\% &	 19.3\% &	 16.0\% &	 12.5\% \\
        \cline{2-15}
        & \multirow{2}{*}{Consistency} & Neyman-Pearson & 
         53.1 \% & 50.5 \% & 48.6 \% & 45.5 \% & 43.6 \% & 41.5 \% & 38.7 \% & 36.7 \% & 35.1 \% & 32.0 \% & 29.1 \% & 25.7 \% \\
        &  \multirow{2}{*}{$(\sigma=0.50)$} & DSRS($T$-heuristic) &
        53.1 \% & 50.7 \% & 48.7 \% & 45.7 \% & 44.0 \% & 41.8 \% & 39.6 \% & 37.8 \% & 36.0 \% & 34.4 \% & 31.3 \% & 28.6 \% \\
        & &  Opt $(\lambda=4e-06)$&
         53.1\% &	 50.7\% &	 48.7\% &	 45.7\% &	 44.0\% &	 41.8\% &	 39.5\% &	 37.8\% &	 36.0\% &	 34.4\% &	 31.4\% &	 28.4\% \\
        \hline\hline
        
        \multirow{6}{*}{ImageNet} 
        &  \multirow{2}{*}{Gaussian Aug.} & Neyman-Pearson & 
        57.5 \% & 55.1 \% & 52.2 \% & 49.7 \% & 47.0 \% & 43.9 \% & 40.8 \% & 38.1 \% & 35.0 \% & 33.2 \% & 29.6 \% & 25.3 \% \\
        & \multirow{2}{*}{$(\sigma=1.00)$} & DSRS($T$-heuristic) &
        57.7 \% & 55.6 \% & 52.7 \% & 51.0 \% & 48.4 \% & 45.5 \% & 43.1 \% & 40.2 \% & 37.9 \% & 35.3 \% & 32.8 \% & 30.5 \% \\
        &  &  Opt $(\lambda=1e-05)$&
         57.7\% &	 55.5\% &	 52.4\% &	 50.5\% &	 48.2\% &	 45.0\% &	 42.9\% &	 40.0\% &	 38.0\% &	 35.0\% &	 32.7\% &	 29.9\% \\
        \cline{2-15}
        & \multirow{2}{*}{Consistency} & Neyman-Pearson & 
        55.9 \% & 54.4 \% & 53.0 \% & 51.2 \% & 48.2 \% & 46.2 \% & 44.2 \% & 41.7 \% & 39.1 \% & 36.4 \% & 34.4 \% & 32.1 \%\\
        &  \multirow{2}{*}{$(\sigma=1.00)$} & DSRS($T$-heuristic) &
        56.0 \% & 54.6 \% & 53.1 \% & 51.8 \% & 49.9 \% & 47.4 \% & 45.7 \% & 44.2 \% & 41.7 \% & 39.3 \% & 37.8 \% & 35.8 \% \\
        & &  Opt $(\lambda=1e-05)$&
         56.0\% &	 54.6\% &	 53.1\% &	 51.8\% &	 49.7\% &	 47.4\% &	 45.3\% &	 44.0\% &	 41.6\% &	 39.3\% &	 37.8\% &	 35.9\% \\
        \bottomrule
    \end{tabular}
    }
    \label{tab:t-heuristic-attempt-better-opt}
\end{table*}
    
    
    \subsection{Full Curves and Separated Tables}
        
        \label{newadx:sub:curve-sepa-tables}
        
        \subsubsection{Separate Tables by Smoothing Variance}

            \begin{table*}[htbp]
                \centering
                \caption{\small MNIST: Certified robust accuracy for models smoothed with different variance $\sigma$ certified with different certification approaches.}
                \resizebox{\linewidth}{!}{
                \begin{tabular}{c|c|c|cccccccccccc}
                    \toprule
                    \multirow{2}{*}{Variance} & Training & Certification & \multicolumn{12}{c}{Certified Accuracy under Radius $r$} \\
                    & Method & Approach & 0.25 & 0.50 & 0.75 & 1.00 & 1.25 & 1.50 & 1.75 & 2.00 & 2.25 & 2.50 & 2.75 & 3.00 \\
                    \hline\hline
                    \multirow{6}{*}{0.25} & Gaussian Aug. & Neyman-Pearson & 97.9\%	& 96.4\%	& 92.1\% \\
                    & \citep{cohen2019certified} & \textbf{DSRS} & 97.9\%	& 96.6\%	& 92.7\% \\
                    \cline{2-15} 
                    & Consistency & Neyman-Pearson & 98.4\%	& 97.5\%	& 94.4\%\\
                    & \citep{jeong2020consistency} & \textbf{DSRS} & 98.4\%	& 97.5\%	& 95.4\%\\
                    \cline{2-15} 
                    & SmoothMix & Neyman-Pearson & 98.6\%	& 97.6\%	& 96.5\%\\
                    & \citep{jeong2021smoothmix} & \textbf{DSRS} & 98.6\%	& 97.7\%	& 96.8\%\\
                    \hline\hline
                    \multirow{6}{*}{0.50} & Gaussian Aug. & Neyman-Pearson & 97.8\%	& 96.9\%	& 94.6\%	& 88.4\%	& 78.7\%	& 52.6\% \\
                    & \citep{cohen2019certified} & \textbf{DSRS} & 97.8\%	& 97.0\%	& 95.0\%	& 89.8\%	& 83.4\%	& 59.1\% \\
                    \cline{2-15} 
                    & Consistency & Neyman-Pearson & 98.4\%	& 97.3\%	& 96.0\%	& 92.3\%	& 83.8\%	& 67.5\% \\
                    & \citep{jeong2020consistency} & \textbf{DSRS} & 98.4\%	& 97.3\%	& 96.0\%	& 93.5\%	& 87.1\%	& 71.8\% \\
                    \cline{2-15} 
                    & SmoothMix & Neyman-Pearson & 98.2\%	& 97.1\%	& 95.4\%	& 91.9\%	& 85.1\%	& 73.0\% \\
                    & \citep{jeong2021smoothmix} & \textbf{DSRS} & 98.1\%	& 97.1\%	& 95.9\%	& 93.4\%	& 87.5\%	& 76.6\% \\
                    \hline\hline
                    \multirow{6}{*}{1.00} & Gaussian Aug. & Neyman-Pearson & 95.2\%	& 91.9\%	& 87.7\%	& 80.6\%	& 71.2\%	& 57.6\%	& 41.0\%	& 25.5\%	& 13.6\%	& 6.2\%	& 2.1\%	& 0.9\% \\
                    & \citep{cohen2019certified} & \textbf{DSRS} & 95.1\%	& 91.8\%	& 88.2\%	& 81.5\%	& 73.6\%	& 61.6\%	& 48.4\%	& 34.1\%	& 21.0\%	& 10.6\%	& 4.4\%	& 1.2\%	 \\
                    \cline{2-15} 
                    & Consistency & Neyman-Pearson & 93.9\%	& 90.9\%	& 86.4\%	& 80.8\%	& 73.0\%	& 61.1\%	& 49.1\%	& 35.6\%	& 21.7\%	& 10.4\%	& 4.1\%	& 1.9\% \\
                    & \citep{jeong2020consistency} & \textbf{DSRS} & 93.9\%	& 91.1\%	& 86.9\%	& 81.7\%	& 75.2\%	& 65.6\%	& 55.8\%	& 41.9\%	& 31.4\%	& 17.8\%	& 8.6\%	& 2.8\% \\
                    \cline{2-15} 
                    & SmoothMix & Neyman-Pearson & 92.0\%	& 88.9\%	& 84.4\%	& 78.6\%	& 69.8\%	& 60.7\%	& 49.9\%	& 40.2\%	& 31.5\%	& 22.2\%	& 12.2\%	& 4.9\% \\
                    & \citep{jeong2021smoothmix} & \textbf{DSRS} & 92.2\%	& 89.0\%	& 84.8\%	& 79.7\%	& 72.0\%	& 63.9\%	& 54.4\%	& 46.2\%	& 37.6\%	& 29.2\%	& 18.5\%	& 7.2\% \\
                    \bottomrule
                \end{tabular}
                }
                \label{tab:mnist-separate-table}
            \end{table*}

            \begin{table*}[htbp]
                \centering
                \caption{\small CIFAR-10: Certified robust accuracy for models smoothed with different variance $\sigma$ certified with different certification approaches.}
                \resizebox{\linewidth}{!}{
                \begin{tabular}{c|c|c|cccccccccccc}
                    \toprule
                    \multirow{2}{*}{Variance} & Training & Certification & \multicolumn{12}{c}{Certified Accuracy under Radius $r$} \\
                    & Method & Approach & 0.25 & 0.50 & 0.75 & 1.00 & 1.25 & 1.50 & 1.75 & 2.00 & 2.25 & 2.50 & 2.75 & 3.00 \\
                    \hline\hline
                    \multirow{6}{*}{0.25} & Gaussian Aug. & Neyman-Pearson & 56.1\%	& 35.7\%	& 13.4\%\\
                    & \citep{cohen2019certified} & \textbf{DSRS} & 57.4\%	& 39.4\%	& 17.3\%	 \\
                    \cline{2-15} 
                    & Consistency & Neyman-Pearson & 61.8\%	& 50.9\%	& 34.7\% \\
                    & \citep{jeong2020consistency} & \textbf{DSRS} & 62.5\%	& 52.5\%	& 37.8\% \\
                    \cline{2-15} 
                    & SmoothMix & Neyman-Pearson & 63.9\%	& 53.3\%	& 38.4\%\\
                    & \citep{jeong2021smoothmix} & \textbf{DSRS} & 64.7\%	& 55.5\%	& 41.1\%\\
                    \hline\hline
                    \multirow{6}{*}{0.50} & Gaussian Aug. & Neyman-Pearson & 53.7\%	& 41.3\%	& 27.7\%	& 17.1\%	& 9.1\%	& 2.8\% \\
                    & \citep{cohen2019certified} & \textbf{DSRS} & 54.1\%	& 42.7\%	& 30.6\%	& 20.3\%	& 12.6\%	& 4.0\% \\
                    \cline{2-15} 
                    & Consistency & Neyman-Pearson & 49.2\%	& 43.9\%	& 38.0\%	& 32.3\%	& 23.8\%	& 18.1\% \\
                    & \citep{jeong2020consistency} & \textbf{DSRS} & 49.6\%	& 44.1\%	& 38.7\%	& 35.2\%	& 28.1\%	& 19.7\% \\
                    \cline{2-15} 
                    & SmoothMix & Neyman-Pearson & 53.2\%	& 47.6\%	& 40.2\%	& 34.2\%	& 26.7\%	& 19.6\% \\
                    & \citep{jeong2021smoothmix} & \textbf{DSRS} & 53.3\%	& 48.5\%	& 42.1\%	& 35.9\%	& 29.4\%	& 21.7\% \\
                    \hline\hline
                    \multirow{6}{*}{1.00} & Gaussian Aug. & Neyman-Pearson & 40.2\%	& 32.6\%	& 24.7\%	& 18.9\%	& 14.9\%	& 10.2\%	& 7.5\%	& 4.1\%	& 2.0\%	& 0.7\%	& 0.1\%	& 0.1\%	 \\
                    & \citep{cohen2019certified} & \textbf{DSRS} & 40.3\%	& 33.1\%	& 25.9\%	& 20.6\%	& 16.1\%	& 12.5\%	& 8.4\%	& 6.4\%	& 3.5\%	& 1.8\%	& 0.7\%	& 0.1\%	\\
                    \cline{2-15} 
                    & Consistency & Neyman-Pearson & 37.2\%	& 32.6\%	& 29.6\%	& 25.9\%	& 22.5\%	& 19.0\%	& 16.4\%	& 13.8\%	& 11.2\%	& 9.0\%	& 7.1\%	& 5.1\% \\
                    & \citep{jeong2020consistency} & \textbf{DSRS} & 37.1\%	& 32.5\%	& 29.8\%	& 27.1\%	& 23.5\%	& 20.9\%	& 17.6\%	& 15.3\%	& 13.1\%	& 10.9\%	& 8.9\%	& 6.5\% \\
                    \cline{2-15} 
                    & SmoothMix & Neyman-Pearson & 43.2\%	& 39.5\%	& 33.9\%	& 29.1\%	& 24.0\%	& 20.4\%	& 17.0\%	& 13.9\%	& 10.3\%	& 7.8\%	& 4.9\%	& 2.3\% \\
                    & \citep{jeong2021smoothmix} & \textbf{DSRS} & 43.2\%	& 39.7\%	& 34.9\%	& 30.0\%	& 25.8\%	& 22.1\%	& 18.7\%	& 16.1\%	& 13.2\%	& 10.2\%	& 7.1\%	& 3.9\% \\
                    \bottomrule
                \end{tabular}
                }
                \label{tab:cifar-separate-table}
            \end{table*}

            \begin{table*}[htbp]
                \centering
                \caption{\small ImageNet: Certified robust accuracy for models smoothed with different variance $\sigma$ certified with different certification approaches.}
                \resizebox{\linewidth}{!}{
                \begin{tabular}{c|c|c|cccccccccccc}
                    \toprule
                    \multirow{2}{*}{Variance} & Training & Certification & \multicolumn{12}{c}{Certified Accuracy under Radius $r$} \\
                    & Method & Approach & 0.25 & 0.50 & 0.75 & 1.00 & 1.25 & 1.50 & 1.75 & 2.00 & 2.25 & 2.50 & 2.75 & 3.00 \\
                    \hline\hline
                    \multirow{6}{*}{0.25} & Gaussian Aug. & Neyman-Pearson & 57.1\%	& 41.6\%	& 17.4\% \\
                    & \citep{cohen2019certified} & \textbf{DSRS} & 58.4\%	& 47.9\%	& 24.4\% \\
                    \cline{2-15} 
                    & Consistency & Neyman-Pearson & 59.8\%	& 49.8\%	& 36.9\% \\
                    & \citep{jeong2020consistency} & \textbf{DSRS} & 60.4\%	& 52.4\%	& 40.4\% \\
                    \cline{2-15} 
                    & SmoothMix & Neyman-Pearson & 46.7\%	& 38.2\%	& 28.2\% \\
                    & \citep{jeong2021smoothmix} & \textbf{DSRS} & 47.4\%	& 40.0\%	& 29.8\% \\
                    \hline\hline
                    \multirow{6}{*}{0.50} & Gaussian Aug. & Neyman-Pearson & 53.3\%	& 47.0\%	& 39.3\%	& 33.2\%	& 24.5\%	& 17.0\% \\
                    & \citep{cohen2019certified} & \textbf{DSRS} & 54.1\%	& 48.4\%	& 41.4\%	& 35.3\%	& 28.8\%	& 19.1\% \\
                    \cline{2-15} 
                    & Consistency & Neyman-Pearson & 53.6\%	& 48.3\%	& 43.3\%	& 36.8\%	& 31.4\%	& 24.5\% \\
                    & \citep{jeong2020consistency} & \textbf{DSRS} & 53.7\%	& 49.9\%	& 44.7\%	& 39.3\%	& 34.8\%	& 27.4\% \\
                    \cline{2-15} 
                    & SmoothMix & Neyman-Pearson & 38.7\%	& 33.5\%	& 28.8\%	& 24.6\%	& 18.1\%	& 13.5\% \\
                    & \citep{jeong2021smoothmix} & \textbf{DSRS} & 39.1\%	& 34.9\%	& 30.3\%	& 26.8\%	& 21.6\%	& 15.6\% \\
                    \hline\hline
                    \multirow{6}{*}{1.00} & Gaussian Aug. & Neyman-Pearson & 42.5\%	& 37.2\%	& 33.0\%	& 29.2\%	& 24.8\%	& 21.4\%	& 17.6\%	& 13.7\%	& 10.2\%	& 7.8\%	& 5.7\%	& 3.6\% \\
                    & \citep{cohen2019certified} & \textbf{DSRS} & 42.5\%	& 38.1\%	& 34.4\%	& 30.2\%	& 27.0\%	& 23.3\%	& 21.3\%	& 18.7\%	& 14.2\%	& 11.0\%	& 9.0\%	& 5.7\% \\
                    \cline{2-15} 
                    & Consistency & Neyman-Pearson & 40.0\%	& 38.3\%	& 34.2\%	& 31.8\%	& 28.7\%	& 25.6\%	& 22.1\%	& 19.1\%	& 16.1\%	& 14.0\%	& 10.6\%	& 8.5\%\\
                    & \citep{jeong2020consistency} & \textbf{DSRS} & 40.2\%	& 38.5\%	& 35.4\%	& 32.6\%	& 30.7\%	& 28.1\%	& 25.4\%	& 22.6\%	& 19.6\%	& 17.4\%	& 14.1\%	& 10.4\%\\
                    \cline{2-15} 
                    & SmoothMix & Neyman-Pearson & 29.8\%	& 25.6\%	& 21.8\%	& 19.2\%	& 17.0\%	& 14.2\%	& 11.8\%	& 10.1\%	& 8.9\%	& 7.2\%	& 6.0\%	& 4.6\%\\
                    & \citep{jeong2021smoothmix} & \textbf{DSRS} & 29.7\%	& 26.2\%	& 23.0\%	& 20.6\%	& 18.0\%	& 15.7\%	& 14.0\%	& 12.1\%	& 9.9\%	& 8.4\%	& 7.2\%	& 5.3\%\\
                    \bottomrule
                \end{tabular}
                }
                \label{tab:imagenet-separate-table}
            \end{table*}
            
            Due to the page limit, in the main text~(\Cref{sec:exp}, \Cref{tab:ell-2}), we aggregate the certified robust accuracy across models with different smoothing variance $\sigma\in\{0.25, 0.50, 1.00\}$.
            To show the full landscape, we present the certified robust accuracy for each model trained with each variance.
            The evaluation protocol is the same as the one in the main text, and the tables for MNIST, CIFAR-10, and ImageNet models are \Cref{tab:mnist-separate-table}, \Cref{tab:cifar-separate-table}, and \Cref{tab:imagenet-separate-table} respectively.
            We observe that \shortApproach outperforms standard Neyman-Pearson certification for a wide range of radii.
        
        \subsubsection{Curves}
        
            Following the convention~\citep{cohen2019certified,salman2019provably}, we plot the certified robust accuracy - radius curve in \Cref{fig:curves}.
            
            The curves correspond to the certified robust accuracy data in \Cref{tab:ell-2}~(in \Cref{sec:exp}), i.e., the certified robust accuracy under each radius $r$ is the maximum certified robust accuracy among models trained with variance $\sigma\in \{0.25,0.50,1.00\}$.
            We observe that among all medium to large radii~(including those not shown in \Cref{tab:ell-2}), \shortApproach provides higher certified robust accuracy than Neyman-Pearson certification.
            The margin of \shortApproach is relatively small on CIFAR-10 but is apparent on MNIST and ImageNet.
            Especially, the margins on ImageNet reflect that \shortApproach is particularly effective on large datasets.

            \begin{figure*}[!t]
                \centering
                \subfigure[]{\includegraphics[width=0.33\textwidth]{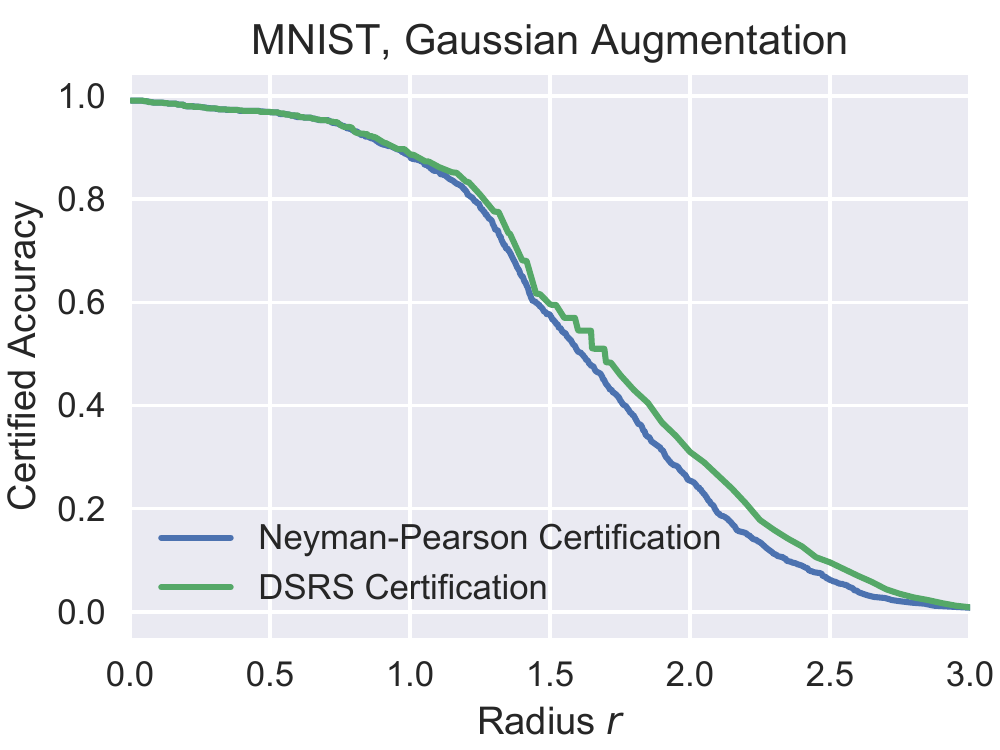}}
                \subfigure[]{\includegraphics[width=0.33\textwidth]{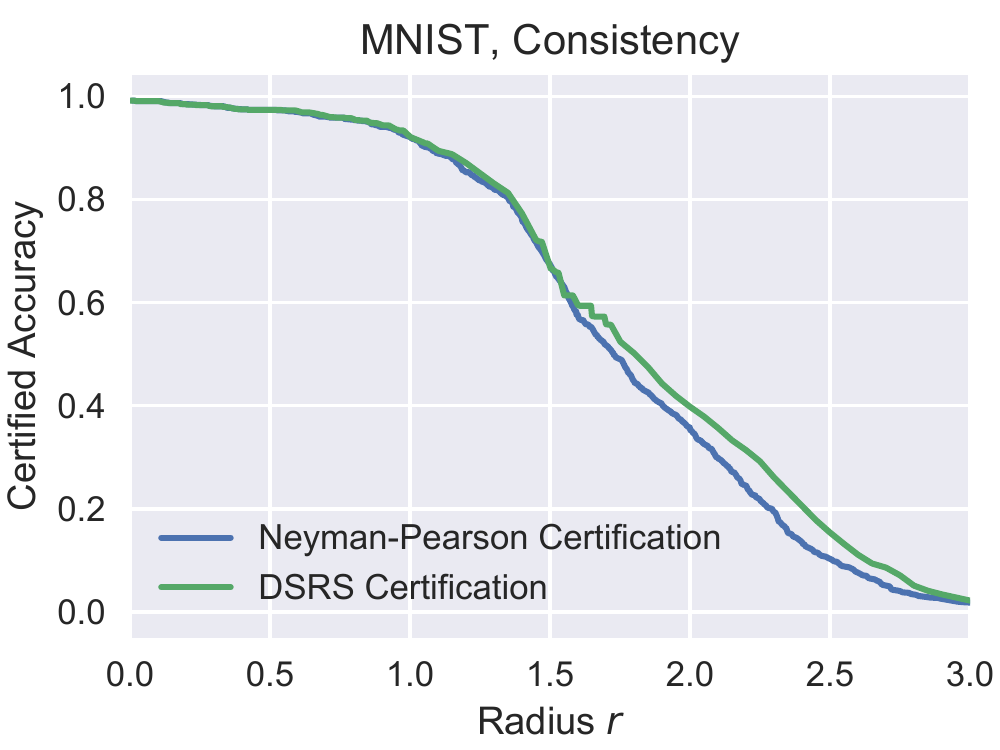}}
                \subfigure[]{\includegraphics[width=0.33\textwidth]{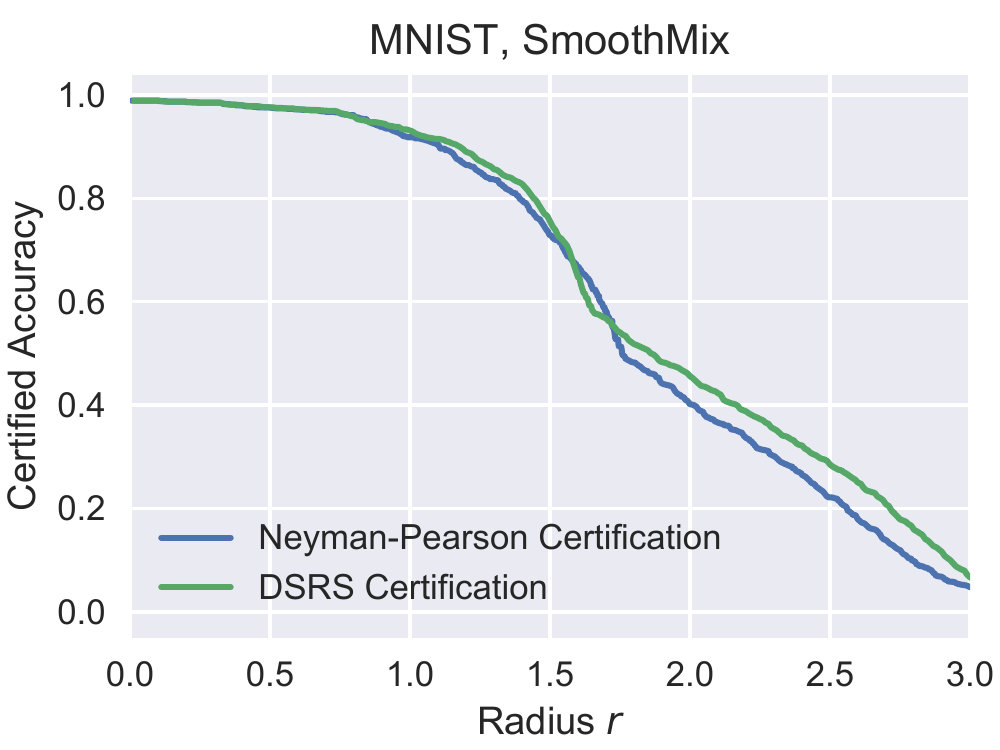}}
                \\
                
                
                \subfigure[]{\includegraphics[width=0.33\textwidth]{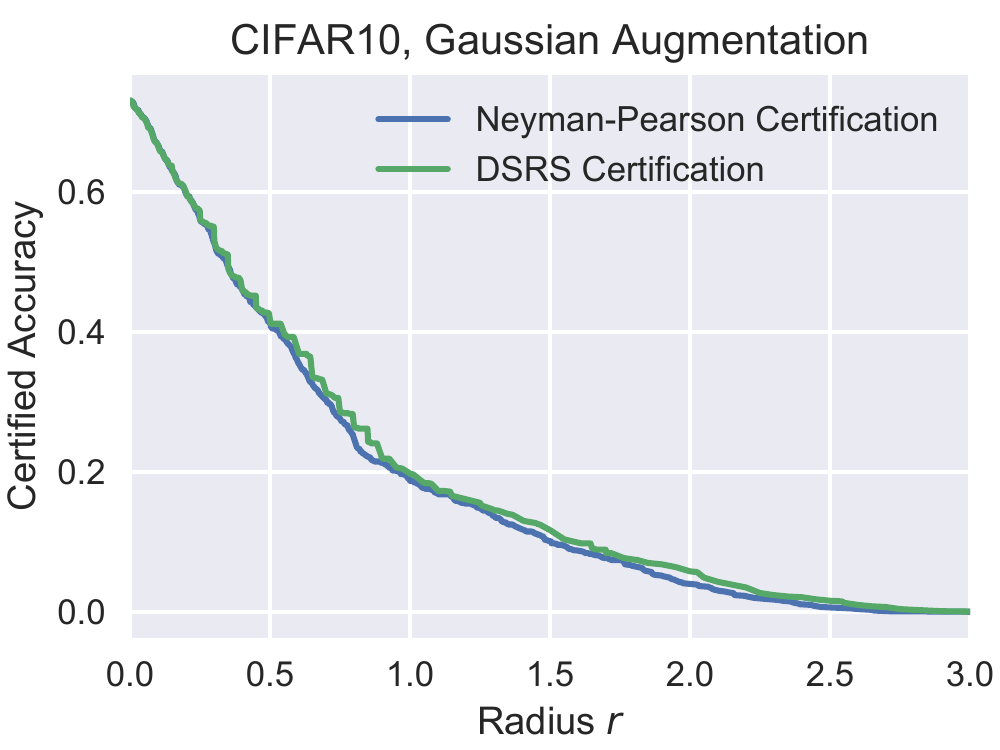}}
                \subfigure[]{\includegraphics[width=0.33\textwidth]{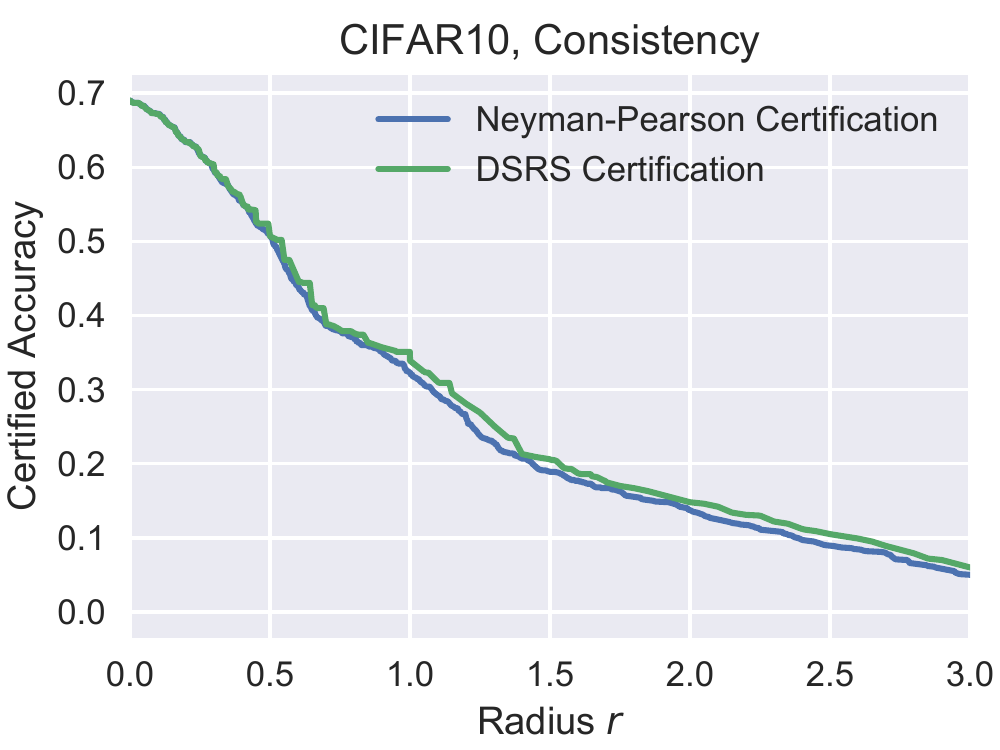}}
                \subfigure[]{\includegraphics[width=0.33\textwidth]{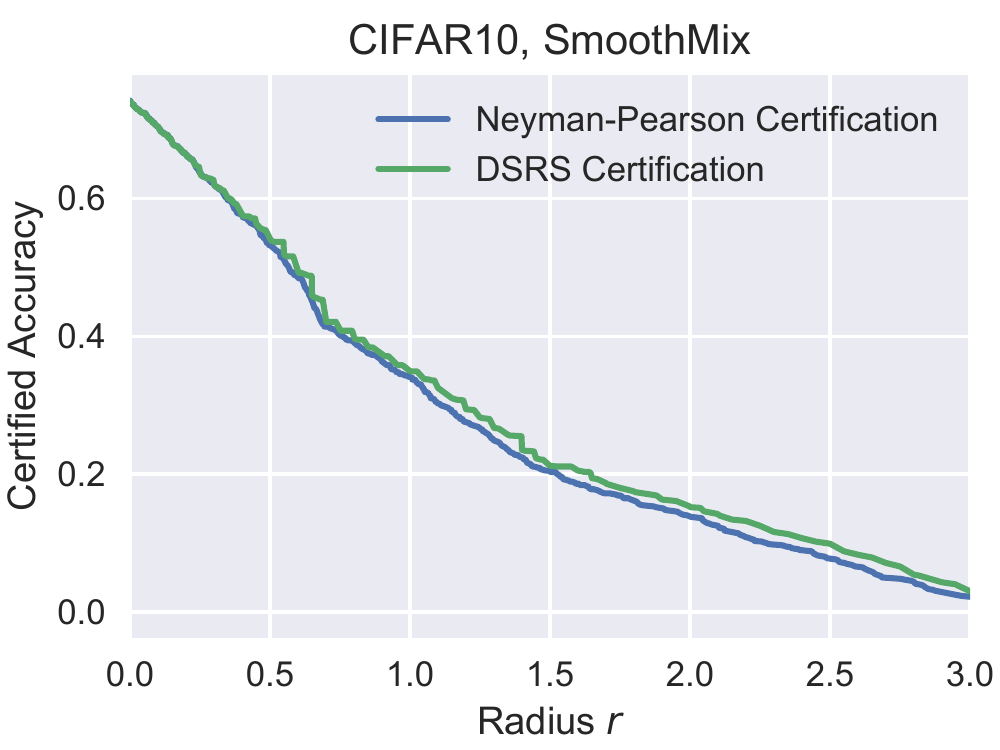}}
                \\
                
                \subfigure[]{\includegraphics[width=0.33\textwidth]{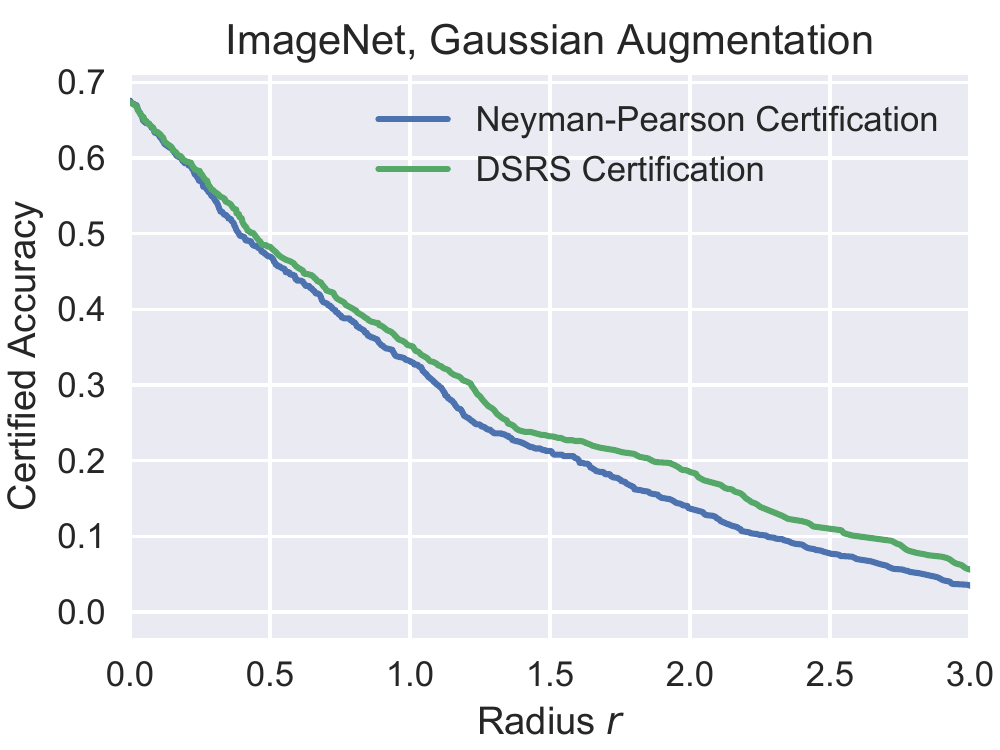}}
                \subfigure[]{\includegraphics[width=0.33\textwidth]{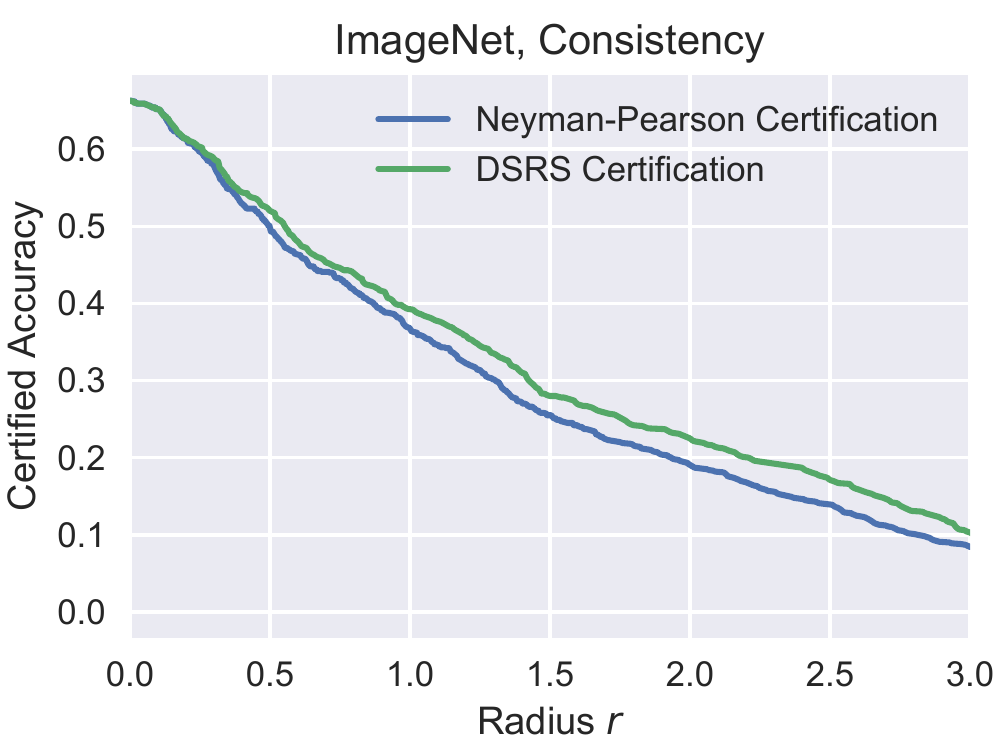}}
                \subfigure[]{\includegraphics[width=0.33\textwidth]{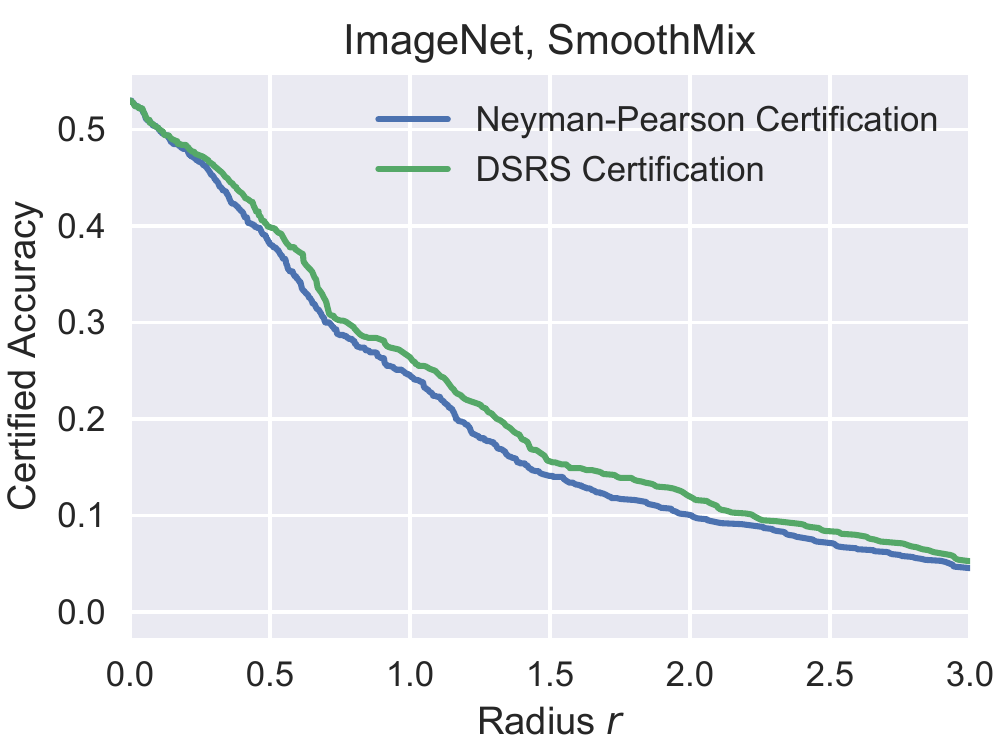}}
                \\
                
                \caption{Certified robust accuracy - radius curve corresponding to \Cref{tab:ell-2}.
                }
                \label{fig:curves}
            \end{figure*}
        
        \subsubsection{ACR Results}
            \label{adxsubsec:acr}
        
            In the literature, another common metric of certified robustness is ACR~(average certified radius)~\cite{zhai2019macer,jeong2020consistency,jeong2021smoothmix}.
            In \Cref{tab:acr}, we report the ACR comparison between Neyman-Pearson-based certification and our \shortApproach certification. 
            Across the three smoothing variance choices $\sigma\in\{0.25,0.50,1.00\}$, we find $\sigma=1.00$ yields the highest ACR, so we only report the ACR for models smoothed with $\sigma=1.00$.
            As we can see, in all cases, \shortApproach significantly improves over Neyman-Pearson-based certification in terms of ACR.

            \begin{table}[!t]
                \caption{Average certified radius~(ACR) statistics. The smoothing variance hyperparameter $\sigma=1.00$. The evaluation protocol is the same as that in the main text.}
                \label{tab:acr}
                \centering
                \resizebox{\linewidth}{!}{
                \begin{tabular}{cc|ccc}
                    \toprule
                    Training Method & Certification & MNIST & CIFAR-10 & ImageNet \\
                    \hline
                     \multirow{3}{*}{\shortstack{Gaussian\\ Augmentation}} & Neyman-Pearson & 1.550 & 0.447 & 0.677 \\
                     & \shortApproach & 1.629 & 0.469 & 0.750 \\
                     & Relative Improvement & \bf +5.10\% & \bf +4.92\% & \bf +10.78\% \\
                    \hline
                     \multirow{3}{*}{Consistency} & Neyman-Pearson & 1.645 & 0.636 & 0.796 \\
                     & \shortApproach & 1.730 & 0.659 & 0.862 \\
                     & Relative Improvement & \bf +5.17\% & \bf +3.62\% & \bf +8.29\% \\
                     \hline
                     \multirow{3}{*}{SmoothMix} & Neyman-Pearson & 1.716 & 0.676 & 0.490 \\
                     & \shortApproach & 1.806 & 0.712 & 0.525 \\
                     & Relative Improvement & \bf +5.24\% & \bf +5.33\% & \bf +7.14\% \\
                    \bottomrule
                \end{tabular}
                }
            \end{table}
    
    \subsection{Using Distribution with Different Variance as \texorpdfstring{$\gQ$}{Q}}
    
        \label{newadx:sub:diff-var-dsrs}
        
        
        We take the models trained with Gaussian augmentation~\citep{cohen2019certified} and variance $\sigma=1.00$ as examples.
        We use ``DSRS-trunc'' to represent \shortApproach using truncated generalized Gaussian as $\gQ$, and ``DSRS-var'' to represent \shortApproach using generalized Gaussian with different variance as $\gQ$, and compare their robustness certification~(i.e., certified robust accuracy) in \Cref{tab:variant-compare}.
        From the table, we find that on MNIST and CIFAR-10, DSRS-var is better than DSRS-trunc, whereas on ImageNet, DSRS-trunc is slightly better than DSRS-var.
        Both DSRS-trunc and DSRS-var are significantly better than Neyman-Pearson-based certification.
        
        To investigate the reason, we follow the protocols for studying the concentration property in \Cref{newadx:sub:study-concentration} to plot the landscape of models on MNIST, CIFAR-10, and ImageNet, as shown in \Cref{fig:landscape-for-diff-var}.
        From the figure, we find that the curves on ImageNet are generally steeper, which corresponds to that the concentration property is better satisfied on ImageNet.
        Therefore, we conjecture that when the concentration property~(see \Cref{def:concentration}) is better satisfied, \shortApproach with truncated Gaussian as $\gQ$ is better than Gaussian with different variance as $\gQ$.
        
        \begin{figure*}
            \centering
            \subfigure[MNIST model.]{
                \includegraphics[width=0.31\textwidth]{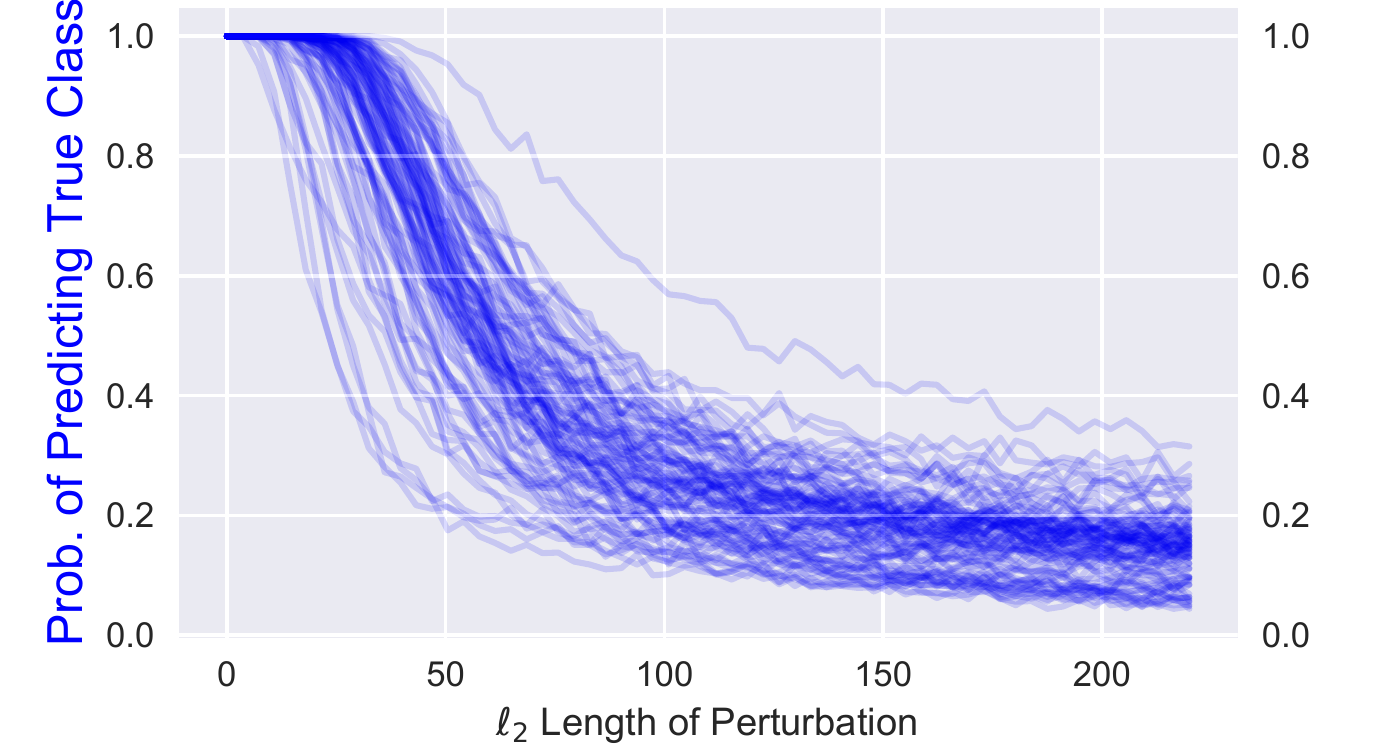}
            }
            \subfigure[CIFAR-10 model.]{
                \includegraphics[width=0.31\textwidth]{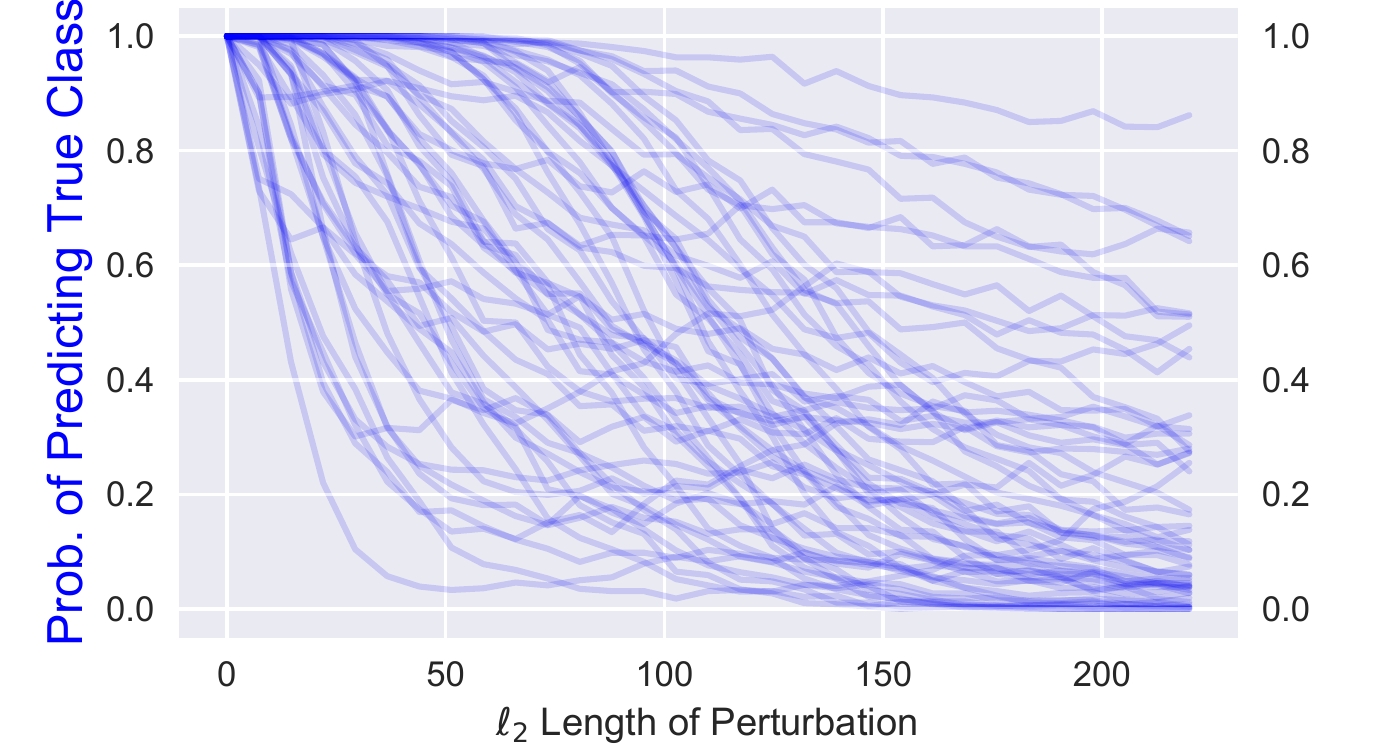}
            }
            \subfigure[ImageNet model.]{
                \includegraphics[width=0.31\textwidth]{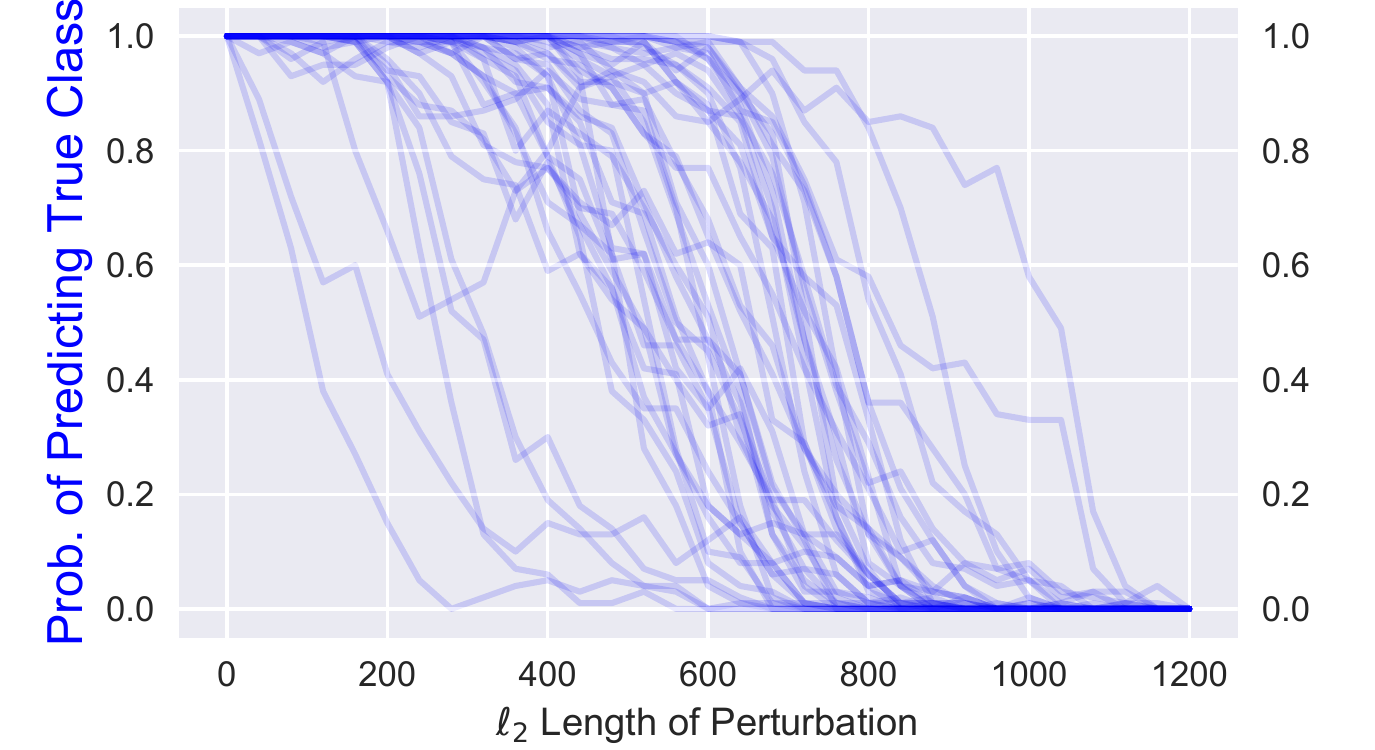}
            }
            \caption{Probability of true-prediction w.r.t. $\ell_2$ length of perturbations for base classifiers from Gaussian augmentation training studied in \Cref{newadx:sub:diff-var-dsrs}. Figures are plotted following the same protocol as in \Cref{newadx:sub:study-concentration}.}
            \label{fig:landscape-for-diff-var}
        \end{figure*}
        
        \begin{table*}[t]
            \centering
            \caption{Comparison of \shortApproach certified robust accuracy with different types of additional smoothing distribution $\gQ$ and Neyman-Pearson-based certification. Detail experiment settings are in \Cref{newadx:sub:diff-var-dsrs}.}
            \resizebox{\linewidth}{!}{
            \begin{tabular}{c|c|cccccccccccc}
                \toprule
                \multirow{2}{*}{Dataset} & \multirow{2}{*}{Certification} & \multicolumn{12}{c}{Certified Accuracy under Radius $r$} \\
                & & 0.25 & 0.50 & 0.75 & 1.00 & 1.25 & 1.50 & 1.75 & 2.00 & 2.25 & 2.50 & 2.75 & 3.00 \\
                \hline\hline
                \multirow{3}{*}{MNIST} & Neyman-Pearson & \bf 95.2\%	& \bf 91.9\%	& 87.7\%	& 80.6\%	& 71.2\%	& 57.6\%	& 41.0\%	& 25.5\%	& 13.6\%	& 6.2\%	& 2.1\%	& 0.9\%	 \\
                & DSRS-trunc & 95.1\%	& 91.8\%	& 87.9\%	& 81.3\%	& 72.9\%	& 60.2\%	& 46.1\%	& 30.9\%	& 17.4\%	& 9.4\%	& 3.6\%	& \bf 1.2\% \\
                & DSRS-var & 95.1\%	& 91.8\%	& \bf 88.2\%	& \bf 81.5\%	& \bf 73.6\%	& \bf 61.6\%	& 4\bf 8.4\%	& \bf 34.1\%	& \bf 21.0\%	& \bf 10.6\%	& \bf 4.4\%	& \bf 1.2\%	 \\
                \hline\hline
                \multirow{3}{*}{CIFAR-10} & Neyman-Pearson & 40.2\%	& 32.6\%	& 24.7\%	& 18.9\%	& 14.9\%	& 10.2\%	& 7.5\%	& 4.1\%	& 2.0\%	& 0.7\%	& 0.1\%	& \bf 0.1\% \\
                & DSRS-trunc & \bf 40.3\%	& 32.9\%	& 25.5\%	& 20.1\%	& 15.7\%	& 11.5\%	& 8.0\%	& 5.5\%	& 2.7\%	& 1.5\%	& 0.6\%	& \bf 0.1\%	 \\
                & DSRS-var & \bf 40.3\%	& \bf 33.1\%	& \bf 25.9\%	& \bf 20.6\%	& \bf 16.1\%	& \bf 12.5\%	& \bf 8.4\%	& \bf 6.4\%	& \bf 3.5\%	& \bf 1.8\%	& \bf 0.7\%	& \bf 0.1\%	 \\
                \hline\hline
                \multirow{3}{*}{ImageNet} & Neyman-Pearson & 42.5\%	& 37.2\%	& 33.0\%	& 29.2\%	& 24.8\%	& 21.4\%	& 17.6\%	& 13.7\%	& 10.2\%	& 7.8\%	& 5.7\%	& 3.6\% \\
                & DSRS-trunc & 42.5\%	& 38.1\%	& 34.4\%	& 30.2\%	& \bf 27.0\%	& \bf 23.3\%	& \bf 21.3\%	& \bf 18.7\%	& 14.2\%	& \bf 11.0\%	& \bf 9.0\%	& \bf 5.7\% \\
                & DSRS-var & \bf 42.9\%	& \bf 38.5\%	& \bf 35.0\%	& \bf 31.0\%	& 26.5\%	& 23.2\%	& 21.0\%	& 18.3\% & \bf 14.6\%	& 10.5\%	& 8.2\%	& 5.3\% \\
                \bottomrule
            \end{tabular}
            }
            \label{tab:variant-compare}
        \end{table*}
    
    \subsection{Comparison with Higher-Order Randomized Smoothing}

        \label{newadx:sub:compare-higher-order}
        
        It is difficult to have a direct comparison with higher-order randomized smoothing~\citep{mohapatra2020higherorder,levine2020tight}, which is the only work to the best of our knowledge that uses additional information beyond $P_A$ to tighten the robustness certification in randomized smoothing.
        This difficulty comes from the following reasons:
        (1)~Higher-order randomized smoothing only supports standard Gaussian smoothing, while \shortApproach is particularly useful with generalized Gaussian smoothing.
        (2)~All experiment evaluations in higher-order randomized smoothing are conducted with large sampling numbers~($N=2\times 10^5$ on CIFAR-10 and $N=1.25\times 10^6$ on ImageNet) that makes the evaluation costly, while \shortApproach is designed for practical sampling numbers~($N=10^5$).
        (3)~The code is not open-source yet~\citep{mohapatra2020higherorder}.
        Nonetheless, we can directly compare with the curves provided by \citet{mohapatra2020higherorder}.
        
        We capture the certified robust accuracy vs. $\ell_2$ radius $r$ curves from \citep{mohapatra2020higherorder} in \Cref{fig:higher-order}.
        As we can see, compared with Neyman-Pearson-based certification, the improvements from higher-order randomized smoothing are small especially on the large ImageNet dataset despite the excessive sampling numbers~($1.25\times 10^6$).
        In contrast, as shown in \Cref{fig:curves}, within only $10^5$ sampling number, \shortApproach is visibly tighter than Neyman-Pearson-based certification.
        In fact, to the best of our knowledge, \shortApproach is the first model-agnostic approach that is visibly tighter than Neyman-Pearson-based certification under $\ell_2$ radius.
        
        \begin{figure*}[t]
            \centering
            \subfigure[{\citep[Figure 2(b)]{mohapatra2020higherorder}: Higher-order randomized smoothing on CIFAR-10.}]{
                \includegraphics[width=0.25\linewidth]{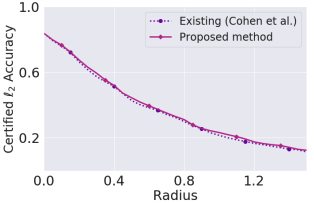}
            }
            \hspace{0.01\linewidth}
            \subfigure[{\citep[Figure 4]{mohapatra2020higherorder}: Higher-order randomized smoothing on ImageNet for models trained with different smoothing variances.}]{
                \includegraphics[width=0.6\linewidth]{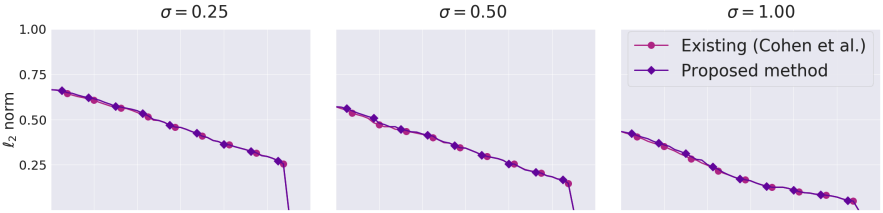}
            }
            
            \caption{Higher-order randomized smoothing certification~(solid curves) compared with standard Neyman-Pearson-based certification~(dotted curves).}
            \label{fig:higher-order}
        \end{figure*}

\section{Extended Related Work Discussion}

    \label{newadx:related}
    
    In this appendix, we discuss related work from two branches: training approaches for randomized smoothing, and data-dependent randomized smoothing.
    Both branches aim to improve the certified robustness of randomized smoothing.
    For tighter certification approaches leveraging additional information, a detailed discussion can be found in \Cref{adxsubsec:more-discussion-related-work}.
    For more related work we refer the interested readers to recent surveys and books~\citep{liu2019algorithms,li2020sok,albarghouthi2021introduction}.
    
    To improve the certified robustness of randomized smoothing, efforts have been made on both the training~\citep{li2019certified,zhai2019macer,salman2019provably,jeong2020consistency} and the certification sides~\citep{lecuyer2019certified,cohen2019certified,li2019certified,yang2020randomized,zhang2020black,yang2021certified}.
    On the training side, data augmentation~\citep{cohen2019certified}, regularization~\citep{li2019certified,zhai2019macer,jeong2020consistency}, and adversarial training~\citep{salman2019provably}  help to train stable base models under noise corruptions so that higher certified robustness for a smoothed classifier can be achieved.
    In this work, we focus on certification, and these training approaches can be used in conjunction with ours to provide higher certified robustness.
    
    Another potential way to improve the certified robustness of randomized smoothing is to dynamically change the smoothing distribution $\gP$ based on the input toward maximizing the certified radius~\citep{alfarra2020data,eiras2021ancer,schuchardt2022localized}.
    In this scenario, the certification needs to take into account that the attacker may adaptively mislead the pipeline to choose a ``bad'' smoothing distribution.
    Therefore, additional costs such as memorizing training data need to be paid to defend such adaptive robustness vulnerabilities.
    A recent work~\citep{sukenik2021intriguing} shows that input-dependent randomized smoothing may not bring substantial improvements in certified robustness.
    In \shortApproach, we select the additional smoothing distribution $\gQ$ dynamically based on the input, which may appear like input-dependent randomized smoothing.
    However, we select such distribution $\gQ$ only for certification purposes, and the original distribution $\gP$ that is used to construct the smoothed classifier remains static.
    Thus, we do not need to consider the existence of adaptive attackers.
    Indeed, for any smoothing distribution $\gQ$, with \shortApproach, we generate valid robustness certification for the static smoothed classifier $\wF^\gP$.
    

\section{Discussions on Generalizing \shortApproach Framework}

    \label{newadx:general}

    In this appendix, we first introduce prior work that leverages additional information for certification in randomized smoothing, then generalize our \shortApproach as a more general framework to theoretically compare with the related work and highlight future directions.
    
    \subsection{Existing Work on Leveraging Additional Information for Certification}

        \label{adxsubsec:more-discussion-related-work}
        We discuss all known work that leverages more information to achieve tighter robustness certification for randomized smoothing prior to this paper to the best of our knowledge.
        
        \paragraph{Additional Black-Box Information.}
        Our \shortApproach leverages additional information to tighten the certification for randomized smoothing.
        We leverage the information from an additional smoothing distribution.
        This information can be obtained from the base classifier that we \emph{only} have query access on the predicted label.
        We call the information from this limited query access ``black-box'' information.
        The certification approaches that only require black-box information can be applied to any classification model regardless of the model structure.
        Thus, they are usually more general and more scalable.
        The classical Neyman-Pearson certification only requires black-box information.
        
        \emph{Besides our \shortApproach, the only other form of additional black-box information that is leveraged is the higher-order information~\citep{levine2020tight,mohapatra2020higherorder}.}
        Formally, our additional black-box information has the form
        \begin{equation}
            \Pr_{\vepsilon\sim\gQ} [F_0(\vx + \vepsilon) = y].
        \end{equation}
        In contrast, the higher-order information, especially second-order information used by \citet{levine2020tight, mohapatra2020higherorder} has the form
        \begin{equation}
            \|\nabla f_0^\gP(\vx_0)_{y_0}\|_p = \|\nabla \Pr_{\vepsilon\sim\gP} [F_0(\vx_0 + \vepsilon) = y_0]\|_p
        \end{equation}
        that is also shown to be estimable given the black-box query access.
        However, the higher-order information has several limitations:
        (1)~It is hard to leverage the information beyond the second order.
        Therefore, only second-order information is used in existing certification approaches yet.
        However, to achieve optimal tightness, one needs to leverage infinite orders of information, which brings an infinite number of constraints and is thus intractable.
        In contrast, we show that extra information from only one additional distribution suffices to derive a strongly tight certification.
        (2)~In practice, the second-order information shows marginal improvements in the widely used $\ell_2$ and $\ell_\infty$ certification settings on real-world datasets~\citep{levine2020tight,mohapatra2020higherorder} and even such improvements require a large number of samples~(usually in million order instead of ours $10^5$).
        
        \citet{dvijotham2020framework} also propose to use additional information to tighten the robustness certification for randomized smoothing~(``full-information'' setting).
        They formalize the tightest possible certification and compare it with Neyman-Pearson-based certifiation~(``information-limited'' setting), but in practice, they do not try to leverage information from distributions other than $\gP$.

        \paragraph{Constraining Model Structure.}
        If we discard the ``black-box information'' constraint, we can obtain tighter robustness certification than classical Neyman-Pearson.
        For example, knowing the model structure can benefit the certification.
        \citet{lee2019tight} show that when the base classifier is a decision tree, we can use dynamic programming to derive a strongly tight certification against $\ell_0$-bounded perturbations.
        \citet{awasthi2020adversarial} show that, if the base classifier first performs a known low-rank projection, then works on the low-rank projected space, for the corresponding smoothed classifier, we can have a tighter certification on both $\ell_2$ and $\ell_\infty$ settings.
        However, it is challenging to find a projection such that the model preserves satisfactory performance while the projection rank is low.
        Indeed, the approach is evaluated on DCT basis to show the improvement on $\ell_\infty$ certification, and there exists a gap between the actual achieved certified robustness and the state of the art.
        We do not compare with these approaches since they impose additional assumptions on the base classifiers so their applicable scenarios are limited, and currently, the state-of-the-art base classifier does not satisfy their imposed constraints under $\ell_2$ and $\ell_\infty$ certification settings.
        Recently, for $\ell_1$ certification, a deterministic and improved smoothing approach~(a type of non-additive smoothing mechanism) is proposed to handle the case where input images are constrained in space $\{0, 1/255, \cdots, 244/255, 1\}^d$~\cite{levine2021improved}.
        This could be viewed as constraining the attack space from another aspect and implies that certified robustness can be improved by better smoothing mechanisms, which is orthogonal to our work that focuses on certification for existing smoothing mechanisms.
        
        \paragraph{Confidence Smoothing.}
        A group of certification approaches assumes that the base classifier outputs normalized confidence on the given input, and the smoothed classifier predicts the class with the highest expectation of normalized confidence under noised input.
        This assumption can be viewed as a special type of  ``Constraining Model Structure''.
        Under this assumption, we can query and approximate the \emph{quantile} of the confidence under noised input: $F_0(\vx_0+\vepsilon)$ where $\vepsilon\sim\gP$.
        With this information, we can leverage the Neyman-Pearson lemma in a more delicate way to provide a tighter~(higher) lower bound of the expected confidence under perturbation, i.e., $\E_{\vepsilon\sim\gP} F_0(\vx+\vdelta+\vepsilon)$.
        
        These certification approaches provide tighter certification than the classical Neyman-Pearson for the smoothed classifier that predicts the class with the highest expected normalized confidence.
        They are also useful for regression tasks such as object detection in computer vision as shown in \citep{chiang2020detection}.
        However, for the classification task, for utilizable base classifiers~(i.e., benign accuracy $> 50\%$ under noise), if we simply set the predicted class to have $100\%$ confidence, we only increase the certified radius of the classifier and the certification for this ``one-hot'' base classifier only requires classical Neyman-Pearson.
        Thus, these certification approaches, e.g., \citep{kumar2020certifying}, may not achieve higher certified robustness on the classification task than Neyman-Pearson and therefore we do not compare with them.
    
    \subsection{General Framework}
    
        \label{adxsubsec:general-framework}
        
        Focusing on the certification with additional black-box information, we generalize the \shortApproach to allow more general additional information.
        
        \begin{definition}[General Additional Black-Box Information]
            For given base classifier $F_0$, suppose the true label at $\vx_0$ is $y_0$, for certifying robustness at $\vx_0$, 
            the general additional black-box information has the form
            \begin{equation}
                \left\{
                \begin{aligned}
                    & \int_{\sR^d} f_1(\vx) \1\left[F_0(\vx) = y_0\right] \dif \vx = c_1, \\
                    & \cdots \\
                    & \int_{\sR^d} f_i(\vx) \1\left[F_0(\vx) = y_0\right] \dif \vx = c_i, \\
                    & \cdots \\
                    & \int_{\sR^d} f_N(\vx) \1\left[F_0(\vx) = y_0\right] \dif \vx = c_N,
                \end{aligned}
                \right.
                \label{eq:def-general-additional-information}
            \end{equation}
            where $f_i$ and $c_i$ are pre-determined; $f_i$ is integrable in $\sR^d$ and $c_i \in \sR$ for $1\le i\le N$.
            \label{def:general-additional-information}
        \end{definition}
        
        \begin{remark}
            Obtaining the information in \Cref{eq:def-general-additional-information} requires only the black-box access to whether $F_0(\vx)$ equals to $y_0$.
            We define the general additional black-box information in this way because:
            (1)~The information from finite points is useless since the smoothed classifier has zero probability mass on finite points, so the useful information is based on integration;
            (2)~To provide a lower bound of $\wF_0^\gP(\vx_0 + \vdelta)_{y_0}$, we only need to care whether $F_0(\vx)$ equals to $y_0$ in region of interest.
        \end{remark}
        
        \paragraph{Examples.}
        (1)~Our \shortApproach, the additional information $\E_{\vepsilon\sim\gQ} [f(\vepsilon)] = Q_A$ instantiates \Cref{def:general-additional-information} by setting $N=1$, $f_1(\cdot) = q(\cdot - \vx_0)$ and $c_1 = Q_A$.
        (2)~In \citep{mohapatra2020higherorder,levine2020tight}, the second-order information $\nabla f_0^\gP(\vx_0)$ instantiates \Cref{def:general-additional-information} by setting $N=d$, $f_i(\vx) = -\nabla p(\vx - \vx_0)_i$, and $c_i = \left(\nabla f_0^\gP(\vx_0)\right)_i$ according to Theorem~1 in \cite{mohapatra2020higherorder}.
        We remark that due to the sampling difficulty, instead of using the whole vector $\nabla f_0^\gP(\vx_0)$ as the information, second-order smoothing~\citep{mohapatra2020higherorder,levine2020tight} uses its $p$-norm in practice.
        However, using the full information only gives tighter certification so we consider this a more ideal variant.
        
        Then, we can extend the constrained optimization problem $(\tC)$ in \Cref{subsec:cons-opt-formulation} to $(\tC^\ext)$ to accommodate the general information as such
        \begin{align}
                \underset{f}{\mathrm{minimize}} \quad & \E_{\vepsilon\sim\gP} [f(\vdelta+\vepsilon)] \\
                \mathrm{s.t.} \quad & \E_{\vepsilon\sim\gP} [f(\vepsilon)] = P_A, \nonumber \\
                & \int_{\sR^d} f_1(\vepsilon) f(\vepsilon) \dif \vepsilon = c_1, \nonumber \\
                & \cdots \nonumber \\
                & \int_{\sR^d} f_N(\vepsilon) f(\vepsilon) \dif \vepsilon = c_N, \nonumber \\
                & 0 \le f(\vepsilon) \le 1 \quad \forall \vepsilon\sim\sR^d. \nonumber
        \end{align}
        Similarly, by the strong duality~(\Cref{thm:strong-duality}), to solve the certification problem
        \begin{equation}
            \forall \vdelta \, \mathrm{s.t.} \, \|\vdelta\|_p \le r,\, \tC_\vdelta^\ext(P_A,\,c_1,\,\dots,\,c_N) > 0.5,
        \end{equation}
        we only need to solve the dual problem $(\tD^\ext)$:
        \begin{equation}
            \underset{\eta,\lambda_1,\dots,\lambda_N\in\sR}{\mathrm{maximize}} 
            \Pr_{\vepsilon\sim\gP} \left[p(\vepsilon) < \eta p(\vepsilon + \vdelta) + \sum_{i=1}^N \lambda_i f_i(\vepsilon+\vdelta)\right]
        \end{equation}
        \begin{align*}
            \mathrm{s.t.} & \Pr_{\vepsilon\sim\gP} \left[p(\vepsilon-\vdelta) < \eta p(\vepsilon) + \sum_{i=1}^N \lambda_i f_i(\vepsilon)\right] = P_A, \nonumber \\
            & \int_{\sR^d} \1\left[p(\vepsilon-\vdelta) < \eta p(\vepsilon) + \sum_{i=1}^N \lambda_i f_i(\vepsilon)\right] f_1(\vepsilon) \dif \vepsilon = c_1, \nonumber \\
            & \cdots \nonumber \\
            & \int_{\sR^d} \1\left[p(\vepsilon-\vdelta) < \eta p(\vepsilon) + \sum_{i=1}^N \lambda_i f_i(\vepsilon)\right] f_N(\vepsilon) \dif \vepsilon = c_N. \nonumber
        \end{align*}
        We remark that this generalization shares a similar spirit as one type of generalization of Neyman-Pearson Lemma~\citep{chernoff1952generalization,mohapatra2020higherorder}.
        Following the same motivation, \citet{dvijotham2020framework} try to generalize the certification by adding more constraints in their ``full-information setting''.
        However, it is unclear whether their constraints in $f$-divergences form have the same expressive power as ours in practice~(i.e., the practicality of theoretically tight Hockey-Stick divergence).
        A study of these different types of generalization  would be our future work.
        
        More importantly, we believe that the pipeline of \shortApproach can be adapted to solve this generalized dual problem.
        We hope that this generalization and the corresponding \shortApproach could enable the exploration and exploitation of more types of additional information for tightening the robustness certification of randomized smoothing.
        
    
    
    \paragraph{Implications.}
    The generalized \shortApproach framework allows us to explicitly compare different types of additional information.
    For example, comparing our additional distribution information and higher-order information, we find that 
    (1)~for additional distribution information, from \Cref{thm:suffices-binary-classification} and \Cref{cor:suffices-multiclass-classification}, the strong tightness can be acheived for $N=C-1$ where $C$ is the number of classes;
    (2)~for higher-order information, from \citep[Asymptotic-Optimality Remark]{mohapatra2020higherorder}, the strong tightness can be achieved when all orders of information are used, i.e., $N \to \infty$.
    This comparison suggests that our additional information from another smoothing distribution should be more efficient.
    
    Another implication is that, from \Cref{cor:suffices-multiclass-classification}, under multiclass setting, with proper choices of the $(C-1)$ additional smoothing distributions, if we base \shortApproach on solving dual problem $(\tD^\ext)$, the \shortApproach can achieve \emph{strong tightness} in multiclass setting.
    
    \subsection{Limitations and Future Directions}
        \label{adxsubsec:limitation-future-directions}
    
        Despite the promising theoretical and empirical results of \shortApproach, \shortApproach still has some limitations that open an avenue for future work.
        We list the following future directions: (1)~tighter certification from a more ideal additional smoothing distribution $\gQ$:
        there may exist better smoothing distribution $\gQ$ or better methods to optimize hyperparameters in $\gQ$ than what we have considered in this work in terms of certifying larger certified radius in practice;
        (2)~better training approach for \shortApproach:
        there may be a large space for exploring training approaches that favor \shortApproach certification since all existing training methods are designed for Neyman-Pearson-based certification.
        We believe that advances in this aspect can boost the robustness certification with \shortApproach to achieve state-of-the-art certified robustness.
        (3)~better additional information:
        more generally, besides the prediction probability from an additional smoothing distribution, there may exist other useful additional information for certification in randomized smoothing.
        We hope our generalization of the \shortApproach framework in this appendix can inspire future work in tighter and more efficient certification for randomized smoothing.
        
        
        



\end{document}